\pdfoutput=1
\documentclass[twoside,11pt]{article}

\usepackage{jmlr-smile}
\usepackage{jmlr2e}
\usepackage{jmlr-smile}
\usepackage{booktabs}
\usepackage[margin=1in]{geometry}
\usepackage{array}
\usepackage{microtype}
\usepackage{graphicx}
\usepackage{subfigure}
\usepackage{booktabs} 
\usepackage{enumerate}
\usepackage[OT1]{fontenc}
\usepackage{amsmath,amssymb,scalerel}
\usepackage{natbib}
\usepackage[english]{babel}
\usepackage{breakcites}
\usepackage{microtype}
\usepackage{enumitem}
\usepackage{setspace} 
\usepackage{siunitx}
\usepackage{mathtools}
\usepackage{tikz}
\usepackage{cleveref}
\crefname{figure}{Figure}{Figures}
\crefname{remark}{Remark}{Remarks}
\crefname{section}{Section}{Sections}
\crefname{appendix}{Appendix}{Appendices}
\crefname{theorem}{Theorem}{Theorems}
\crefname{equation}{Eq.}{Eqs.}
\crefname{assumption}{Assumption}{Assumptions}
\crefname{example}{Example}{Examples}
\crefname{table}{Table}{Tables}

\usetikzlibrary{shapes,decorations,arrows,calc,arrows.meta,fit,positioning}
\tikzset{
    -Latex,auto,node distance =1 cm and 1 cm,semithick,
    state/.style ={ellipse, draw, minimum width = 0.7 cm},
    point/.style = {circle, draw, inner sep=0.04cm,fill,node contents={}},
    bidirected/.style={Latex-Latex,dashed},
    el/.style = {inner sep=2pt, align=left, sloped}
}

\def\##1\#{\begin{align}#1\end{align}}
\def\$#1\${\begin{align*}#1\end{align*}}

\def\given{\,|\,}

\def\tr{\mathop{\text{tr}}\kern.2ex}

\def\J{{\mathbb J}}
\def\P{{\mathbb P}}
\def\E{{\mathbb E}}

\long\def\comment#1{}

\def\tr{\mathop{\text{Tr}}}

\def\cS{{\mathcal{S}}}

\def\cX{{\mathcal{X}}}

\def\cP{{\mathcal{P}}}

\def\cN{{\mathcal{N}}}
\def\cT{{\mathcal{T}}}

\def\cT{{\mathcal{T}}}
\def\tr{{\text{Tr}}}

\newcommand{\bel}{\begin{eqnarray}\label}
\newcommand{\eel}{\end{eqnarray}}
\newcommand{\bes}{\begin{eqnarray*}}
\newcommand{\ees}{\end{eqnarray*}}
\newcommand*\diff{\mathop{}\!\mathrm{d}}
\newcommand{\defeq}{\vcentcolon=}

\def\real{{\mathbb{R}}}
\def\R{{\real}}

\newcommand{\doo}{\operatorname{do}}
\def\indi{{\mathds{1} }}

\def \inv {^{-1}}
\def \invsq {^{-2}}

\def \sq {^{2}}

\def \LA {L_A}
\def \LB {L_B}
\def \LP {L_P}

\def \muA {\mu_A}
\def \muB {\mu_B}
\def \muIV {\mu_{\operatorname*{IV}}}

\def\sumh{ \sum_{h=1}^H }

\def\Vstar{{V^*}}
\def\Vhat{{\hat V}}
\def\Vpihat{{V^{\hat \pi}}}

\def\Vstarh{{V^*_h}}
\def\Vhath{{\hat V_h}}
\def\Vpihath{{V^{\hat \pi}_h}}

\def\Vstarhp{{V^*_{h+1}}}
\def\Vhathp{{\hat V_{h+1}}}
\def\Vpihathp{{V^{\hat \pi}_{h+1}}}

\def\Qstar{{Q^*}}
\def\Qhat{{\hat Q}}
\def\Qpihat{{Q^{\hat \pi}}}

\def\Qstarh{{Q^*_h}}
\def\Qhath{{\hat Q_h}}
\def\Qpihath{{Q^{\hat \pi}_h}}

\def \ome {{\omega}}
\def \omehat{{\hat \omega}}
\def \omest{{\omega^*}}
\def \omet{{\omega_t}}
\def \ometp{{\omega_{t+1}}}
\def \etaomet{{\eta^\omega_{t}}}
\def \etaometsq{{ (\eta^\omega_{t}) ^ 2 }}

\def \th {{\theta}}
\def \thest {{\theta^*}}
\def \thet{{\theta_t}}
\def \thetp{{\theta_{t+1}}}
\def \etathet{{\eta^\theta_{t}}}
\def \etathetsq{{ (\eta^\theta_{t} ) ^2 }}

\def \signathesq {{\sigma_{\nabla \theta}^2}}
\def \signaomesq {{\sigma_{\nabla \omega}^2}}

\def \natthephi {{\Tilde{\nabla}_\theta \Phi}}
\def \natomephi {{\Tilde{\nabla}_\omega \Phi}}

\def \Hphi {{\cH_{\phi}}}
\def \Hpsi {{\cH_{\psi}}}

\def \ltxa {{L^2(\cS,\cA)}}
\def \ltz {{L^2(\cZ)}}

\def \Wsad {{W^{\operatorname*{sad}}}}

\newcommand{\ass}{Assumption}

\newcommand{\csineq}{Cauchy–Schwarz inequality}

\newcommand{\la}{\langle}
\newcommand{\ra}{\rangle}

\newcommand{\indep}{\perp \!\!\! \perp}

\def\sad{{\operatorname*{sad}}}

\def \cond {{\,|\,}}

\def \I {{\operatorname*{I}}}
\def \II {{\operatorname*{II}}}
\def \III {{\operatorname*{III}}}
\def \iid {{\scriptscriptstyle \operatorname*{iid}}}
\hypersetup{ hidelinks }

\newcommand{\change}[1]{ {#1}}
\usepackage{lastpage}
\jmlrheading{25}{2024}{1-\pageref{LastPage}}{8/22; Revised
9/24}{10/24}{22-0965}{Luofeng Liao, Zuyue Fu, Zhuoran Yang, Yixin Wang, Dingli Ma, Mladen Kolar, Zhaoran Wang}

\ShortHeadings{IVVI}{Liao, Fu, Yang, Wang, Ma, Kolar and Wang}
\firstpageno{1}

\newif\ifjmlr
\jmlrtrue 

\begin{document}

\title{Instrumental Variable Value Iteration for Causal Offline Reinforcement Learning}

\author{\name Luofeng Liao$^*$ \email{ll3530@columbia.edu} \\
		\addr Department of Industrial Engineering and Operations Research \\
		Columbia University \\
		New York, NY 10027, USA
\AND
\name Zuyue Fu$^*$ \email{zuyuefu2022@u.northwestern.edu} \\
\addr Department of Industrial Engineering and Management Sciences \\
Northwestern University \\
Evanston, IL 60208, USA
\AND
\name Zhuoran Yang \email{zhuoran.yang@yale.edu} \\
\addr Department of Statistics and Data Science  \\
Yale University \\
New Haven, CT 06520, USA
\AND
\name Yixin Wang \email{yixinw@umich.edu} \\
\addr Department of Statistics \\
University of Michigan \\
Ann Arbor, MI 48109, USA
\AND
\name Dingli Ma \email{dingli98@uw.edu} \\
\addr Department of Information Systems and Operations Management, Michael G. Foster School of Business \\
University of Washington \\
Seattle, WA 98195, USA
\AND
\name Mladen Kolar \email{mkolar@marshall.usc.edu} \\
\addr Department of Data Sciences and Operations, Marshall School of Business \\
University of Southern California \\
Los Angeles, CA 90089, USA
\AND
\name Zhaoran Wang \email{zhaoranwang@gmail.com} \\
\addr Department of Industrial Engineering and Management Sciences \\
Northwestern University \\
Evanston, IL 60208, USA
}

\editor{Eric Laber}
\maketitle
\begin{abstract}%
In offline reinforcement learning (RL) an optimal policy is learned solely from a priori collected observational data. However, in observational data, actions are often confounded by unobserved variables. Instrumental variables (IVs), in the context of RL, are the variables whose influence on the state variables is all mediated by the action. When a valid instrument is present, we can recover the confounded transition dynamics through observational data. We study a confounded Markov decision process where the transition dynamics admit an additive nonlinear functional form. Using IVs, we derive a conditional moment restriction through which we can identify transition dynamics based on observational data. We propose a provably efficient IV-aided Value Iteration (IVVI) algorithm based on a primal-dual reformulation of the conditional moment restriction. To our knowledge, this is the first provably efficient algorithm for instrument-aided offline~RL.
\end{abstract}
\section{Introduction}
\label{sec:intro}

In reinforcement learning (RL) \citep{sutton2018reinforcement}, an agent maximizes its expected total reward by sequentially interacting with the environment. RL algorithms have been applied in the healthcare domain to dynamically suggest optimal treatments for patients with certain diseases \citep{raghu2017continuous, komorowski2018artificial,futoma2018learning, namkoong2020off,guez2008adaptive,parbhoo2017combining,prasad2017reinforcement}. One of the main concerns of working with observational data, especially for RL applications in healthcare, is the confounding caused by unobserved variables. Because available data may not contain measurements of important prognostic variables that guide treatment decisions or heuristic information such as visual inspection or discussions with patients during each treatment period, 
there exist variables that affect both treatment decisions and the next stage health status of patients. See \citet{brookhart2010confounding} for a detailed discussion of sources of confounding in healthcare datasets. 

Another motivating example for this work is recommender systems. In movie recommendation systems, the platform collects users' viewing history and movie ratings. It is desirable to learn from the collected datasets a movie recommendation policy that fits users' preferences and results in high movie ratings. However, there are often factors that affect users' action (watch movies or not) and movie preference. For example, the director or the star of the movie \citep{wang2020causal}.

Instrumental variables (IVs) are a well-known tool in econometrics and causal inference to identify causal effects in the presence of unobserved confounders (UCs). Informally, a variable $Z$ is an IV for the causal effect of the treatment variable $X$ on the outcome variable $Y$, if (i) it is correlated with $X$, and (ii) $Z$ affects only $Y$ through $X$, and (iii) 
$Z$ should be exogenous, 
e.g.,
$Z$ is independent of unobserved confounders. 
We provide several concrete use cases below, beginning with a recommendation system application. 

\begin{example}[(Recommendation as an IV, MovieLens 1M data)]
    In recommender systems, we could model users' experience as the outcome variable, and watching some movie as the treatment.
    The goal is to identify a sequence of movies that improve user experience when the user actually watches these movies.
    Conditional on a user, when the recommendation is sufficiently randomized, the recommendation itself can be used as an IV to deconfound the effect of a movie to user experience.  We discuss this application with a semi-synthetic dataset based on the MovieLens 1M dataset \citep{harper2015movielens} in \cref{sec:movielense}.
\end{example}

IVs are also commonly used in the healthcare domain to identify the effects of a treatment or intervention on health outcomes. There are some common sources of IVs in the medical literature, such as preference-based IVs (see Example~\ref{ex:preference}), distance to a specialty care provider (see Example~\ref{ex:nicu}), and genetic variants \citep{baiocchi2014instrumental}. 
Such a wide range of potential use of IVs in these healthcare sequential decision-making settings is a key motivator of the paper.

\begin{example}[(Differential travel time as an IV, NICU data)]
\label{ex:nicu}
    
\cite{Lorch2012}, \cite{michael2020instrumental}, and \cite{chen2021estimating} studied the effect of delivery on neonatal mortality in high-level neonatal intensive care units (NICU), using the same differential travel time as an IV. The goal is to design a neonatal regionalization system that designates hospitals according to the level of care that infants need. The available dataset has $\sim$180,000 records of mothers who delivered exactly two births during 1995 and 2009 in Pennsylvania and relocated at the second delivery. In \cref{fig:nicu} we present a possible causal DAG for the NICU application. UCs are present due to mothers' self-selection effects or unrecorded side information on which the physicians base the NICU suggestion. The differential travel time to the closest high-level NICU versus low-level NICU serves as a valid IV, since it affects the choice of the mother's NICU and does not impact clinical outcomes through other means. A neonatal regionalization system (\cref{fig:nicu}, bottom panel) designates the NICU solely based on the clinical outcome at the previous stage (since differential travel time does not affect the clinical outcome anymore once we actually \textit{assign} NICU, and confounders remain unobserved), removing arrows that point to the decision of the NICU in the DAG presented in the upper panel.

\end{example}

\begin{example}[(Preference-based IV, MIMIC-III data)]
    \label{ex:preference}
    For example, the work of \citet{Brookhart2007} discusses the use of preference-based IVs. 
They assume that different healthcare providers, at the level of geographic regions, hospitals, or 
individual physicians, have different preferences on how medical procedures are performed. 
Then preference-based IVs are variables 
that represent the variation in these healthcare providers. 
In the context of sepsis management by applying RL \citep{komorowski2018artificial} on the MIMIC-III dataset \citep{johnson2016mimic}, 
the effect of doses of intravenous fluids and vasopressors ($X$) on the health status of patients ($Y$) 
is likely to be confounded by unrecorded severity level of comorbidities. 
Then a physician's preference for prescribing vasopressors ($Z$)
is a potentially valid IV since it affects directly the actual doses given ($X$),
but is unlikely to affect the next-stage health status through other causes of $Y$.
\end{example}

\begin{figure}[t]
    \centering
    \fbox{\includegraphics[scale=.41]{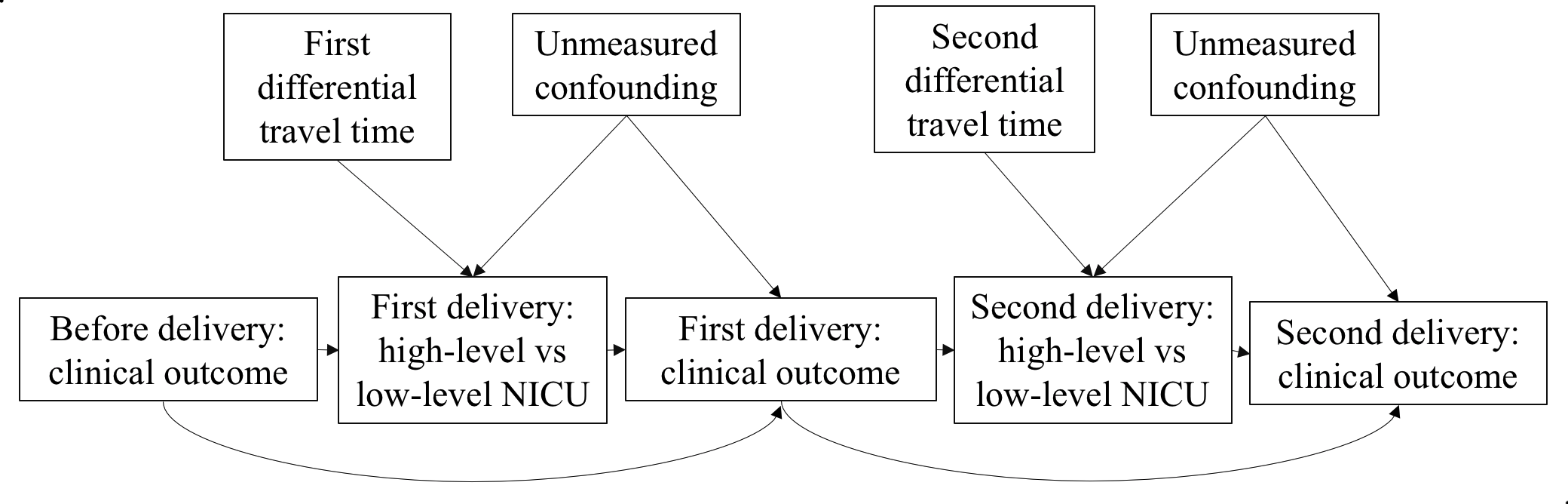}}
    \\
    \hspace{-.13cm} \fbox{\includegraphics[scale=.41]{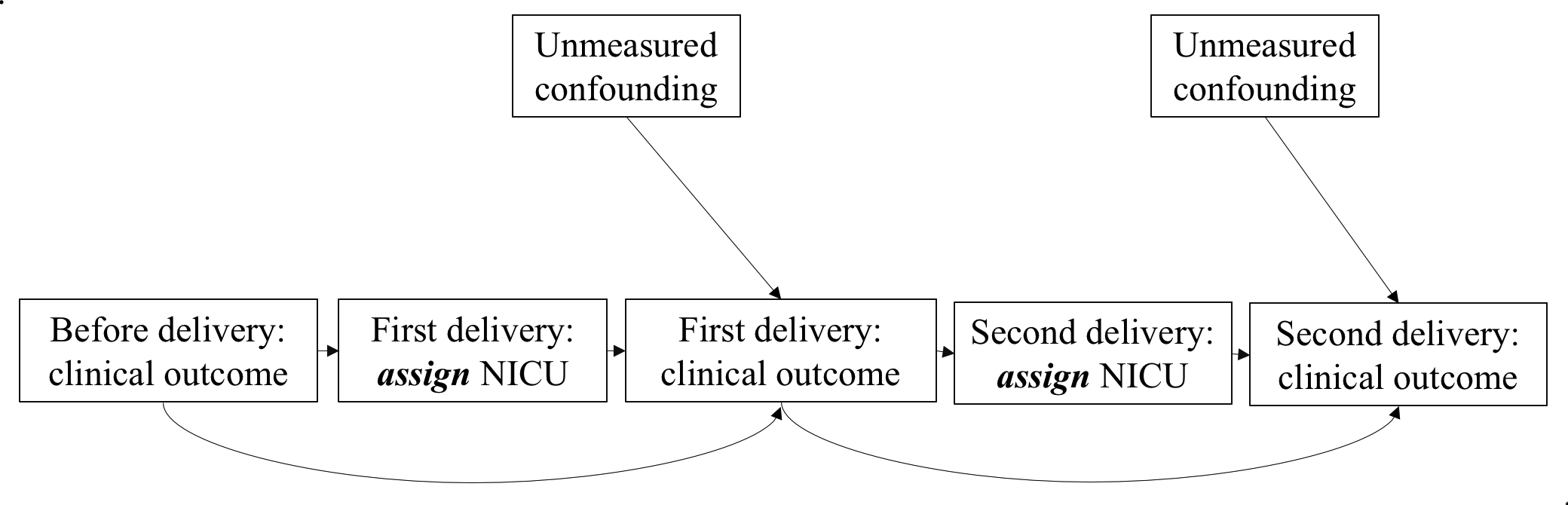}}
    \caption{The NICU application, adapted from \citet[Figure 1]{chen2021estimating}. Sufficient covariates have been conditioned on.  Top panel: DAG representing data generation process where UCs are present. Bottom panel: DAG representing a prenatal regionalization system in action.}
    \label{fig:nicu}
\end{figure}

We summarize three aspects of offline sequential datasets often encountered by RL practitioners: (i) there is a large amount of logged data where the actual effects of action on the outcome are confounded, (ii) a valid IV is, in some situations, available, and (iii) it is expensive or unethical to do experimentation and then inspect the actual performance of a target policy. We ask 
\vspace{-.2cm}
\begin{center}
\emph{
    When a valid IV is present, can we design a provably efficient offline RL\\ algorithm using only confounded observational data?}
\end{center}
\vspace{-.2cm}

\noindent 
We answer this question affirmatively. We formulate the sequential decision-making process in the presence of both IVs and UCs through a model that we call Confounded Markov Decision Process with Instrumental Variables (CMDP-IV). We then propose an IV-aided Value Iteration (IVVI) algorithm to recover the optimal policy through a model-based approach. Our contribution is threefold. First, under the additive UC assumption, we derive a conditional moment restriction through which we point identify transition dynamics.
Second, we reformulate the conditional moment restriction as a primal-dual optimization problem and propose an estimation procedure that enjoys computational and statistical efficiency jointly. 
Finally, we show that the sample complexity of recovering an $\epsilon$-optimal policy using observational data with IVs is $O(\muIV^{-4}\muB^{-2.5}H^4d_x\epsilon^{-2})$, where $0< \muIV< 1$ quantifies the strength of the IV, $\mu_B$ is the minimum eigenvalue of the dual feature covariance matrix, quantifying the compatibility of the dual linear function space and the IV, $H$ is the horizon of the MDP, and $d_x$ is the dimension of states. To the best of our knowledge, this is the first result on sample complexity for an IV-aided offline RL.

Several results developed in the paper are worth noting. 
We propose a stochastic approximation estimator for nonparametric IV problem, which is jointly computationally and statistically efficient. 
We are also among the first to study offline RL in multi-stage settings with continuous actions in the face of unobserved confounding and continuous IVs. 
Our results on solving a stochastic quadratic saddle-point problem may be of independent interest.

\subsection{Related Work}

\textbf{Identification of Causal Estimand in Sequential Settings}
RL in the presence of UCs has attracted increasing attention.
One major difficulty of working with unobserved confounders is the issue of identification. When unobserved confounders are present, causal effects of actions are not identifiable from data without further assumptions. In these settings, several approaches are available. The first one is the sensitivity-analysis based approach \citep{rosenbaum2002}, where we posit additional sensitivity assumptions on how strong the unobserved confounding can possibly be. These sensitivity assumptions enable partial identification of the causal quantity. This approach is employed by a series of work by \citet{pmlr-v89-kallus19a, kallus2020confounding,Kallus2021,namkoong2020off}. The second approach is to assume access to other auxiliary variables that can enable point or partial identification. We adopt the
second approach in this work, by assuming the access to instrumental variables. 
Under an additive UC assumption (see~Eq.\ \ref{eq:modelbegin}), 
instrumental variables can enable \textbf{point identification} of the structural quantity through conditional moment restriction (along with certain completeness assumptions), allowing us to work with continuous actions and continuous IVs. For example, in the NICU application, differential travel time (the IV) is a continuous quantity. Note that  several other related works also study the use of instrumental variables~\citep{Pu_2021,chen2021estimating}.
These works, and in particular the work by \cite{chen2021estimating},
rely on partial identification bounds in the fully nonparametric IV setting \citep{manski1990nonparametric, balke1994counterfactual}. These bounds are only available for binary IVs or binary treatments, restricting the use of their algorithms in many real-world scenarios where the IV is continuous. A continuous IV like the differential travel time must be dichotomized if one were to apply these algorithms.

\textbf{Dynamic treatment regime (DTR)}
DTRs \citep{Murphy2003, Chakraborty2013, Chakraborty2014} are a popular model for sequential decision making. DTR learning differs from RL in that it does not require the Markov assumption and the quantity of interests is an optimal adaptive dynamic policy that makes its decision based on all information available prior to the decision point. However, unobserved confounding is often expected in observational data, and yet few works handle UCs in DTR learning. A concurrent work by \cite{chen2021estimating} study the policy improvement problem in the presence of UCs, using partial identification results of causal quantities with IVs \citep{manski1990nonparametric, balke1994counterfactual}. However, these identification results often apply to binary treatments or binary IVs, restricting their use in many real-world scenarios where the IV is continuous. In our work, the transition function is point-identified under the additive UC assumption. This enables us to work with continuous actions and continuous IVs.

\textbf{RL in the presence of UCs.} \citet{zhang2016markov} formulate
the MDP with unobserved confounding using the 
language of structural causal models. \citet{lu2018deconfounding} study a model-based RL algorithm
in a combined online and observation setting. They propose a structural 
causal model for the confounded MDP and estimate the structural 
function with neural nets using the observational data. \citet{buesing2018woulda} propose a model-based RL algorithm
in the evaluation setting that learns the optimal policy for a partially 
observable Markov decision process (POMDP). \citet{pmlr-v97-oberst19a} propose a class of 
structural causal models (SCMs) for the data generating process of POMDPs and then discuss identification of counterfactuals of trajectories in the SCMs. \citet{tennenholtz2019off} study offline policy evaluation in POMDP. Their identification strategy relies on the identification results of proxy variables in causal inference \citep{miao2018identifying}.
\citet{zhang2019near,pmlr-v119-zhang20a} study the dynamic treatment regime and propose an algorithm to recover 
optimal policy in the online RL setting that is based on partial identification bounds
of the transition dynamics, which they use to design an online RL algorithm. 
\citet{namkoong2020off} study offline policy evaluation when UCs affect only one of the many decisions made. They work with a partially identified model and construct partial identification bounds of the target policy value. \citet{andrew2020offpolicy} study off-policy evaluation in infinite horizon. Their method relies on estimation of the density ratio of the behavior policy and target policy through a conditional moment restriction. \citet{kallus2020confounding} study off-policy evaluation in infinite horizon. 
They characterize the partially identified set of policy values and compute bounds on such a set. \citet{NEURIPS2018_3a09a524, Kallus2021} study policy improvement using sensitivity analysis.

\textbf{Primal-dual estimation of nonparametric IV (NPIV)} Typical nonparametric approaches to IV regression include smoothing kernel estimators and sieve estimators \citet{newey2003instrumental,carrasco2007linear, chen2018optimal, darolles2011nonparametric}, and very recently, reproducing kernel Hilbert space-based estimators \citet{singh2019kernel, muandet2019dual}. However, traditional nonparametric methods are not scalable and thus not suitable for modern-day RL datasets. 

Our proposed method builds on a recent line of work
that investigates primal-dual estimation of NPIV
\citep{dai17a, lewis2018adversarial,NIPS2019_8615,muandet2019dual,dikkala2020minimax, liao2020provably}.

This paper differs from previous works in primal-dual estimation of NPIV in two aspects. First, we solve the NPIV problem through a stochastic approximation (SA) approach \citep{robbins1951stochastic}. The SA approach is an online procedure in the sense it updates the estimate upon receiving new data points. This is a more desirable framework for practical RL applications. For example, in business application of RL, data is logged following business as usual, streaming into the database system. New technology such as wearable devices allows real-time collection of health information, medical decisions and their associated outcomes. Faced with large amounts of data, practitioners typically prefer algorithms that process new data points in real time; see Remark~\ref{rm:savssaa} for a detailed comparison with the sample average approximation approach. Our stochastic approximation approach to NPIV problem tackles computational error and statistical error jointly and is well-suited for streaming data.

Second, despite that the stochastic saddle-point problem is not strongly-convex-strongly-concave, we show a fast rate of $O(1/T)$ can be attained by a simple stochastic gradient descent-ascent algorithm.


\subsection{Notation}\label{sec:notation}
We use $\|\cdot\|_2$ to denote the $\ell_2$-norm of a vector or the spectral norm of a matrix, and use $\|\cdot\|_F$ to denote the Frobenius norm of a matrix. For vectors $a,b$ of the same length, let $a\cdot b$ denote the inner product. We denote by $\Delta(\cM;\cN)$ the set of distributions on $\cM$ indexed by elements in $\cN$. For a real symmetric matrix $A$, let $\sigma_\text{max}(A)$ and $\sigma_\text{min}(A)$ be its largest and smallest eigenvalues, respectively. For any positive integer $n$, we define $[n]=\{1,\dots,n\}$.
For any bounded function $\varphi\colon \cX\to \RR^{d_\varphi}$, we define the linear function space spanned by $\varphi$ as $\cH_\varphi = \{\theta\cdot \varphi \colon \theta \in \R^{ d_\varphi} \}$. For any function $f=\theta \cdot \varphi \in \cH_\varphi$, we denote by $\|f\|_{\varphi}=\| \theta \|_{2}$ its norm.

\section{Problem Setup} \label{sec:problemsetup}
We formulate the problem in this section. 
We first define instrumental variables (IVs) in \cref{sec:defiv} as a preliminary.
In \cref{sec:onlinesetting}, we describe the \textit{evaluation setting}, where we test the performance of our learned policy. 
In \cref{sec:offlinedata}, we describe the \textit{observation setting} in
which we collect the observational data to learn a policy. 
Our goal is then to recover the optimal policy for the evaluation setting, using only data collected 
in the observation setting.

\subsection{Preliminaries: Instrumental Variables} \label{sec:defiv}

We define confounders and IVs as follows.
\begin{definition}[Confounders and Instrumental Variables, \citealt{pearl2009causality}]\label{def:iv}
A variable $\varepsilon$ is a confounder relative to 
the pair $(X, Y)$ if $(X,Y)$ are both caused by $\varepsilon$. 
A variable $Z$ is an IV relative to the pair $(X, Y)$, 
if it satisfies the following two conditions: (i) $Z$ is independent of all 
variables that have influence on $Y$ and are not mediated by $X$; (ii) $Z$ is not independent of $X$.
\end{definition}
Figure~\ref{fig:confoundedmcp} (left panel) illustrates a typical causal directed acyclic graph (DAG) for an IV,
where $Z$ is the IV relative to the pair $(X,Y)$, and $\varepsilon$ is the UC relative to the pair $(X,Y)$. The the DAG in Figure~\ref{fig:confoundedmcp} (left) 
can also be characterized by $X=g(Z,\varepsilon)$ and $Y=f(X,\varepsilon)$ given independent $Z$ and $\varepsilon$, 
where $f$ and $g$ are two deterministic functions.

\subsection{CMDP-IV}
\label{sec:cmdp-iv}

We first introduce a type of finite-horizon Markov Decision Process (MDP)
in the observation setting with UCs and IVs, 
which we term \textit{Confounded Markov Decision Process with Instrumental Variables} (CMDP-IV). CMDP-IV is a natural extension of the IV model introduced in \cref{sec:defiv} to the multi-stage decision making process.

A CMDP-IV is defined as a tuple $M=(\cS,\cA,\cZ,\cU, H, r; \xi_0, \cP_e, \cP_z,F^*,\pi_b)$, 
where the sets $\cS \subseteq \R^{d_x}$ and $\cA$ are state and action spaces;
the set $\cZ \subseteq \RR^{d_z}$ is the space of IVs;
the set $\cU \subseteq \R^{d_x}$ is the space of UCs;
the integer $H$ is the length of each episode;
and $r = \{r_h:\cS\times\cA\to [0,1]\}_{h=1}^H$ is the set of deterministic reward functions, 
where $r_h$ is the reward function at the $h$-th step. 
For simplicity of presentation, 
we assume that the reward function $r_h$ is known for any $h\in [H]$. 
Furthermore, $\xi_0 \in \Delta(\cS)$ is the initial state distribution,
$\cP_e = \mathcal{N}(0,\sigma\sq I_{d_x})$ is the distribution of UCs, 
and $\cP_z$ is the distribution of IVs. The function $F^*:\cS\times \cA\to \cS$ is a 
deterministic transition function and  
$\pi_b = \{ \pi_{b,h}\in\Delta(\cA;\cS,\cZ,\cU)\}_{h=1}^H$ is the behavior policy,
where $\pi_{b,h}$ is the behavior policy at the $h$-th step.

\subsubsection{Evaluation setting: Bellman Equations and Performance Metric} \label{sec:onlinesetting}




We now introduce the evaluation setting of CMDP-IV. The evaluation setting is the same as the usual RL setup \citep{sutton2018reinforcement}: we want to find an optimal policy in the MDP.

For a policy $\pi = \{ \pi_h \in \Delta(\cA;\cS)\}_{h=1}^H$, given an initial state $x_1\sim \xi_0$, for any $h\in [H]$, 
the dynamics in an evaluation setting at the $h$-th step is
\# \label{eq:modelonlinebegin}
a_h \sim \pi_h(\cdot \cond x_h)\;,\quad
x_{h+1} = F^*(x_h,a_h) + e_h \;, 
\#
where $\{ e_h\}_{h=1}^H \stackrel{\iid}{\sim} \cP_e$ is the sequence of Gaussian innovations. The episode terminates if we reach the state $x_{H+1}$.
For simplicity, for any $F:\cS\times\cA\to\R^{d_x}$ we define the following transition kernel
\# \label{eq:defguasstranskernel}
\cP_F(\cdot \cond x_h,a_h) = \mathcal{N}(F(x_h,a_h), \sigma\sq I_{d_x})\;.
\#
We define the value function and the Q-function of a policy under the evaluation setting \cref{eq:modelonlinebegin}. For any $h\in [H]$, given any policy $\pi_h$ at the $h$-th step, we define its value function $V^{\pi}_h: \cS \rightarrow \RR$ and its Q-function $Q^{\pi}_h: \cS \times \cA \rightarrow \R$ as follows, 
\#\label{eq:vfunc}
  V^{\pi}_h(x) \defeq \EE_{\pi} \Bigl[ \sum_{i=h}^H r_i(x_i,a_i) \,\Big|\, x_h=x \Bigr]
  \;,\,\,\,
  Q_{h}^{\pi}(x, a) \defeq\mathbb{E}_{\pi}\Bigl[\sum_{i=h}^{H} r_{i}\left(x_{i}, a_{i}\right) \Big| x_{h}=x, a_{h}=a\Bigr]
  \;.
\#
Here, the expectation $\EE_{\pi}$ is taken with respect to the randomness of the state-action sequence $\{(x_i,a_i)\}_{i=h}^H$, where the action $a_i$ follows the policy $\pi_i(\cdot \,|\,x_i)$ and the next state $x_{i+1}$ follows the transition kernel $\cP_{F^*}(\cdot\,|\,x_i,a_i)$ defined in \cref{eq:defguasstranskernel} for any $i \in \{ h, h+1,  \ldots, H\}$. 

An optimal policy $\pi^*$ gives the optimal value $V_{h}^{*}(x)=\sup _{\pi} V_{h}^{\pi}(x)$ for any $(x,h)\in\cS \times [H]$. We assume that such an optimal policy $\pi^*$ exists.
For a given policy $\pi = \{ \pi_h \in \Delta(\cA;\cS)\}_{h=1}^H$, 
its suboptimality compared to the optimal policy $\pi^* = \{ \pi_h^* \}_{h=1}^H$ is defined as \footnote{We should use $\operatorname{esssup}$ to be more measure-theoretically rigorous.} 
\# \label{eq:subopt}
\| V^*_1 - V^\pi_1  \|_\infty \coloneqq \sup_{x\in \cS} V^*_1(x) - V^\pi_1(x)\;. \;
\#
We describe the Bellman equation and the Bellman optimality equation for the evaluation setting. For any $(x, a, h) \in \mathcal{S} \times \mathcal{A}\times [H]$, 
the Bellman equation of the policy $\pi$ takes the following form, 
\$
 Q^{\pi}_h(x,a) = (r_h + \mathbb{P} V^{\pi}_{h+1})(x,a) \;,\quad
 V^{\pi}_h(x) = \la Q^{\pi}_h(x, \cdot),  \pi_h(\cdot \given x) \ra_\cA \;,\quad
V^\pi_{H+1}(x) = 0\;,
\$
where $\la Q^{\pi}_h(x, \cdot), \pi_h(\cdot \given x) \ra_\cA 
= \int_\cA Q_h^\pi(x, a) \pi_h( \ud a\given x)$
and $\mathbb{P}$ is the operator form of the transition kernel $\cP_{F^*}$, i.e., defined as
$( \P f)(x,a) = \EE_{x'\sim \cP_{F^*}(\cdot\,|\,x,a)}[ f(x')  ]$
for any function $f:\cS \rightarrow \R$. 
The subscript $\cA$ is omitted subsequently 
if it is clear from the context.
Similarly, the Bellman optimality equation takes the following form, 
\#\label{eq:fu-opt-bell}
Q_{h}^{*}(x,a)= (r_{h}+\mathbb{P} V_{h+1}^{*})(x,a)
\;,\quad
V_{h}^{*}(x)=\max _{a \in \mathcal{A}} Q_{h}^{*}(x, a)
\;,\quad
V_{H+1}^{*}(x)=0
\;, 
\#
which implies that to find an optimal policy $\pi ^*$, 
it suffices to estimate the optimal Q-function
and then construct the greedy policy with respect to the optimal Q-function.

\subsubsection{Observation Setting: Data Collection Process} \label{sec:offlinedata}

We describe the observation setting of CMDP-IV, in which we collect the data by executing the behavior policy $\pi_b \in \Delta(\cA;\cS, \cZ,\cU)^H$. This distinguishes our work from most works in offline RL since we need to handle the issue of unobserved confounders, which makes the already difficult offline RL problem even more challenging.

At the beginning of each episode, the environment  generates an initial state $x_1\sim\xi_0$,  a sequence of UCs $\{e_h\}_{h} \stackrel{ \iid}{\sim} \cP_e$, and a sequence of observable IVs $\{z_h\}_{h} \stackrel{\iid}{\sim}\cP_z$. At the $h$-th step, given the current state $x_h$, the action $a_h$ and the next state $x_{h+1}$ are generated according to the following dynamics, 
\#
\label{eq:modelbegin}
a_h  \sim \pi_{b,h}(\cdot \cond x_h,z_h,e_h)\;,\quad
 x_{h+1}= F^*(x_h, a_h) + e_h\;.
\#
The episode terminates if we reach the state $x_{H+1}$
and we collect all observable variables, 
i.e., $\{ (x_h, a_h, z_h, x_h')  \}_{h\in [H]}$, where $x_h' = x_{h+1}$ for any $h\in [H]$. 
\change{
    \begin{assumption}
        \label{as:dynamic}
    The collection of random variables $\{ e_1,\dots, e_H, z_1, \dots, z_H, x_1 \}$ are independent.
    Moreover, we assume the marginal distribution of confounders are identical, and that and the marginal distribution of confounders of instruments are identical, i.e., $e_h \sim \cP_e$, $z_h \sim \cP_z$ for all $h\in[H]$.
    \end{assumption}
}
A causal DAG is given in Figure~\ref{fig:confoundedmcp} (right) to graphically illustrate 
such dynamics. At any stage $h$, the variable $z_h$ is an IV relative to the pair $(a_h,x_{h+1})$.  Indeed, $z_h$ affects the action $a_h$ only through \cref{eq:modelbegin}, and its effect on $x_{h+1}$ must be channelled through $a_h$ because it does not appear in the second equation in \cref{eq:modelbegin}.

The main difference between the evaluation setting \cref{eq:modelonlinebegin} and the observation setting \cref{eq:modelbegin} is whether the UC $e_h$ has an effect on the action $a_h$. 
In the language of causal inference \citep{pearl2009causality}, a policy $\pi = \{ \pi_h\in \Delta(\cA;\cS) \}_{h=1}^H$ induces the stochastic intervention $\doo(a_1\sim \pi_1(\cdot \cond x_1), \dots, a_H\sim\pi_H(\cdot \cond x_H))$ on the DAG in Figure~\ref{fig:confoundedmcp}~(left part of the right panel), and the resulting DAG is obtained by removing all arrows pointing into the action $a_h$; see Figure~\ref{fig:confoundedmcp}~(right part of the right panel).  \cref{sec:defofscm} includes more details on the do-operation. 



Under the observation and the evaluation settings 
described in Sections~\ref{sec:offlinedata}~and~\ref{sec:onlinesetting}, respectively, 
we aim to answer the following question:
\begin{center}
\emph{Given data collected from the confounded dynamics \cref{eq:modelbegin} 
in the observation setting, can we find a policy that minimizes the suboptimality defined in \cref{eq:subopt} in the evaluation setting? }
\end{center}

We now remark on the modeling assumptions.

\begin{figure}[t]
\fbox{   
\begin{tikzpicture}[scale = 1.225]
    \node (z) at (-1.3,1) [label=below:$Z$, circle, fill=black]{};
    \node (x) at (0,0) [label=below:$X$,circle, fill=black]{};
    \node (y) at (1.3,0) [label=below:$Y$,circle, fill=black]{};
    \node (u) at (1,1) [label=right:{$\varepsilon$}, circle, draw]{};
    \path (z) edge (x);
    \path (x) edge (y);
    \path (u) [dashed ] edge (x)
              [dashed] edge (y);
\end{tikzpicture}
}
\hfill
\fbox{ 
\begin{tikzpicture}[scale=.95]
        \node (s) at (-1.6,0) [label=below:{$x_h$}, circle, fill=black]{};
        \node (sp) at (1.6,0) [label=below:{$x_{h+1}$}, circle, fill=black]{};
        \node (u) at (1,1) [label=right:{$e_h$}, circle, draw]{};
        \node (z) at (-1,1) [label=left:{$z_h$}, circle, fill=black]{};
        \node (a) at (0,0) [label=below:{$a_h$}, circle, fill=black]{};
    \path (z) edge (a);
    \path (a) edge (sp);
    \path (s) edge (a);

        \path (s) [out=-45,in=220] edge (sp);

    \path (u) [dashed ] edge (a)
              [dashed] edge (sp);
\end{tikzpicture}
 \hspace{.5cm}
\begin{tikzpicture}[scale=.8]
        \node (s) at (-1.6,0) [label=below:{$x_h$}, circle, fill=black]{};
        \node (sp) at (1.6,0) [label=below:{$x_{h+1}$}, circle, fill=black]{};
        \node (u) at (1,1) [label=right:{$e_h$}, circle, draw]{};
        \node  (a) at (0,0) [label=above:{$\textstyle\doo(a_h)\sim \pi_h$}, circle, fill=black]{};
    \path (a) edge (sp);

    \path (u) 
              [dashed] edge (sp);
                      \path (s) [out=-45,in=220] edge (sp);
                      
\end{tikzpicture}
}
    \caption{Left panel: An illustration of Definition~\ref{def:iv} with one UC $\varepsilon$ and three observable variables $X$, $Y$, and $Z$. Right panel:
    \textit{Observation setting} of CMDP-IV with a behavior policy $\pi_b$ (left). \textit{Evaluation setting} of CMDP-IV with intervention induced by $\pi$ (right).}
    \label{fig:confoundedmcp}
\end{figure}

\begin{remark}[Generalization of Figure~\ref{fig:confoundedmcp}~(right panel)] We have made two simplifying assumptions. First, we assume $e_h$ only confounds the transition dynamics (the arrow from $a_h$ to $x_{h+1}$). The unobservables $e_h$ could also affect the action and the reward, or state and reward, or both. Second, we assume in each stage, $z_h$ and $e_h$ are generated in an i.i.d.\ manner and are independent of all other random variables in the MDP. In practice it is likely that the sequences $\{z_h\}$ and $\{ e_h\}$ exhibit temporal dependence. We focus on this simplified model because it captures the essence of IVs: a variable that affects $x_{h+1}$ only through the action $a_h$. In the work of \citet{andrew2020offpolicy} where the authors study policy evaluation with unobserved confounders, confounders are also assumed i.i.d.
\end{remark}

\begin{remark}[On additive noise assumption]
A more general version of this problem, which we leave for future work, would be the setting where the transition dynamics are of the form $x_{t+1} = F(x_h,a_h,e_h)$, in contrast to our additive Gaussian noise assumption. We remark non-identification is a key issue in the fully non-parametric model. Let us revisit the IV diagram presented in Figure \ref{fig:confoundedmcp}, which represents the simplest case of an IV with structural equations $Y=f(X,\varepsilon)$ and $X = g(Z,\varepsilon)$, with $Z\indep \varepsilon$. It is well-known that the conditional independence implied by the IV diagram is not enough to identify the causal effect of $X$ on $Y$ \citep{causal2012bareinboim, hunermund2019causal}. Roughly this means there exist two distributions of random variables $(X,Y,Z)$ that are compatible with the IV diagram, and yet the structural functions $f$ are different. One could instead work with a partially identified IV model, using bounds of the causal effects \citep{balke1994counterfactual, balke1997bounds,zhang2021bounding}.

\end{remark}

\begin{remark}[The challenge of UCs] \label{rm:challengeofconfounder}
 The challenge stems from the fact that the UC $e_h$
 enters both of the equations \cref{eq:modelbegin}.
 For ease of discussion, suppose
 that the behavior policy $\pi_b$ is deterministic.
 With slight abuse of notations, we denote by 
 $\pi_{b,h}:\cS\times\cZ\times\cU\to\cA$ the deterministic
 behavior policy at the $h$-th step for any $h\in [H]$. 
 Now, \cref{eq:modelbegin} writes 
 $a_h = \pi_{b,h}( x_h,z_h,e_h)$. We further assume 
 that the behavior policy $\pi_{b,h}(x,z,e)$ is invertible in the third argument $e$ for any $(x,z)\in \cS\times \cZ$, which allows us to define its inverse $\pi_{b,h}\inv: \cS\times \cZ\times \cA\to\cU$. Then, by substituting $e_h = \pi_{b,h}\inv(x_h,z_h,a_h)$ into \cref{eq:modelbegin}, we have
$
x_{h+1} = F^*(x_h, a_h) + \pi_{b,h}\inv (x_h,z_h,a_h).
$
By taking expectation conditioning on $(x_h,a_h)$, we obtain  
$
\E[x_{h+1}\cond x_h,a_h] = F^*(x_h,a_h) + \delta(x_h,a_h), 
$
where $\delta(x_h,a_h) \defeq \E [ \pi_{b,h}\inv (x_h,z_h,a_h)\cond x_h, a_h]$.
This indicates that
the true transition function $F^*$ 
cannot be obtained by simply regressing $x_{h+1}$ on $(x_h,a_h)$, 
since that would result in a biased estimate. 
\end{remark}

\begin{remark}[Global IVs and global UCs] 
    \label{rm:globalivuc}
{
    Our method directly extends to cases where, instead of a time-varying IV, we only have access to a global IV that affects all the actions taken on a trajectory simultaneously, e.g. a doctor's  preference to certain treatments. The reason is that the global IV, conditional on the past history, is also a valid IV for each time step, mimicking the structure of the time-varying IV. Specifically, having a global IV is equivalent to having $z_h = z$ for all $h$, i.e. all local IVs take the same value. Then, by the full independence between $\{e_h\}_h$ and $z$, the core requirement of the time-varying IV $\E[e_h\given z]=0$ still holds, and thus our result applies. }
    
    {
    We feel that policy learning would be difficult if the global confounder affects both actions and states. 
    In more detail, consider global UCs that affect all stages of decision making, and thus affect all states $x_h$ and actions $a_h$. While IV can deconfound the effects of global UCs on the actions $a_h$, it cannot deconfound their effects on the states $x_h,x_{h+1}$. The transition dynamics from $x_h$ to $x_{h+1}$ would depend on the global UCs. This dependence would limit the performance of the learned policy in evaluation settings if the evaluation transition dynamics from $x_h$ to $x_{h+1}$ does not depend on the global UCs in the same way. 
   
    Global confounders do not seem a natural extension in our additive dynamics.
    For example, suppose the dynamics for stage $h$ write
    \$
       a_h\sim \pi_b(\cdot\given x_h, z_h, e) \;, \quad x_{h+1}=F^*(x_h, a_h) + e \;,
    \$ 
    where the UC at each stage is identical and is denoted $e$. One can difference the sequence $\{x_h\}_h$, and obtain $x_{h+1}-x_h = F^*(x_h,a_h)-F^*(x_{h-1},a_{h-1})$, where the global UC is cancelled. Due to these considerations, we focus on the CMDP-IV setting in this work, which itself is a natural extension of the IV model introduced in \cref{sec:defiv} to the multi-stage decision making process.
    }
\end{remark}


\section{IV-Aided Value Iteration} \label{sec:algo}

How can an IV help us design an offline RL algorithm? To answer this question, we proceed by a model-based approach. We estimate the transition function $F^*$ first. And then any planning algorithm (value iteration in our case) can be used to recover the optimal policy under the evaluation setting.

\subsection{A Primal-Dual Estimand}
We observe that, thanks to the presence of IVs, the transition function $F^*$ is the solution of a conditional moment restriction (CMR). To estimate the transition function $F^*$ based on the CMR, we derive a primal-dual formulation of the CMR in \cref{sec:pdformulation}. 

\subsubsection{Conditional Moment Restriction}
\label{sec:cmr}

Following the confounded dynamics \cref{eq:modelbegin}, the behavior policy $\pi_b$ induces the distribution of the observable trajectories $\{x_h,a_h,z_h ,x'_h=x_{h+1}\}_{h=1}^H$. We denote by $d_{h,\pi_b}$ the distribution of the tuple $(x_h,a_h,z_h,x'_h) \in \cS\times\cA\times\cZ\times\cS$ at the $h$-th step for any $h\in[H]$, i.e., $d_{h,\pi_b}(x,a,z,x') $. 
We further define the \textit{average visitation distribution} as follows,
\# \label{eq:def:averagemixture}
& \bar d_{\pi_b}(x,a,z,x') \defeq \frac1H \cdot \sum_{h=1}^H d_{h,\pi_b}(x,a,z,x')
\#
for any $(x,a,z,x') \in \cS\times\cA\times\cZ\times\cS$.  
We denote by $\ltxa = \{ f\colon \cS\times\cA \to \R,~\E[f(x,a)\sq]<\infty\}$ 
the space of square integrable functions equipped with 
the norm $\| f\|^2_{\ltxa} = \E[f(x,a)\sq]$.
Similarly, we define $\ltz$ and the norm  $\| g\|^2_{\ltz} = \E[g(z)\sq]$.
The operator $\cT \colon \ltxa \to \ltz$ is defined as
\# \label{eq:defT}
(\cT f)(\cdot)=\E[f(x,a)\cond z = \cdot\,] \;.
\#

The following proposition states the conditional moment restriction (CMR)
implied by the IVs in the observational confounder dynamics \cref{eq:modelbegin}. See \cref{pf:lm:mixture} for the proof.

\begin{proposition}[CMR] \label{lm:mixture}
If $(x,a,z,x')$ is distributed according to the law $\bar d_{\pi_b}$, then for any $z\in \cZ$, 
\# \label{eq:thecondmoment}
\E[ F^*(x,a)\cond z] = \E[x' \cond z]\,.
\#
\end{proposition}
Proposition~\ref{lm:mixture} implies that the transition function $F^*$ satisfies the equation $\cT F^*=\E[x'\cond z]$, where the operator $\cT$ is defined in \cref{eq:defT}. Such an equation is a Fredholm integral equation of the first kind \citep{kress1989linear}. Given data collected from $\bar d_{\pi_b}$, we aim to estimate $F^*$ based on the CMR.

\subsubsection{A Primal-Dual Estimand} \label{sec:pdformulation}
We derive a primal-dual estimand for $F^*= [f_1^*, \dots, f_{d_x}^*]^\top$.
For any $i\in[d_x]$, by Proposition~\ref{lm:mixture},
$
\E[f^*_i(x,a)\cond z ] = \E[x'_i \cond z],
$
where $x'_i$ is the $i$-th element of the next state $x'$.
We find $f^*_i$ by solving the least-square problem
$ 
\min_{f_i\in \ltxa}  \tfrac12  \E[ (\E[f_i(x,a)\cond z ] - \E[x'_i \cond z])^2].
$
By Fenchel duality, the least-square problem admits a primal-dual formulation
\#\label{eq:fu31-2}
\min_{f_i\in \ltxa} \; \max_{u_i\in \ltz} \Bigg\{ \E\big[(f_i(x,a) - x'_i) u_i(z)\big] - \tfrac12 \E\big[u_i(z)\sq\big] \Bigg\} \,,
\#
where $u_i$ is the dual variable. To approximate the $L^2$ spaces, we introduce two known feature maps
\$\phi\colon \cS\times\cA\to \R^{d_\phi} \; ,\quad \psi\colon \cZ \to \R^{d_\psi} \;,\$ and let $\cH_\phi$ and $\cH_\psi$ denote the spaces
spanned by $\phi$ and $\psi$, respectively. 
For simplicity, we define the following uncentered covariance matrices
\begin{equation}
    \hspace*{-.25cm}   
\label{eq:defAB}
    \begin{aligned}
    A \defeq \E[ \psi(z) \phi(x,a) ^\top]
    \;,
    B \defeq \E[\psi(z)\psi(z)^\top]
    \;,
    C \defeq \E[x' \psi(z)^\top]
    \;,
    D \defeq \E[\phi(x,a)\phi(x,a)^\top]
    \;.
    \end{aligned}
\end{equation}
where the expectations are taken following $\bar d_{\pi_b}$. We replace the $L^2$ spaces in \cref{eq:fu31-2} by their finite-dimensional subspaces,
\$
    \min_{f_i\in \cH_\phi}\,\max_{u_i\in \cH_\psi} \Bigg\{ \E[(f_i(x,a) - x'_i) u_i(z)] - \tfrac12 \E[u_i(z)\sq] \Bigg\} 
\,, 
\$
which, in matrix form, writes
\# \label{eq:theminimax}
\min_{\theta_i\in \R^{d_\phi}} \; \max_{\omega_i\in \R^{d_\psi}} 
\Bigg\{ \omega_i^\top A \theta_i - b_i ^ \top \omega_i - \tfrac12 \omega_i^\top B\omega_i \Bigg\}
\,, 
\#
where $b_i  \defeq \E[x_i' \psi(z)]$
and $A$ and $B$ are defined in \cref{eq:defAB}. We address the 
approximation error incurred by such finite-dimensional approximation in \cref{sec:nonparacase}.
Now we collect \cref{eq:theminimax} for all coordinates $i\in[d_x]$, giving the key primal-dual estimand $\Wsad$ 
\#\label{eq:estimatorw}
\Wsad \defeq \argmin_{W} \max_{K} \, L(W, K)
\;,
\#
where 
$L(W,K)
     \defeq \tr(KAW^\top) - \tr(CK^\top) - \tfrac{1}{2}\tr(KBK^\top)$
with $W=[\theta_1,\dots, \theta_{d_x}]^\top\in \RR^{d_x\times d_\phi}$
and $K=[\omega_1,\dots, \omega_{d_x}]^\top\in \RR^{d_x\times d_\psi}$. For appropriately chosen feature maps we expect~$\Wsad \phi \approx F^*$.


\subsection{Algorithm}\label{sec:algorithm}

We first introduce the following data sampling assumption for the algorithm.

\begin{assumption}[Observation data]\label{as:iiddata}
    We have access to i.i.d.~data from the average visitation distribution 
    defined in \cref{eq:def:averagemixture}. That is,
    $\{(x_t,a_t,z_t,x'_t)\}_{t=0}^{T-1} \stackrel{\iid}{\sim} \bar d_{\pi_b}$.
    \end{assumption}
    
    \ass~\ref{as:iiddata} is only used to simplify the presentation of our results, 
    by ignoring the temporal dependence in the data.

Algorithm~\ref{algo} introduces the backbone of the paper, IV-aided Value Iteration (IVVI),
which recovers the optimal policy under the evaluation setting 
given data collected from the confounded dynamics 
under the observation setting. Algorithm~\ref{algo} 
consists of the following two phases.

\begin{algorithm}[t]
\caption{IV-aided Value Iteration (IVVI)}
\begin{spacing}{1}
\begin{algorithmic}[1]
\STATE \textbf{Input:} 
Reward functions $\{r_h\}_{h=1}^H$, 
feature maps $\phi$ and $\psi$, 
iterations $T$, 
stepsizes $\{\eta_t^\theta, \eta_t^\omega\}_{t=1}^T$, 
initial estimates $K_0$ and $W_0$, variance $\sigma^2$, samples $\{(x_t,a_t,z_t,x'_t)\}_{t=0}^{T-1}$ in \ass~\ref{as:iiddata}. 
\STATE \textbf{Phase 1 (Estimation of $\Wsad$ in Eq. \ref{eq:estimatorw})} 
\FOR{$t=0,1,\ldots, T-1$} \label{algo:beginsgda} \label{algo:phase1begin}
\STATE $\phi_t\gets \phi(x_t,a_t)\;,$ $\psi_t\gets \psi(z_t)\;.$ 
\STATE $W_{t+1}\gets  W_{t}-\eta^\theta_t \cdot (K_t \psi_t \phi_t ^\top)\;,\quad K_{t+1}\gets K_t + \eta^\omega_t\cdot (K_t \psi_t \psi_t ^\top + x'_t \psi_t ^\top  -W_t\phi_t \psi_t ^\top)\;.$\label{algo:primalupdate}
\ENDFOR\label{algo:endsgda}
\STATE \textbf{Phase 2 (Value iteration)}   \label{algo:phase1end}

\STATE{$\hat V_{H+1}(\cdot) \leftarrow 0$,\quad $\hat W \gets W_{T}$.} \label{algo:initvhp}
\FOR{$h=H,H-1,\ldots,1$}  \label{eq:loop-s}
\STATE $\hat{Q}_h(\cdot, \cdot)   \leftarrow r_h(\cdot,\cdot) + \int_\cS \hat V_{h+1} (x')\cP_{\hat W}(\diff x'\cond \cdot,\cdot)\;.$ \label{algo:integral} 
\STATE $\hat \pi_h( \cdot) \leftarrow \argmax_{{a}} \hat Q_{h}(\cdot,a),\quad\hat V_h (\cdot) \leftarrow \max_a \hat Q_h(\cdot, a)\;.$  \label{algo:max1}
\ENDFOR \label{algo:lsviend} \label{eq:loop-t}
\STATE \textbf{Output:} $\hat \pi =\{ \hat \pi_h\}_{h=1}^H\;.$
\end{algorithmic}\label{algo}
\end{spacing}
\end{algorithm} 
\noindent\textbf{Phase 1.} 
In Lines~\ref{algo:phase1begin}--\ref{algo:phase1end} of Algorithm~\ref{algo}, we solve \cref{eq:estimatorw} using stochastic gradient descent-ascent. 
At the $t$-th iteration, we have
$
\frac{\partial L}{\partial W} = K_t A,  \frac{\partial L}{\partial K} = -( K_t B+C -W_t A^\top),
$
which combined with the definitions of $A$, $B$, and $C$ in \cref{eq:defAB}, 
gives us the updates of $W_{t+1}$ and $K_{t+1}$ in Line~\ref{algo:primalupdate}, respectively. 

\noindent\textbf{Phase 2.} Given the estimated matrix $\hat W$
generated from Phase~1, in Lines~\ref{algo:initvhp}--\ref{algo:lsviend} of Algorithm~\ref{algo},
we implement value iteration to recover an optimal policy for the evaluation setting. 
In the optimality Bellman equation \cref{eq:fu-opt-bell}, we replace the true transition operator $\PP$ with the estimated transition operator induced by $\hat W$, i.e., 
$
\hat Q_{h} (x, a) = r_{h}(x,a) + (\hat\PP \hat V_{h+1}) (x, a),
$
for any $(x,a)\in \mathcal{S} \times \mathcal{A}$. Here, $\hat\PP$ is the operator form of $\cP_{\hat W}\defeq\cP_{\hat W \phi}$, such that $(\hat\P f)(x, a) = \EE_{x'\sim \cP_{\hat W }(\cdot \given x,a)}[f(x')]$ for any $f\colon \cS\to \RR$.

We remark that to efficiently implement the integration and maximization in Phase 2 of Algorithm~\ref{algo}, one can use Monte Carlo integration and gradient methods, respectively. 

\section{Theory}\label{sec:theory}

We first introduce two assumptions on the feature maps $\phi$ and $\psi$.
\begin{assumption}[Bounded feature maps]  \label{as:boundedfeature}
    We have $\| \phi(x,a)\|_2 \leq 1$ and $ \| \psi(z) \|_2 \leq 1$ for any $(x,a,z)\in \cS\times\cA\times \cZ$. 
    \end{assumption}



\begin{assumption}[Nondegenerate feature maps] \label{as:nonddualfeature}
It holds that $\operatorname{rank}(A)=d_\phi$ and $\operatorname{rank}(B)=d_\psi$ 
for $A$ and $B$ defined in \cref{eq:defAB}.
\end{assumption}

\textbf{Uniqueness of $\Wsad$.}\quad
\cref{as:nonddualfeature} implies the minimax problem \cref{eq:estimatorw} admits a unique solution.
In the min-max problem \cref{eq:estimatorw}, for a fixed primal variable $W$, 
the unique maximizer $K^*(W)$ of the inner problem in 
takes the form $
K^*(W)\defeq (WA^\top - C)B\inv$.
This holds by the invertibility of $B$, whose minimum eigenvalue is now denoted by $\mu_B \defeq \sigma_{\min}(B) > 0$. Plug in this optimal value we have $\max_K L(W,K) = \tfrac12 \tr[(WA^\top - C)B\inv (WA^\top - C)^\top]$. By full-rankness of $A$ we know $\Wsad$ is the unique minimizer of the map $W \mapsto \max_K L(W,K)$.

\textbf{Instrument Strength.}\quad
\ass~\ref{as:nonddualfeature} implicitly impose sufficient correlation between $\phi(x,a)$ and $\psi(z)$. In other words, IVs needs to be strong to have enough explanatory power
for the behavior policy $\pi_b$. 
Weak IV is a well-known pitfall in applied economic research \citep{angrist2008mostly}.
For RL applications with confounded data, practitioners should take into account domain knowledge
of the behavior policy to avoid using weak IVs. We introduce a quantity $\muIV$, which quantifies the strength of IVs. 
We define the IV strength $\muIV$ as follows, 
\#\label{eq:mupdef}
\muIV  \defeq \inf \Bigg\{  \frac{\| \Pi_\psi \cT f\|_{\ltz}^2 }{\|f \|_{\phi}^2 } \,\bigg|\, f\in \Hphi, \|f \|_{\phi}\neq 0  \Bigg\}
\;, 
\#
where $\Pi_\psi$ is the projection operator onto
the space $\Hpsi$, i.e., 
$\Pi_\psi u = \argmin_{u'\in \Hpsi} \| u - u'\|_{\ltz}^2$ for any $u \in \ltz$. 
The definition of $\muIV$ in \cref{eq:mupdef} mimics the
notion of \textit{sieve measure of ill-posedness}  well-known in the literature on
NPIV as a measure of IV strength \citep{blundell2007semi, chen2018optimal}. 
We next show $\muIV$ admits a simple expression.
\begin{proposition} \label{lm:mupisillposenss}
Let \ref{as:nonddualfeature} hold. Then $\muIV = \sigma_{\operatorname*{min}}(A^\top B\inv A)\,.$
\end{proposition}

\subsection{Parametric Case}

We impose the following assumptions on the transition function $F^*$ and the conditional expectation operator $\cT$.  

\begin{assumption}[Linear representation] \label{as:spec} 
It holds $F^*=W^*\phi$ for some $W^* \in \R^{d_x \times d_\phi}$.
\end{assumption}
Such a linear form of the transition function $F^*$ is commonly assumed in the literature \citep{sham2020information, mania2020active} in 
the context of dynamical system identification. 
\begin{assumption}[Realizability] \label{as:dualapprox}
    For all $f\in \Hphi$, it holds that  $\cT f\in \Hpsi$.
\end{assumption}

\begin{proposition} \label{prop:wstartequalwsad}
    Let \ref{as:nonddualfeature}, \ref{as:spec} and \ref{as:dualapprox} hold. Then $W^* = \Wsad$.    
\end{proposition}

One important contribution of our work is that we quantify how the strength of the IV is playing a role in terms of recovering optimal policy from confounded data. 
We provide a sketch of the proof for Theorem~\ref{thm:convergence} in \cref{sketch}.
The complete proofs are given in \cref{pf:thm:convergence}.

\begin{theorem}[Parametric case] \label{thm:convergence} 

Let \ref{as:iiddata}--\ref{as:dualapprox} hold. 
There exists a choice for stepsizes in Algorithm~\ref{algo} of the form $\etathet={\beta}/{(\gamma+t)}$ and $\etaomet =\alpha \etathet$ for any $t\in [T]$, where 
$\alpha=c_1 \muIV\inv \muB^{-1.5}$, $\beta = c_2\muIV\inv$, $\gamma = c_3 \muIV^{-4}\muB^{-3.5}$, and $c_1, c_2, c_3$ 
are positive absolute constants, such that

    (i) the estimation error satisfies
     \# \label{eq:est_error}
    \E \big[\|W_T - W^* \|_F^2\big] \leq \frac{\nu}{\gamma + T}
    \;,
    \#
    where $\nu =  \max \{\gamma\Tilde{P}_0,\,c_4 \muIV^{-4}\muB^{-2.5} \cdot d_x \sigma^2\}$ and $\Tilde{P}_0 = \|W_0 - W^* \|_F\sq + \sqrt{\muB}\cdot \|K_0-K^*(W_0) \|_F\sq$ with $c_4$ being a positive absolute constant. And

    (ii) the planning error satisfies
    \# \label{eq:planning_error}
    \E \big[ \|V_1^* - V_1^{\hat \pi} \| _\infty \big] \leq  H \cdot \min \bigg \{ 2H\sigma\inv \sqrt{\frac{\nu}{\gamma + T}},1   \bigg\}
    \;.
    \#
    The expectation is taken over the data.
\end{theorem}

For an appropriately chosen initial estimates $W_0$ and $K_0$, Theorem~\ref{thm:convergence} shows that the sample complexity 
needed to recover an $\epsilon$-optimal policy using observational data is of
order \$O(\muIV^{-4}\muB^{-2.5}\cdot H^4d_x \sigma \sq \epsilon^{-2})
\;,\$ where 
$\muIV$ characterizes IV strength, 
i.e., how well the IV is able to explain the behavior policy,  
$\mu_B$ quantifies the compatibility of the dual feature map and the IV, 
$H$ is the horizon of the MDP, and $d_x$ is the dimension of states. 
To the best of our knowledge, this is the first sample complexity 
result for recovering optimal policy using confounded data when a valid IV is present.

\begin{remark}[Joint computational and statistical efficiency] \label{rm:oneovertrate} 
The estimation procedure (phase 1) is readily a scalable algorithm, in contrast to estimators defined as the saddle-point of a finite-sum; see Remark~\ref{rm:savssaa}.
From an optimization perspective, 
the saddle-point problem \cref{eq:theminimax} is a stochastic convex-strongly-concave one, a case rarely investigated in the optimization literature.
The asymmetric structure in the primal and dual variables demands more detailed analysis of the algorithm in order to achieve a fast $O(1/T)$ rate.

We now review literature that studies convex-strongly-concave (CSC) stochastic saddle point problem. A slow rate $O(1/\sqrt{T})$ is obvious by the results for general stochastic convex-concave problem \citep{nemirovski2009robust}. The work of \cite{chambolle2011first} studies deterministic CSC problem with bilinear coupling and shows a $O(1/T\sq)$ rate. \citet{wang2017exploiting,du2017stochastic, du2019linear} consider CSC problem with finite sum structure and bilinear coupling structure, and shows a linear convergence rate by variance reduction techniques. In contrast, our algorithm solves stochastic CSC problem with linear coupling structure with a fast $O(1/T)$ rate without the need of projection. Moreover, the assumption of bounded variance of the stochastic gradient does not hold in our case, rendering most existing analysis invalid.

\end{remark}

\begin{remark}[Dependence on IV strength] 
In \cref{eq:planning_error}, for appropriately chosen 
initial estimates $W_0$ and $K_0$, only the second term in the
definition of $\nu$ matters. We are effectively solving $d_x$ NPIV problems, and 
the asymptotic order for solving just one NPIV problem is $O({\muIV^{-4}\muB^{-2.5}\sigma^2}{T\inv})$. 
The dependence on the dimension of feature maps
$d_\phi$ and $d_\psi$ is hidden in the minimum eigenvalues $\mu_B$ and $\muIV$. 
We compare our result with the work by
\citet{dikkala2020minimax} under \ref{as:dualapprox}.
There the proposed estimator is the saddle-point
of the sample version of \cref{eq:theminimax}; see Remark~\ref{rm:savssaa} for more details. 
In particular, they provide a bound in the $L^2$-norm,
and the order of the variance term is $O(\muIV^{-4} \max\{ d_\phi,d_\psi\} T^{-1})$~ \footnote{In Appendix D of \citet{dikkala2020minimax},
their $(\gamma_n, k_n,m_n)$ is the same as our $(\muIV\inv, d_\phi, d_\psi)$.}. 
The minimax optimal rate for NPIV problem is established 
in the work of \citet{blundell2007semi}, attained by sieve estimators. 
In comparison, the variance term in the minimax optimal rate is 
of order $O \big(\tilde\muIV^{-2} d_\psi T^{-1}\big)$~\footnote{In Theorem 2 of \citet{blundell2007semi}, their $(\tau_n,k_n)$ is the same as  our $(\tilde \muIV\inv, d_\phi)$.}, 
where $\tilde \muIV$ is the minimum nonzero singular value of $D^{-1/2}AB^{-1/2}$, quantifying the strength of an IV in a similar way to our $\muIV$.
\end{remark}

\begin{remark}[Dependence on horizon and state dimension] 
The work of \cite{sham2020information} provides a $\sqrt{T}$-regret bound for 
online learning of an additive nonlinear dynamics. Their regret bound
translates to a $O(d_\phi(d_\phi+d_x+H)H^3 \epsilon^{-2})$ sample complexity bound,
ignoring logarithmic factors; see Corollary~3.3 of \cite{sham2020information}.
Despite that we deal with confounders in 
additive nonlinear dynamics, our dependence on $d_x$ and $H$ 
matches their sample complexity bounds.
\end{remark}

\begin{remark}[Stochastic approximation for instrumental variables] \label{rm:savssaa}
    Our stochastic approximation (SA) estimation procedure is in contrast with the empirical saddle-point estimator proposed in \cite{dikkala2020minimax}. To estimate $f^*_j$, their estimator would be defined as the solution to the finite-sum saddle-point problem
    \# \label{eq:empsaddle}
    \argmin_{f\in \Hphi} \max_{u\in \Hpsi} \frac{1}{n}\sum_{i=1}^n\Big\{ (f(x_i,a_i) - x_{i,j}')u(z_i) + \frac{1}{2} u(z_i)^2 \Big\} - \frac{\lambda}{2} \|u \|_{\Hphi}\sq + \frac{\mu}{2} \| f\|_{\Hpsi}\sq
    \#
    for some positive $\lambda$ and $\mu$. Here the data $\{ x_i,a_i,z_i,x'_i\}$ are i.i.d.\ draws from $\bar d_{\pi_b}$, and $x_{i,j}'$ denotes the $j$-th coordinate of $x_i'\in \R^{d_x}$. Their procedure faces two challenges: (i) choosing the correct regularization parameter, and (ii) finding an approximate solution of the convex-concave optimization problem \cref{eq:empsaddle}, which requires a separate discussion of computational complexity. The theoretical trade-off among regularization bias, statistical error and optimization error is unclear, as is shown in related primal-dual methods in RL; see, e.g., \cite{dai17a,pmlr-v80-dai18c,nachum2019dualdice}. In contrast, the SA approach considered in this work tackles computational error and statistical error jointly and enjoys a fast rate of $O(1/T)$.
    \end{remark}

\subsection{Nonparametric Case}
\label{sec:nonparacase}

In \ref{as:spec}~and~\ref{as:dualapprox} we make the simplifying assumption that both the true transition function $F^*$ and the image of the operator $\cT$ lie in some known finite dimensional spaces. To extend out theory to the nonparametric case (e.g., $F^*:\cS\times\cA\to\R^{d_x}$ is H\"older continuous, and functions of the form $\{\cT f\given f:\cS\times \cA\to\R \text{, bounded and continuous} \}$ are also H\"older continuous), we need to discuss two issues. 

The first one is identification: whether $F^*$ is the unique solution to the CMR \cref{eq:thecondmoment}. Identification in NPIV usually requires some form of completeness assumptions. For example, bounded completeness condition is a relatively weak regularity assumption on the average visitation distribution $\bar d_{\pi_b}$. For two random variables $X$ and $Y$, $X$ is boundedly complete w.r.t\ $Y$ if for all $Y$-a.s.\ bounded function $f$, it holds $\E[f(Y)\cond X]=0$ implies $f=0$ $Y$-a.s. Intuitively, it requires that the distribution of $Y$ exhibits a sufficient amount of variation when conditioning on different values of $X$. It is well-known that there is a wide range of distributions that satisfy bounded completeness; see, for example, \cite{blundell2007semi, d2011completeness,Hu2017,Andrews2017}.

The second issue is the error caused by finite-dimensional approximation which we address below.
Let $f^*$ be one element of $F^*=[f^*_1,\dots, f^*_{d_x}]^\top$. If \ref{as:spec} is violated, we define the primal approximation error 
$$
\eta_1 \defeq \| f^* - \Pi_\phi f^*\|_{\ltxa}.
$$
If \ref{as:dualapprox} is violated, we define the dual  error, which characterizes how well the dual function space $\Hpsi$ approximates functions of the form $\cT(f-f^*)$ for $f\in \Hphi$. Formally we define 
$$
\eta_2 \defeq \sup_{} \{ \| \cT f - \Pi_\psi \cT f \|_{\ltz} : f\in \Hphi, \| f \|_{\ltxa} \leq 1 \}.
$$
Obviously \ref{as:spec} implies $\eta_1=0$ and \ref{as:dualapprox} implies $\eta_2=0$.

We show that, when \ref{as:spec} and \ref{as:dualapprox} are violated, the difference between $F^*$ and $W^\sad \phi$ has only linear dependence on the approximation errors $\eta_1$ and $\eta_2$. Notably, the dual error $\eta_2$ is inflated by $\muIV\inv$.
Recall $f^*_i$ is the $i$-th element of $F^*=[f^*_1,\dots, f^*_{d_x}]^\top$, and $W^{\sad}_i $ is the $i$-th row of the estimand $\Wsad$ defined in \cref{eq:estimatorw}. 

\begin{theorem}[Nonparametric case] \label{lm:misspecification}
    Let \ref{as:iiddata}--\ref{as:nonddualfeature} hold.
    Assume there is a constant $c>0$ such that $
    \muIV \inv \cdot \|\cT(f^*_i - \Pi_\phi f^*_i )\|_{\ltz} \leq c\cdot \|f^*_i - \Pi_\phi f^*_i \|_{\ltxa}$.
    We define the operator $Q_{}\colon\ltxa \mapsto\Hphi$, $Q_{} f = \argmin_{f'\in\Hphi} \| \Pi_{\psi}\cT(f'-f)\|_\ltz$. Let $\mu = \|(\Pi_\phi - Q_{})f^*_i \|_\ltxa$. It holds 
    \$
   \| f^*_i - W^{\sad}_i \cdot \phi\|_\ltxa \leq (1+2c) \cdot \eta_1 + \muIV \inv \cdot \mu \cdot \eta_2\,.
    \$
\end{theorem}
In Theorem~\ref{lm:misspecification}, the existence of such a constant $c$ is called the stability assumption; see \cite{blundell2007semi} and \cite{ chen2018optimal} for a detailed discussion. 
Note the dual approximation error $\eta_2$ is inflated by a factor of $\muIV\inv$.

The estimation phase still produces an estimator that converges to $\Wsad$ at $O(1/T)$ rate. 
The only difference is, in the planning phase, we are performing value iteration with a biased model.


\section{Experiment}\label{sec:experiment}
In this section, we present numerical experiments on the parametric and nonparametric cases described in Section~\ref{sec:theory}. The goal of this section is to verify that Algorithm~\ref{algo} successfully identifies the transition model based on sequential observational data and recovers the optimal policy by planning with the estimated transition model. Importantly, we aim to quantify how the strength of instrument affects estimation of causal quantities in the sequential setting.

All experiments in this section can be reproduced with the code at \url{https://github.com/ChampionRecLuse/ivvi}.

\subsection{Parametric Setting}

We consider data generation procedures with a linear transition dynamic. {We use our algorithm in 5-dimension, 10-dimension, and 20-dimension. In order to compare results in different dimensions, we use similar parameter setting in different dimensions. To be specific, we let $k$ stands for the number of dimensions. The CMDP-IV operates in the following spaces: $\mathcal{S} = \mathcal{U} = \mathcal{Z} = \mathbb{R}^{k}$, $\mathcal{A} = [-1,1]^{k}$. We generate 80 episodes, where each episode has a horizon $H = 1000$. The environment starts from an all-ten vector $x_{0}$, for example, $x_{0} = [10,10,10,10,10]^{\top}$ when $k=5$. The environment generates a sequence of UCs $\{e_{h}\}_{h} \stackrel{ \iid} {\sim} \mathcal{P}_{e}$ and a sequence of observable IVs $\{z_{h}\}_{h}\stackrel{ \iid} {\sim} \mathcal{P}_{z}$. We let $e_{h}$ and $z_{h}$ follow Gaussian distributions, i.e., $\mathcal{P}_{e} = \mathcal{N}(\mu_{e},\Sigma_{e})$, $\mathcal{P}_{z} = \mathcal{N}(\mu_{z},\Sigma_{z})$. At the $h$-th step, the next state $x_{h+1}$ is generated by the equations
\$
a_h  \sim \operatorname{Proj}_{[-1,1]^{k}} \big(\mathcal{N}(z_{h}+e_{h},\Sigma_{a})\big) \;,\quad
x_{h+1}=F^*(x_h, a_h) + e_h = Px_{h}-Qa_{h} + e_h\;,
\$
where 
\begin{gather*}
    \Sigma_{a} =  \begin{pmatrix}
1       &    0.3    & 0 & \ldots  &0\\
0.3     &    1      & 0.3   &\ddots &\vdots\\
0       & 0.3 & \ddots &   \ddots  &0\\
\vdots  &\ddots  &\ddots      &  1 &0.3\\
0  &\ldots  &0 &0.3  & 1
\end{pmatrix},
\quad
P = \begin{pmatrix}
0.5       &    0.2    & 0 & \ldots  &0\\
0.2     &    0.5      & 0.2   &\ddots &\vdots\\
0       & 0.2 & \ddots &   \ddots  &0\\
\vdots  &\ddots  &\ddots      &  0.5 &0.2\\
0  &\ldots  &0 &0.2  & 0.5
\end{pmatrix},\\[1em]
Q = \begin{pmatrix}
0.5       &    0.1    & 0 & \ldots  &0\\
0.1     &    0.5      & 0.1   &\ddots &\vdots\\
0       & 0.2 & \ddots &   \ddots  &0\\
\vdots  &\ddots  &\ddots      &  0.5 &0.1\\
0  &\ldots  &0 &0.1  & 0.5
\end{pmatrix}.
\end{gather*}
Note that each matrix above has size $k\times k$.} Here $a_h$ is generated by first sampling from a Gaussian and then projecting it onto the cube $[-1,1]^k$ w.r.t.~the Euclidean distance. The projection here is mainly to stabilize the state dynamics of the behavior policy. The reward function is $r_{h}(x_{h},a_{h})=-0.05\lVert x_{h}\rVert^{2}$. We can tell the optimal policy easily under this reward. The optimal policy forces the agent to get close to zero vector since this state has the highest reward. We use different covariance matrices $\Sigma_{e}$, $\Sigma_{z}$ to control the instrument strength (defined in Section~\ref{sec:theory}) and study its influences on our algorithm. The intuition is that if the Gaussian distribution of UCs is steeper or the Gaussian distribution of IV is flatter, we can have a stronger instrument strength. In each dimension, we consider three data generation processes, all of which correspond to an identical transition function. In the experiments, we use the following pairwise covariance matrices: 
$$
\text{I) }\Sigma_{z}=0.7I_{k} \text{ and } \Sigma_{e}=\Sigma_{a}; 
\quad\text{ II) }\Sigma_{z}=I_{k} \text{ and } \Sigma_{e}=\Sigma_{a}; 
\quad\text{ III) }\Sigma_{z}=1.5I_{k} \text{ and }\Sigma_{e}=\Sigma_{a}. 
$$ 
In the above setup, it is obvious that $z_{h}$ is a valid IV and $e_{h}$ is a valid UC in the environment.

In order to check the robustness of the proposed method, we introduce variations in the covariance matrices $\Sigma_{a}$ and $\Sigma_{e}$, corresponding to the action generation process and error term, respectively. This also helps to check the robustness of our method when the transition kernel is multivariate normal with general covariance matrices. We construct symmetric Toeplitz matrices where each sub-diagonal has constant values that decay exponentially as the covariance matrices. A symmetric Toeplitz matrix $T(\rho) \in \mathbb{R}^{k\times k}$ with parameter $\rho$ is defined as $T_{ij}=\rho^{|i-j|}$, where $T_{ij}$ is the $i,j$ element of $T(\rho)$ and $\rho$ is the decay parameter. In the experiment, we use the following pairwise covariance matrices:
\begin{align*}
    \text{IV) }&\Sigma_{z}=0.7I_{k}, \Sigma_{e}=T(-0.5) \text{ and } \Sigma_{a}=T(-0.3);\\
\quad\text{ V) }&\Sigma_{z}=I_{k}, \Sigma_{e}=T(0.1) \text{ and } \Sigma_{a}=T(0);\\
\quad\text{ VI) }&\Sigma_{z}=1.5I_{k}, \Sigma_{e}=T(0.5) \text{ and } \Sigma_{a}=T(0.3). 
\end{align*}
It is important to note that we collect data based on these settings and use our method to estimate the transition function. The training and testing environments are consistent with settings I, II, and III.

We instantiate Algorithm~\ref{algo} with feature maps $\phi(x,a) = [1,x,a]^{\top}$ and $\psi(x,z)=[1,x,z]^{\top}$. Note that the current state $x$ is also a variable in the feature map $\psi$. According to the feature map, the true transition function can be written as $W^{*}\phi(x,a)$, where 
\addtocounter{MaxMatrixCols}{10}
$
W^{*} = \begin{bmatrix} 
	\mathbf{0} & P & -Q\\
	\end{bmatrix}\,.
$
We use the minibatch stochastic gradient descent ascent to compute $\Wsad$. We set initial estimates $K_0$ as a zero matrix and $W_0$ as a matrix where every entry is one fifth. {For 5-dimension and 10-dimension case, at $t$-th iteration, stepsizes we use are $\eta_t^\theta=0.05+\frac{1}{18+t}$ and $\eta_t^\omega=\frac{1}{18+t}$. For 20-dimension case, stepsizes are $\eta_t^\theta=0.06+\frac{1}{18+t}$ and $\eta_t^\omega=\frac{1}{18+t}$.} The estimated transition dynamic can be expressed as $W^T\phi(x,a)$, where $W^T$ is the last iterate.

In Phase 2, we use the SPEDE algorithm \citep{ren2021free}, which is a planning algorithm based on the observation that under Gaussian noise, the linear spectral feature of the corresponding Markov transition operator can be obtained in a closed form. Moreover, SPEDE is suitable for a continuous state and action space. Note that in the original implementation, SPEDE samples transition functions from its posterior distribution at each episode, but, in our case, we do not need such a sampling.

We compare our method with a natural baseline: ordinary regression. For a fair comparison, we perform ordinary regression using the feature map $\phi(x,a)$. Let $J_{k}= \{ (x_h^k, a_h^k, x_{h+1}^k) \}_h$ be the trajectory that includes samples in the $k$-th episode. The baseline estimator for the transition model, $W^{\rm baseline}$, is defined as:
\$
W^{\rm baseline} := \argmin_{W} \Big\{\sum_{k=1}^{K}\Big(\sum_{h=0}^{H}\Vert W\phi(x_{h},a_{h})-x_{h+1}\Vert_{2}^{2}\Big)  \Big\}\,.
\$
Note that ordinary regression does not take the IV $z_{h}$ and the UC $e_{h}$ into consideration. We also use SPEDE as the planning component of the baseline algorithm.

\begin{figure}[p]
    \centering
    \subfigure[$k=5$]{\includegraphics[width=0.3\textwidth]{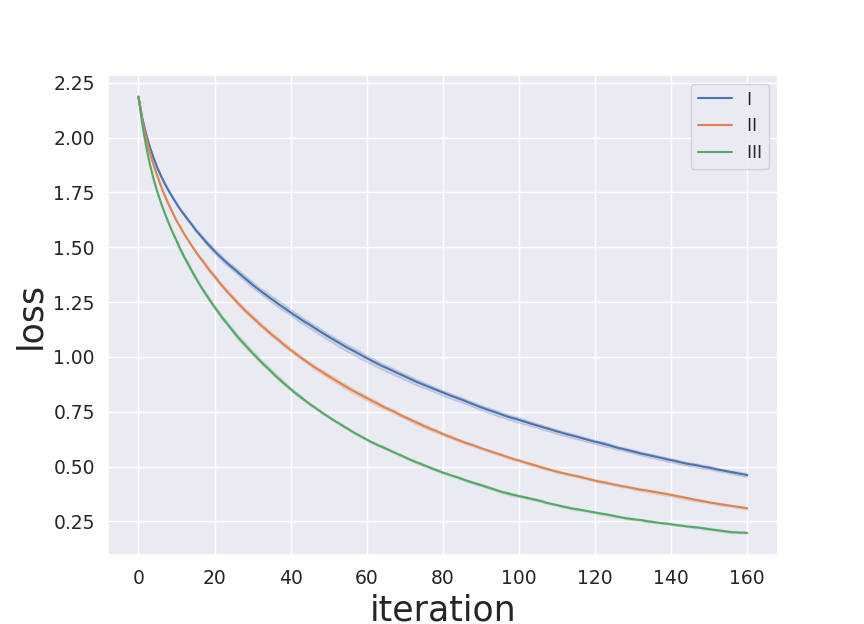}}
     \hfill
    \subfigure[$k=10$]{\includegraphics[width=0.3\textwidth]{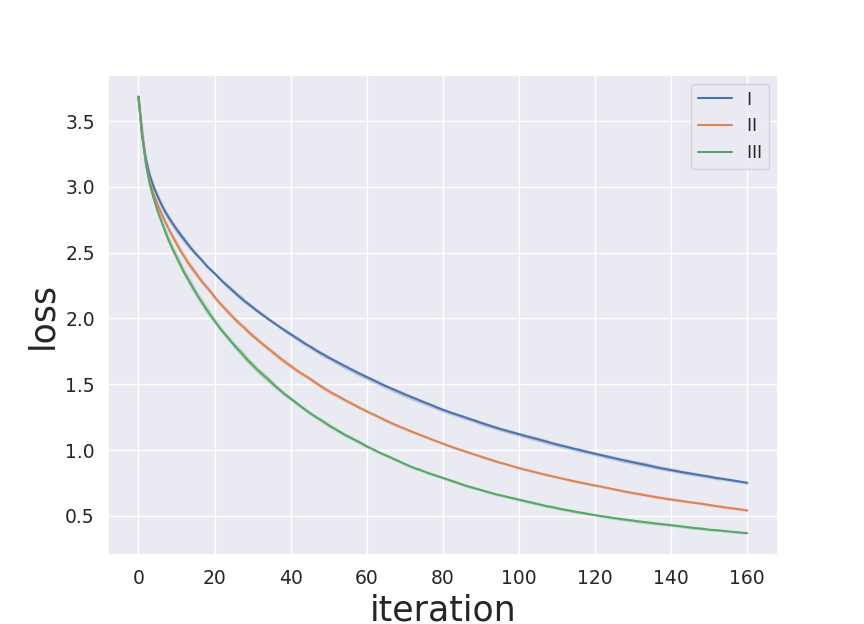}}
     \hfill
     \subfigure[$k=20$]{\includegraphics[width=0.3\textwidth]{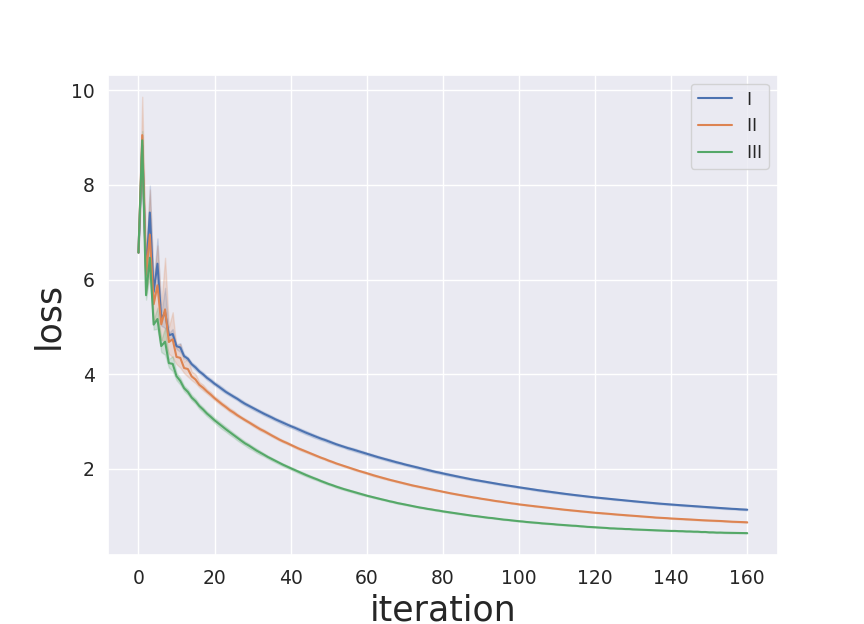}}
        \caption{Experiment results for the parametric setting. The gradient descent ascent loss $\|W^t - W^*\|_{F}$ for different settings of instrument strength under different dimensions}
        \label{fig:parametric-loss}
\end{figure}

\begin{table}[p]
\centering
\begin{tabular}{c *{3}{ccc}}
\toprule
\diagbox{index}{dimensions} & 5 & 10 & 20\\ \hline
\midrule
I    & 0.14  &  0.14  &  0.13    \\ 
II   & 0.19  &  0.18  & 0.17   \\ 
III  & 0.24  &  0.24  &  0.23  \\ 
IV    & 0.14  &  0.14  &  0.13    \\ 
V   & 0.18  &  0.18  & 0.18   \\ 
VI  & 0.24  &  0.24  &  0.23  \\ 
\bottomrule
\end{tabular}
\caption{Instrument Strength of different dynamics (parametric setting)}
\label{tbl:parametric-instrument strength}
\end{table}

\begin{table}[p]
\centering
\begin{tabular}{c *{3}{ccc}}
\toprule
\diagbox{index}{dimension} & 5 & 10 & 20 \\ \hline
\midrule
I                                                                                  & 4.06  &  4.53  &  4.90    \\ 
II                                                                                 & 3.90  & 4.47   & 4.86   \\ 
III                                                                                & 3.71  & 4.44   &  4.89  \\ 
IV                                                                                  & 7.33  &  4.03  &  4.94    \\ 
V                                                                                 & 7.00  & 4.34   & 5.44   \\ 
VI                                                                                & 6.23  & 3.94   &  4.70  \\ 
\bottomrule
\end{tabular}
\caption{Average computation time (in seconds) for estimation phase (parametric setting)}
\label{tab:computational time metrics under parametric setting}
\end{table}

\begin{table}[p]
\centering
\begin{tabular}{c *{3}{ccc}}
\toprule
\diagbox{index}{dimensions} & 5 & 10 & 20 \\ \hline
\midrule
Baseline & $590.19\pm 29.14$  & $832.92\pm 15.15$  & $2257.49\pm 134.16$  \\ 
I        & $14.84\pm 1.31$  & $173.98\pm 3.97$  & $505.63\pm 8.86$ \\ 
II       & $4.17\pm 1.18$  & $38.30\pm 3.04$   & $290.65\pm 9.08$  \\ 
III      & $1.44\pm 1.30$  & $29.52\pm 2.89$   & $117.20\pm 6.48$   \\ 
IV        & $17.21\pm 1.43$  & $100.79\pm 2.92$  & $422.14\pm 18.41$ \\ 
V       & $8.00\pm 1.04$  & $20.97\pm 3.72$   & $290.41\pm 14.28$  \\ 
VI      & $6.36\pm 1.02$  & $19.24\pm 3.71$   & $140.58\pm 4.23$   \\ 
\bottomrule
\end{tabular}
\caption{Regret of different dynamics (parametric setting)}
\label{tbl:parametric-regret}
\end{table}

\change{We show the results in figures and tables. Table~\ref{tbl:parametric-instrument strength} indicates the instrument strength of different dynamics. Table~\ref{tab:computational time metrics under parametric setting} shows the computational time of different dynamics for estimation phase with a sample size of 5. We can observe that the average computation time for estimating dynamics under parametric setting is short.} Figure~\ref{fig:parametric-loss} shows how the instrument strength affects the convergence rate of the gradient descent ascent in Phase 1. 
The loss here is equal to the difference between the true value $W^{*}$ and the estimated value $W^{t}$ after each iteration.
It is evident that with stronger instrument strength the loss decreases faster than with weaker instrument strength. 
This phenomenon is consistent with our theory and intuition. With appropriately chosen instrument variables with sufficient instrument strength, we can still identify transition dynamics with small samples. 
Figure~\ref{fig:parametric-reward} plots the reward and its 95\% confidence interval obtained by the SPEDE planning algorithm. The curve labeled 'opt' represents the policy derived from planning with the true underlying transition function.
We observe that the reward of baseline decreases during training due to a wrong transition dynamic estimated through ordinary regression. 
This shows that in the presence of UCs, not only does ordinary regression produce biased estimates of the transition model, but also such estimation error will be propagated to planning and further amplified due to the sequential nature of the problem, producing a poor policy.
Compared to baseline, our algorithm performs well in this case, and we summarize the results in Table~\ref{tbl:parametric-regret} with a sample size of 5.
We used the observations to compute uncentered covariance matrices, defined in \eqref{eq:defAB}, and instrument strength (IS), defined in \eqref{eq:mupdef}.
The regret is defined as the decrease in reward gained due to the execution of the policy produced by planning with the estimated transition function instead of the execution of the optimal policy produced by planning with the true transition function. 
In our experiment, we used the same online setting for different settings of IS, which allows us to compare the performance of the transition functions across different settings. 
We can observe that the regret for the baseline is large. 
It can also be seen that the influence of UCs is magnified when the dimensions increase. In low dimensions, the regret with weak instrument strength gets close to 0. In high dimensions, the regret gap between our algorithm and the optimal reward becomes larger because of not only the influence of UCs but also the larger reward at every step. However, it is obvious that the reward of all examples with sufficient instrument strength is still close to the optimal reward, which means the estimated transition function is a good estimator. In addition, our algorithm is better than the baseline. The robustness check experiment results is shown in Figure~\ref{fig:parametric-non-diagonal-result}. We observe that the reward of all examples with sufficient instrument strength is still close to the optimal reward under the general covariance matrices setting. This consistency suggests that the estimated transition function is an effective estimator, and the proposed method is robust.

\begin{figure}[t]
    \centering
\begin{minipage}{\textwidth} 
        \subfigure[$k=5$]{\includegraphics[width=0.3\textwidth]{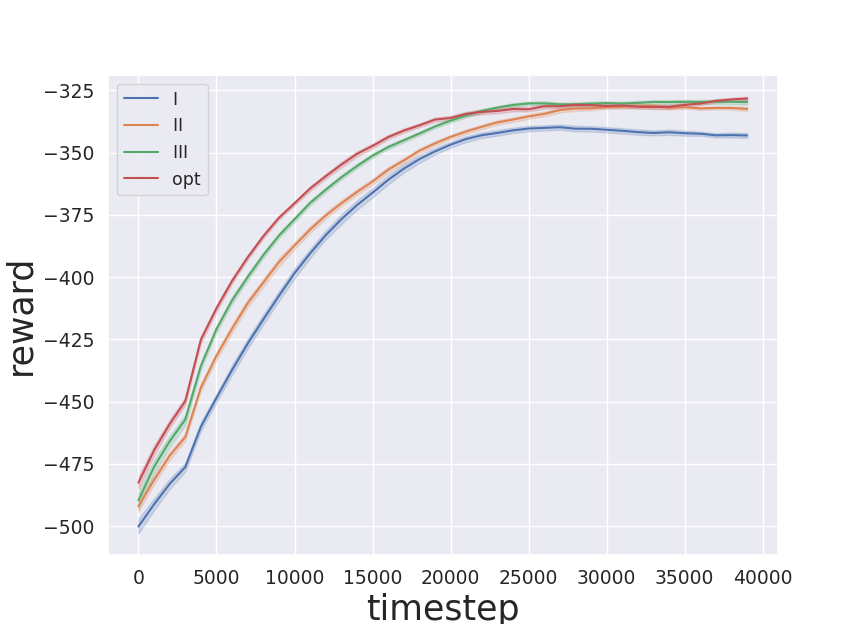}}
     \hfill
    \subfigure[$k=10$]{\includegraphics[width=0.3\textwidth]{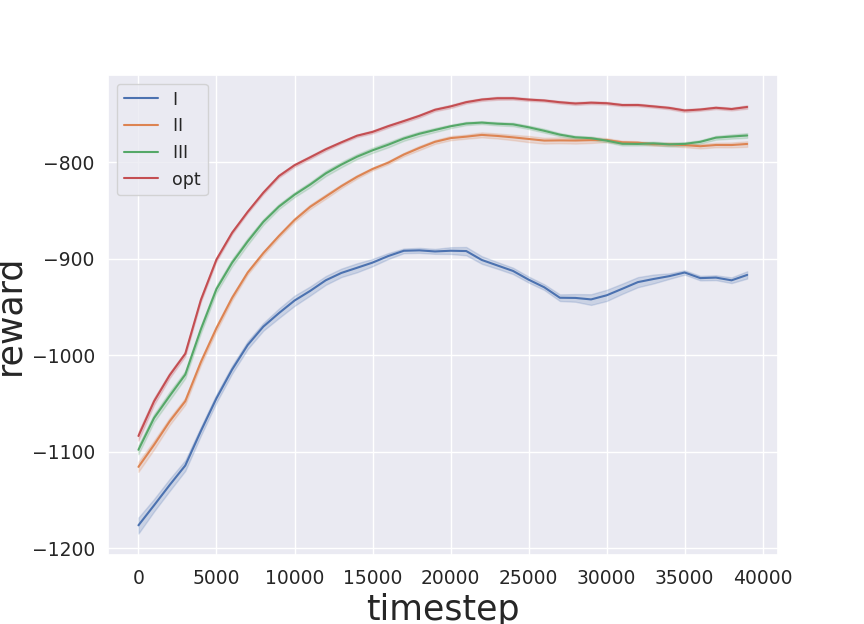}}
     \hfill
     \subfigure[$k=20$]{\includegraphics[width=0.3\textwidth]{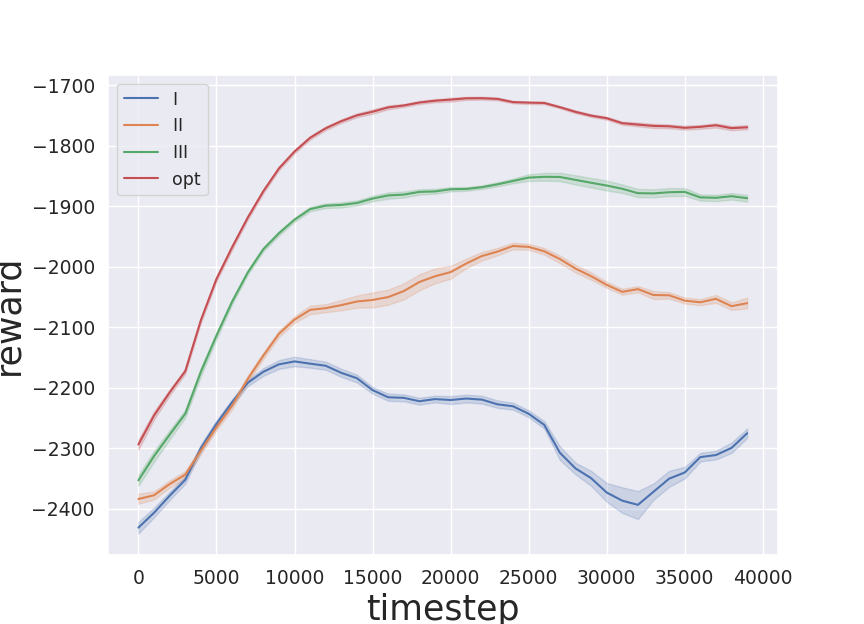}}
		\end{minipage}
\begin{minipage}{\textwidth} 
    \subfigure[$k=5$]{\includegraphics[width=0.3\textwidth]{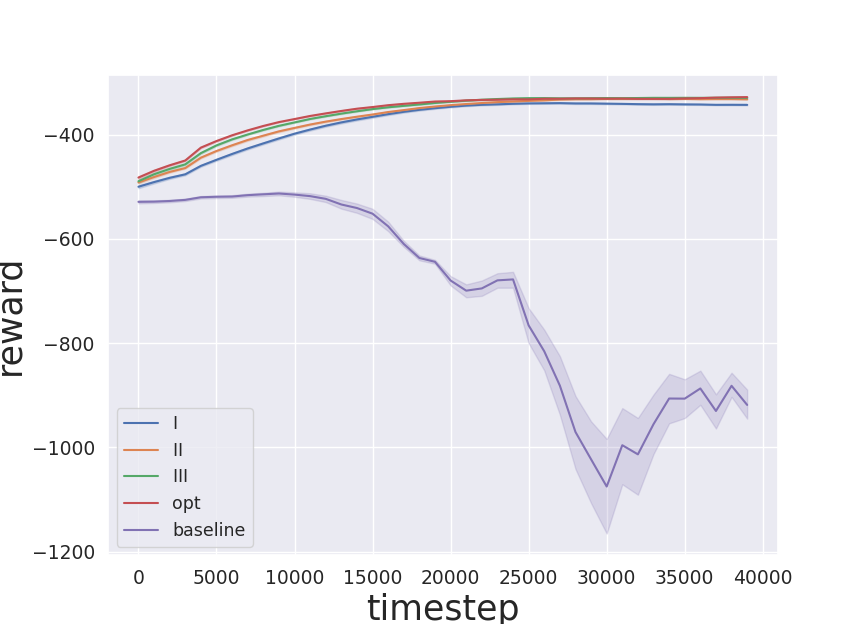}}
     \hfill
    \subfigure[$k=10$]{\includegraphics[width=0.3\textwidth]{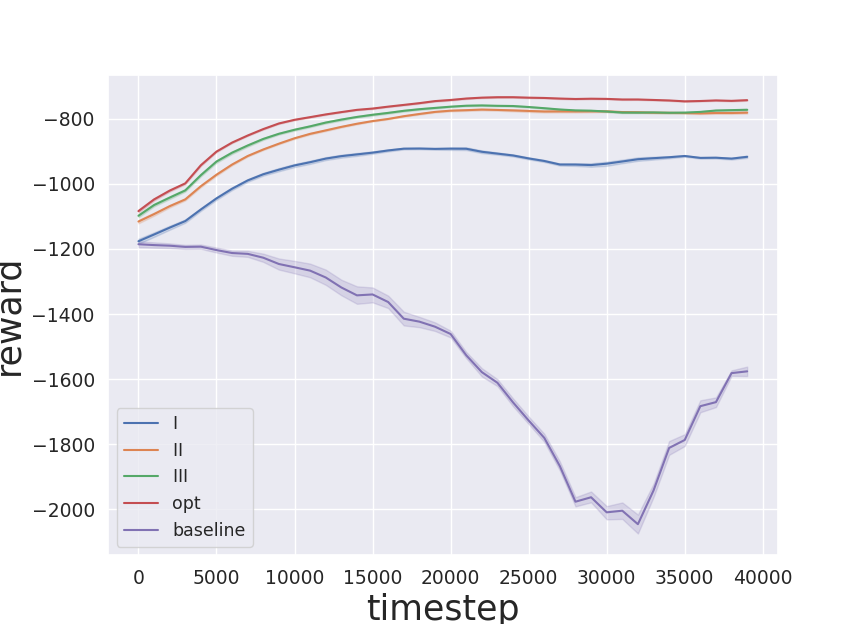}}
     \hfill
     \subfigure[$k=20$]{\includegraphics[width=0.3\textwidth]{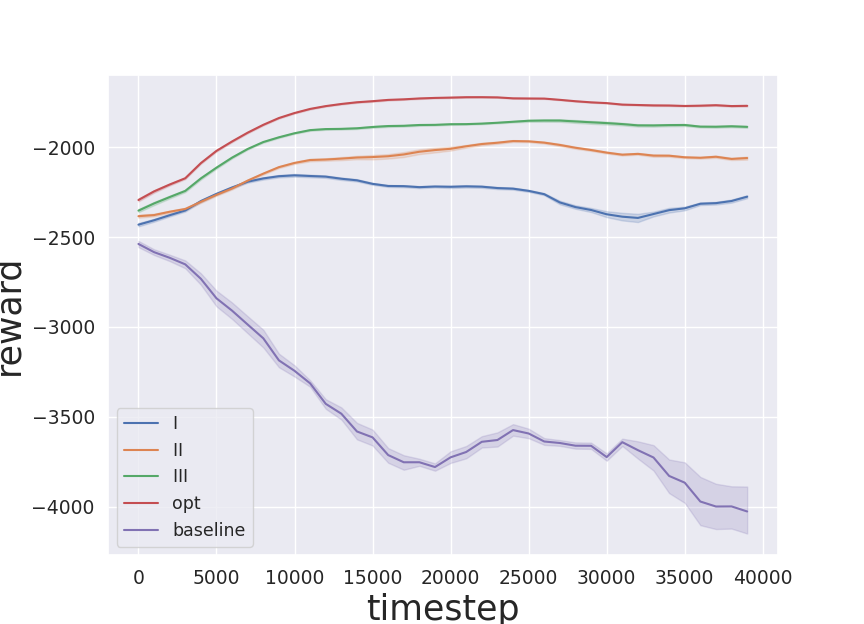}}
		\end{minipage}

    \caption{Experiment results for the parametric setting. Top panel: The performance curves (with 95\% confidence interval) of reward versus the time steps for different transition functions (without baseline). Bottom panel: The performance curves (with 95\% confidence interval) of reward versus the time steps for different transition functions (including the ordinary regression baseline). The time step is the episode for SPEDE.}
    \label{fig:parametric-reward}
        
\end{figure}

\begin{figure}[t]
    \centering
\begin{minipage}{\textwidth} 
        \subfigure[$k=5$]{\includegraphics[width=0.3\textwidth]{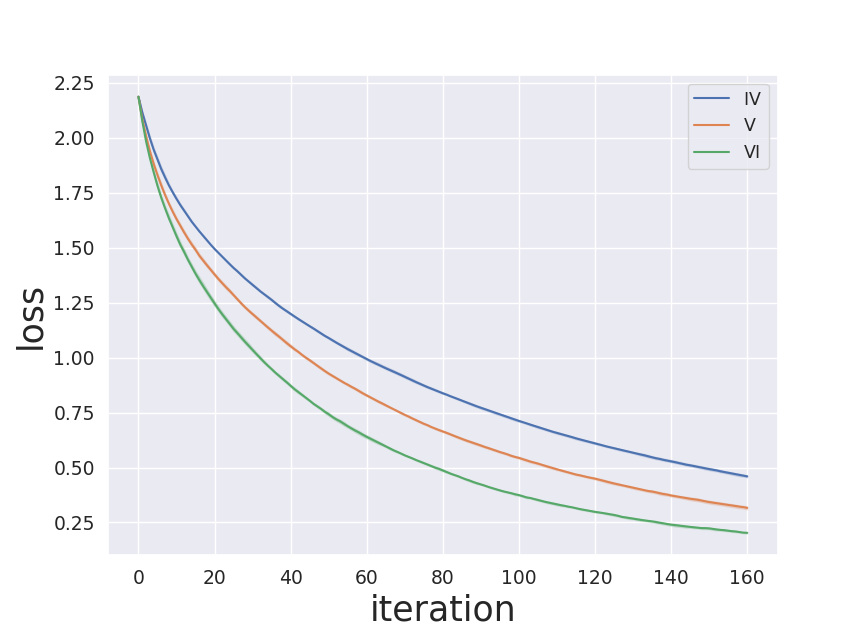}}
     \hfill
    \subfigure[$k=10$]{\includegraphics[width=0.3\textwidth]{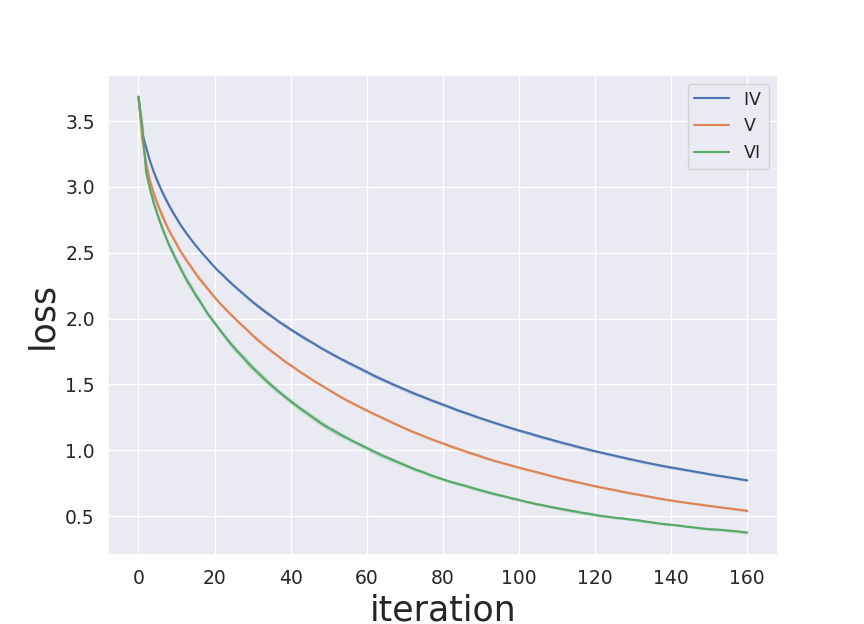}}
     \hfill
     \subfigure[$k=20$]{\includegraphics[width=0.3\textwidth]{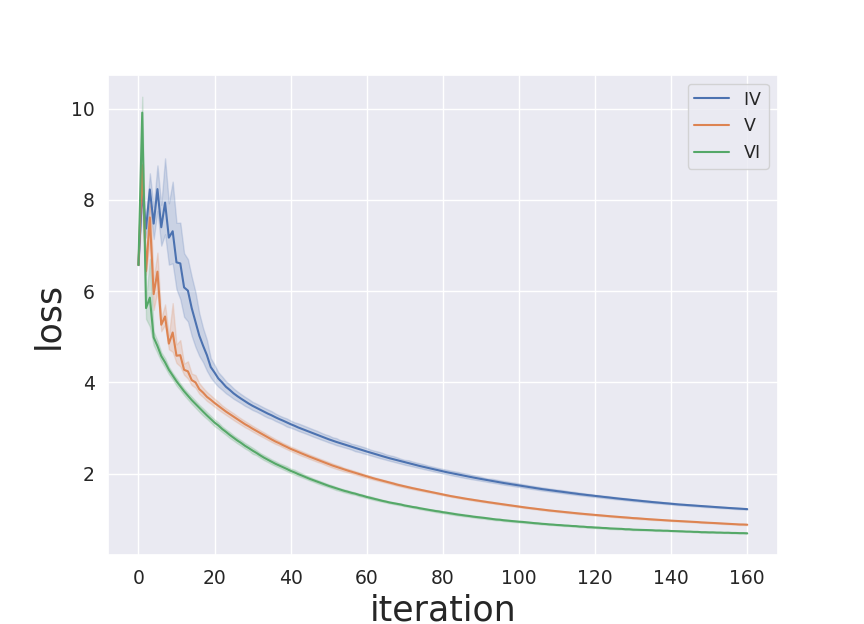}}
		\end{minipage}
\begin{minipage}{\textwidth} 
    \subfigure[$k=5$]{\includegraphics[width=0.3\textwidth]{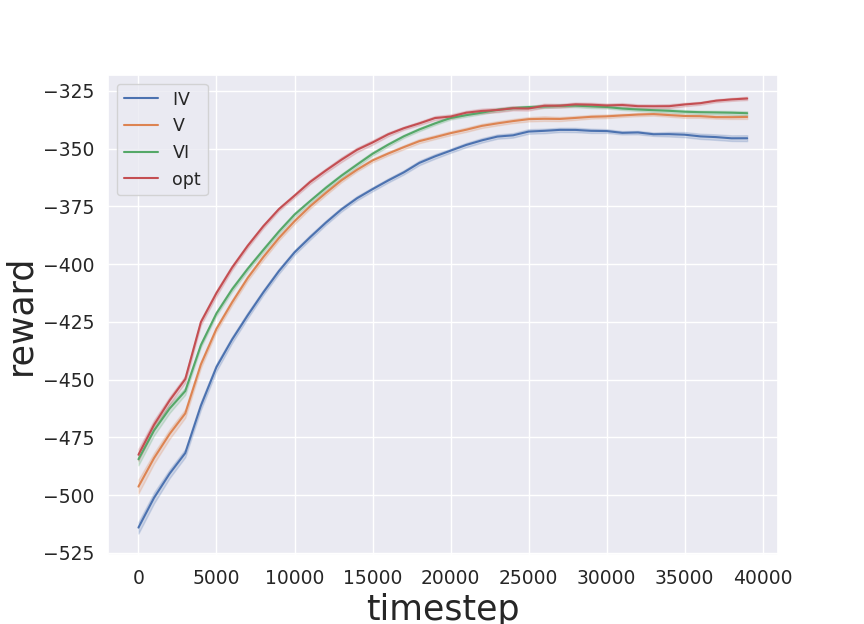}}
     \hfill
    \subfigure[$k=10$]{\includegraphics[width=0.3\textwidth]{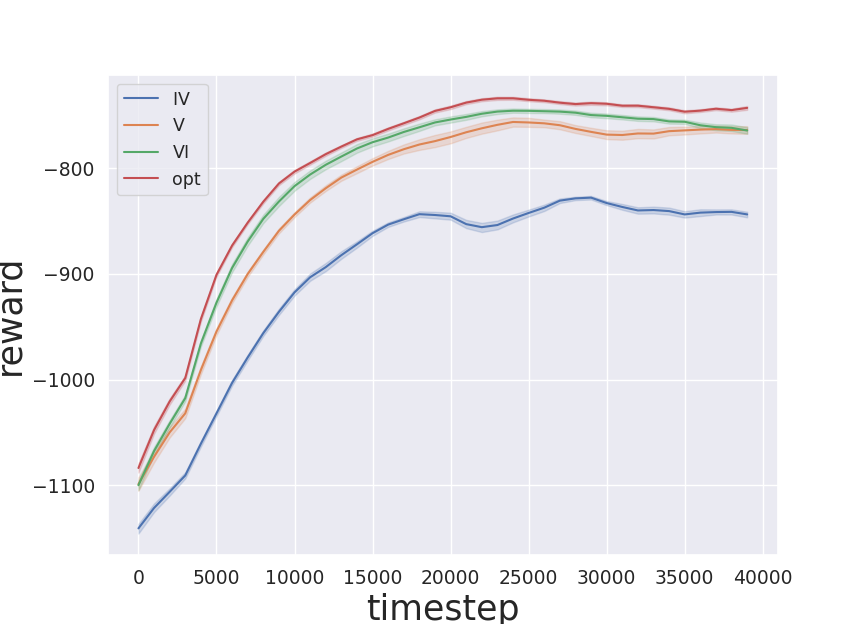}}
     \hfill
     \subfigure[$k=20$]{\includegraphics[width=0.3\textwidth]{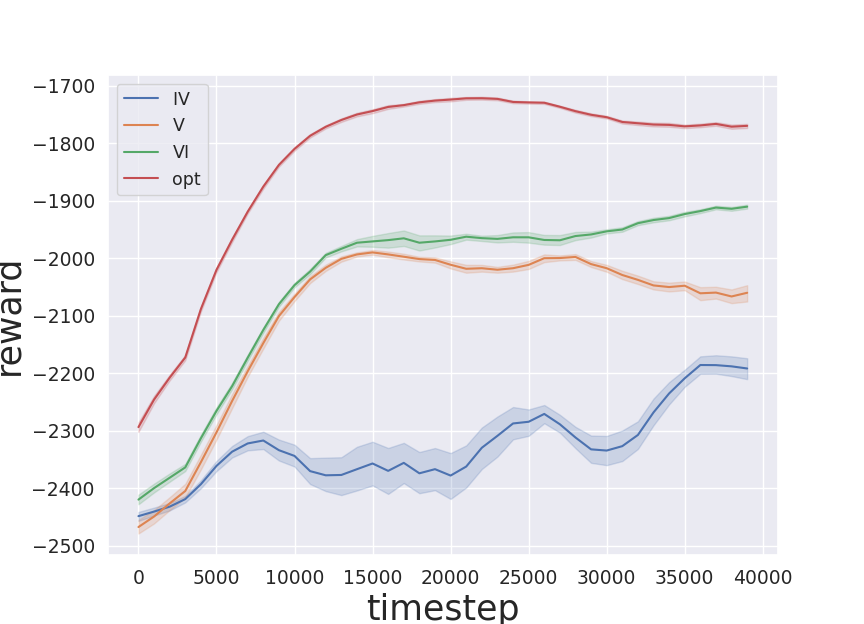}}
		\end{minipage}

    \caption{Robustness check experiment results for the parametric setting. Top panel: The gradient descent ascent loss $\|W^{t}-W^{sad}\|_{F}$ for different settings of instrument strength under different dimensions. Bottom panel: The performance curves (with 95\% confidence interval) of reward versus the time steps for different transition functions (without baseline). The time step is the episode for SPEDE.}
    \label{fig:parametric-non-diagonal-result}
        
\end{figure}

\subsection{Nonparametric Case}


\begin{figure}[p]
    \centering
\begin{minipage}{\textwidth} 
        \subfigure[$k=1$]{\includegraphics[width=0.3\textwidth]{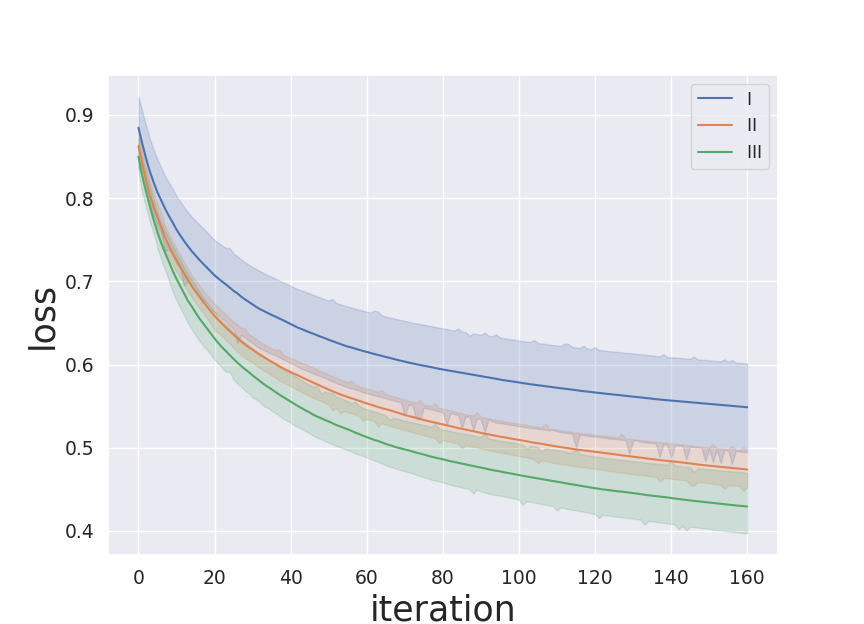}}
     \hfill
    \subfigure[$k=3$]{\includegraphics[width=0.3\textwidth]{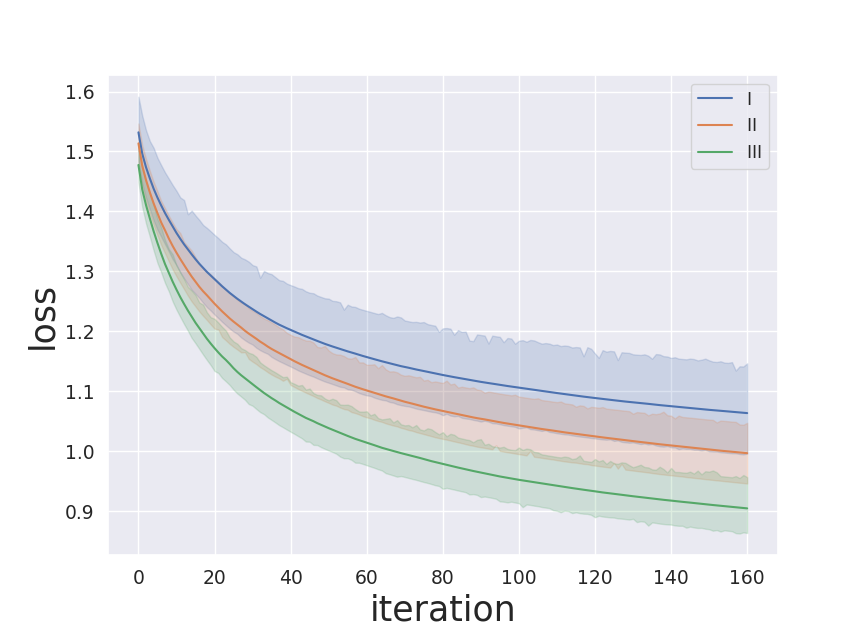}}
     \hfill
     \subfigure[$k=5$]{\includegraphics[width=0.3\textwidth]{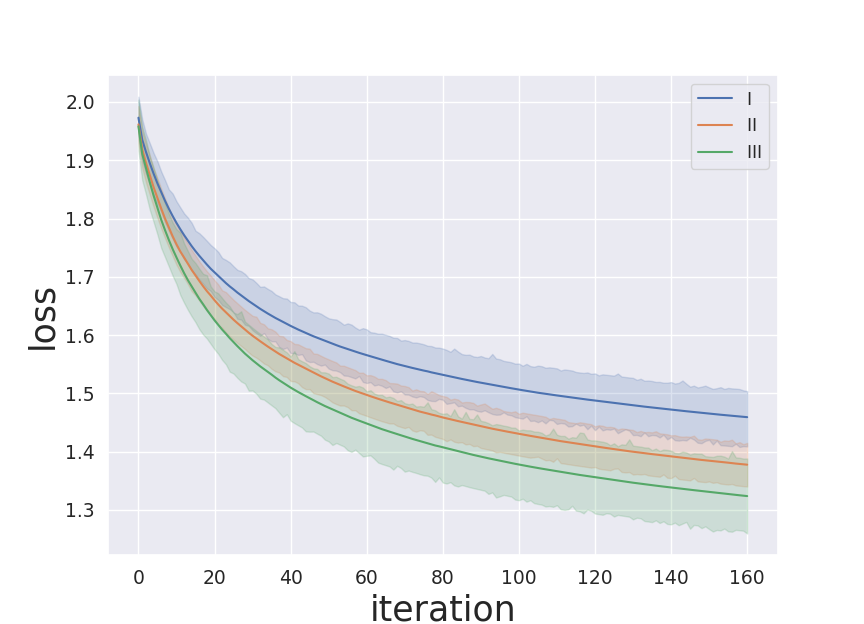}}
		\end{minipage}
\begin{minipage}{\textwidth} 
    \subfigure[$k=1$]{\includegraphics[width=0.3\textwidth]{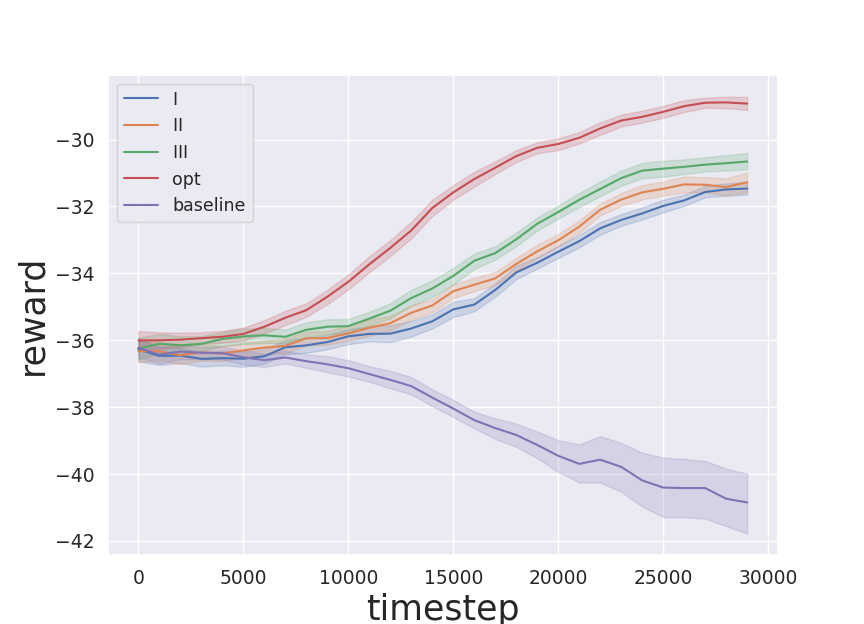}}
     \hfill
    \subfigure[$k=3$]{\includegraphics[width=0.3\textwidth]{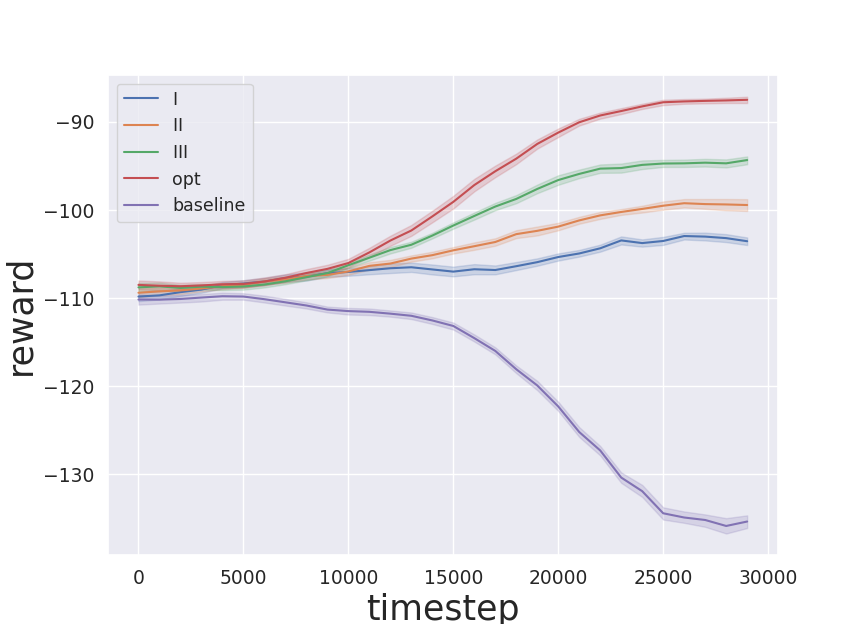}}
     \hfill
     \subfigure[$k=5$]{\includegraphics[width=0.3\textwidth]{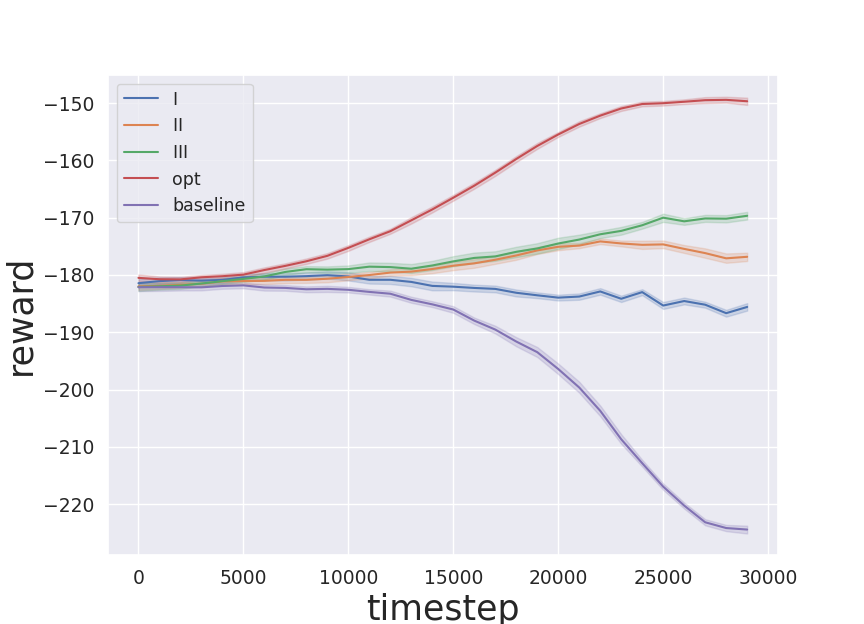}}
		\end{minipage}

    \caption{Experiment results for the nonparametric setting. Top panel: The gradient descent ascent loss $\|W^{t}-W^{sad}\|_{F}$ for different settings of instrument strength under different dimensions. Bottom panel: The performance curves (with 95\% confidence interval) of reward versus the time steps for different transition functions. The time step is the episode for SPEDE.}
    \label{fig:nonparametric-result}
        
\end{figure}

\begin{table}[p]
\centering
\begin{tabular}{c *{3}{ccc}}
\toprule
\diagbox{index}{dimensions} & 1 & 3 & 5 \\ \hline
\midrule
I        & 0.0024  & 0.0015   & 0.0011   \\ 
II       & 0.0044  & 0.0027   & 0.0019   \\ 
III      & 0.0064  & 0.0038   & 0.0028   \\ 
IV        & 0.0012  & 0.0007   & 0.0005   \\ 
V       & 0.0026  & 0.0015   & 0.0011   \\ 
VI      & 0.0041  & 0.0024   & 0.0017   \\ 
\bottomrule
\end{tabular}
\caption{Instrument Strength of different dynamics (nonparametric setting)}
\label{tbl:nonparametric-instrument strength}
\end{table}

\begin{table}[p]
\centering
\begin{tabular}{c *{3}{ccc}}
\toprule
\diagbox{index}{dimension} & 1 & 3 & 5 \\ \hline
\midrule
I                                                                                  & 5.56  & 4.75   & 4.05   \\ 
II                                                                                 & 5.38  & 5.08   & 3.15   \\ 
III                                                                                & 5.03  & 4.81   & 3.43   \\
IV                                                                                 & 7.23  & 8.05   & 5.45   \\ 
V                                                                                 & 6.68  & 8.43   & 5.78   \\ 
VI                                                                                & 7.12  & 7.21   & 6.23   \\
\bottomrule
\end{tabular}
\caption{Average computation time (in seconds) for estimation phase (nonparametric setting)}
\label{tbl:computational time metrics under nonparametric setting}
\end{table}

\begin{table}[p]
\centering
\begin{tabular}{c *{3}{ccc}}
\toprule
\diagbox{index}{dimensions} & 1 & 3 & 5 \\ \hline
\midrule
Baseline & $11.92\pm 0.91$  & $47.85\pm 0.77$ & $74.74\pm 1.01$ \\ 
I        & $2.54\pm 0.27$  & $16.04\pm 0.55$ & $35.90\pm 0.87$  \\ 
II       & $2.35\pm 0.35$  & $11.93\pm 0.82$ & $27.14\pm 0.77$  \\ 
III      & $1.73\pm 0.29$  & $6.83\pm 0.56$ & $19.98\pm 0.99$  \\ 
IV        & $3.22\pm 0.19$  & $13.25\pm 0.40$ & $27.29\pm 0.55$  \\ 
V       & $2.32\pm 0.16$  & $13.20\pm 0.46$ & $26.72\pm 0.53$  \\ 
VI      & $1.89\pm 0.22$  & $9.02\pm 0.41$ & $24.15\pm 0.61$  \\ 
\bottomrule
\end{tabular}
\caption{Regret of different dynamics (nonparametric setting)}
\label{tbl:nonparametric-regret}
\end{table}

We now consider examples with non-linear transition dynamics. In this case, we do not assume that the true transition function $F^{*}$ lies in a known finite-dimensional space. \change{We use our algorithm in 1-dimension, 3-dimension, and 5-dimension. In order to compare results in different dimensions, we use similar parameter setting in different dimensions. We let $k$ stands for the number of dimensions.} Specifically, we consider a CMDP-IV that operates in the following spaces: $\cS = \cU = \cZ = \mathbb{R}^{k}$, $\cA = [-1,1]^{k}$. We generate 80 episodes and each episode has a horizon $H=500$. The environment starts from an zero vector $x_{0}=\mathbf{0}$, generates a sequence of UCs $\{e_{h}\}_{h}  \stackrel{ \iid} {\sim} \mathcal{P}_{e}$ and a sequence of observable IVs $\{z_{h}\}_{h}\stackrel{ \iid} {\sim} \mathcal{P}_{z}$. As in the parametric setting, $e_{h}$ and $z_{h}$ are generated from different Gaussian distributions: $\mathcal{P}_{e} = \mathcal{N}(\mu_{e},\Sigma_{e})$, $\mathcal{P}_{z} = \mathcal{N}(\mu_{z},\Sigma_{z})$. \change{We use a non-linear transition function $F^*(x_h, a_h) = \ln(\lvert x_{h}-1\rvert +1)-\frac{a_{h}}{2}$.} At the $h$-th step, given the current state $x_{h}$, the behavior policy samples the action $a_{h}$, and then the next state $x_{h+1}$ is generated by the equations \change{
\$
a_h  \sim \operatorname{Proj}_{[-1,1]^{k}} \big(\mathcal{N}(z_{h}+e_{h},\Sigma_{a})\big)\;,\quad
x_{h+1}=F^*(x_h, a_h) + e_h = \ln(\lvert x_h-1\rvert +1)-\frac{a_{h}}{2} + e_h\;,
\$
where $\Sigma_{a}$ is a diagonal matrix with all diagonal elements equal to one half.} The action $a_h$ is generated by first sampling from a Gaussian distribution and then projecting the sample onto the interval $[-1,1]^{k}$. The projection is used to stabilize the state dynamics of the behavior policy. The reward function is $r_{h}(x_{h},a_{h})=-0.05\lVert x_{h}\rVert^{2}$. The optimal policy forces the agent to get close to zero vector, since this state has the highest reward. We use different variance $\Sigma_{e}$, $\Sigma_{z}$ to control the instrument strength and study its influences on our algorithm. In our experiment, we fix $\Sigma_{e}$ equals to identity matrix and use the following variances for $z_{h}$: 
$$
\text{I) }\Sigma_{z}=0.5I_{k};
\quad 
\text{II) }\Sigma_{z}=0.9I_{k}; 
\quad
\text{III) }\Sigma_{z}=1.5I_{k}. 
$$

Similar to the parametric setting, we introduce experiments which incorporate variations in the covariance matrices $\Sigma_{a}$ and $\Sigma_{e}$. In the experiment, we use the following pairwise covariance matrices:
\begin{align*}
    \text{IV) }&\Sigma_{z}=0.5I_{k}, \Sigma_{e}=T(-0.5) \text{ and } \Sigma_{a}=T(-0.3);\\
\quad\text{ V) }&\Sigma_{z}=0.9I_{k}, \Sigma_{e}=T(0.1) \text{ and } \Sigma_{a}=T(0);\\
\quad\text{ VI) }&\Sigma_{z}=1.5I_{k}, \Sigma_{e}=T(0.5) \text{ and } \Sigma_{a}=T(0.3). 
\end{align*}

We instantiate Algorithm~\ref{algo} with polynomial feature maps $\phi(x,a) = [1,x,a,x^2,a^2]^{\top}$ and $\psi(x,z) = [1,x,z,x^2,z^2]^{\top}$. Note that our transition function is a logarithmic function with some noise, so it does not lie in the finite-dimensional space spanned by the chosen feature maps. 
We use the minibatch stochastic gradient descent to compute $\Wsad$. We use zero matrices to initialize estimates $K_0$ and $W_0$. \change{At $t$-th iteration, stepsizes we use are $\eta_t^\theta=\frac{1}{16+t}$ and $\eta_t^\omega=\frac{1}{16+t}$.} The estimated transition function can be expressed as $W^T \phi(x,a)$. In Phase 2, we use the SPEDE planning algorithm. 
We compare our method with ordinary regression. For a fair comparison, we do ordinary regression on the feature map $\phi(x,a)$. We use the same baseline estimator as specified in the parametric setting. Note that ordinary regression does not take into account the IV $z_{h}$ and the UC $e_{h}$. We also use the SPEDE algorithm for planning.



\change{Table~\ref{tbl:nonparametric-instrument strength} shows the instrument strength of different dynamics. Table~\ref{tbl:computational time metrics under nonparametric setting} shows the computational time for estimation phase.} Figure \ref{fig:nonparametric-result} shows the result in the nonparametric setting. The top panel shows how the strength of the instrument affects the GDA convergence rate in Phase 1. We monitor the progress of GDA by plotting the difference between $\Wsad$, which is obtained in closed form in \eqref{eq:estimatorw}, and estimated $W^{t}$ after each iteration. We observe that the loss decreases faster with stronger instruments. The bottom panel shows the reward and its 95\% confidence interval obtained by the SPEDE planning algorithm. Similarly, the curve marked as ‘opt’ corresponds to the policy generated by planning with the actual underlying model. The reward of the baseline does not converge to the optimal reward. In the presence of UCs, the bias in the ordinary regression produces biased estimates whose error is propagated to planning and amplified due to the sequential nature of the problem, resulting in a low-performing policy. Table~\ref{tbl:nonparametric-regret} summarizes the results for a sample size of 5. We calculate instrument strength and regret in the same way as in the parametric setting, by computing the smallest nonzero eigenvalue of related covariance matrices. Compared to the baseline, our algorithm has a good performance that improves with the strength of the instrument.
There is a regret gap between our algorithm and the optimal regret, as the exact transition function can not be found in a finite-dimensional space.
This is predicted by \cref{lm:misspecification}. \change{We also observe that the regret gap between our algorithm and the optimal reward becomes larger as the dimension of space increases. This phenomenon is due to not only the larger reward at each step in higher dimensions but also the amplified influence of UCs, which lower the quality of the estimated function. However, it is obvious that the reward of all examples with sufficient instrument strength is still close to the optimal reward, which means the estimated transition function is a good estimator. We present the results of the robustness check experiments in Figure~\ref{fig:nonparametric-non-diagonal-result}. We observe that our algorithm is still better than the baseline, which indicates the robustness of the method.}

\begin{figure}[p]
    \centering
\begin{minipage}{\textwidth} 
        \subfigure[$k=1$]{\includegraphics[width=0.3\textwidth]{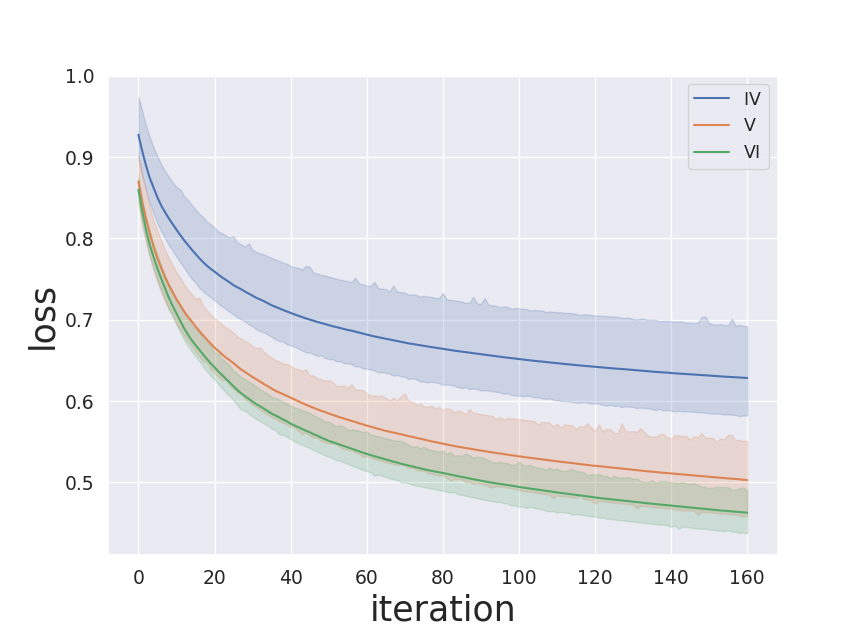}}
     \hfill
    \subfigure[$k=3$]{\includegraphics[width=0.3\textwidth]{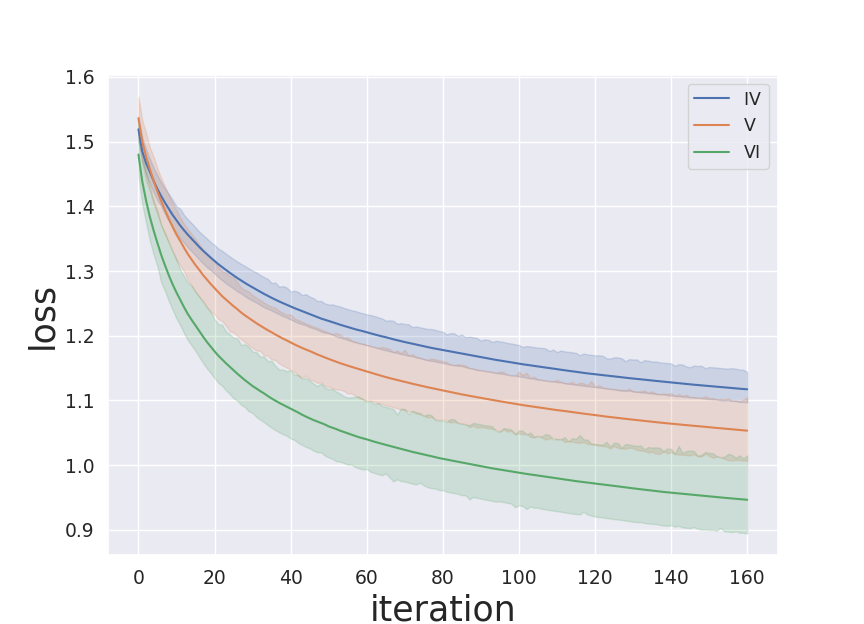}}
     \hfill
     \subfigure[$k=5$]{\includegraphics[width=0.3\textwidth]{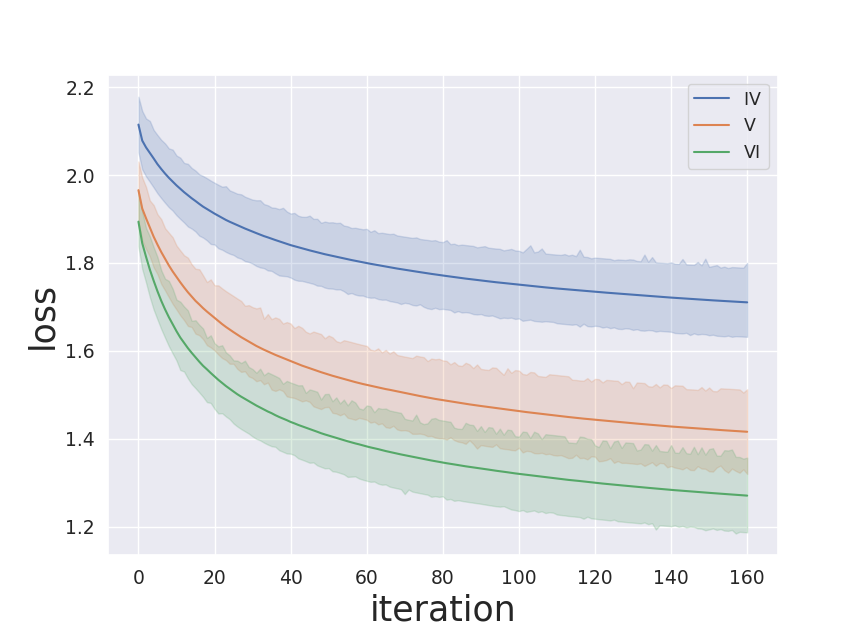}}
		\end{minipage}
\begin{minipage}{\textwidth} 
    \subfigure[$k=1$]{\includegraphics[width=0.3\textwidth]{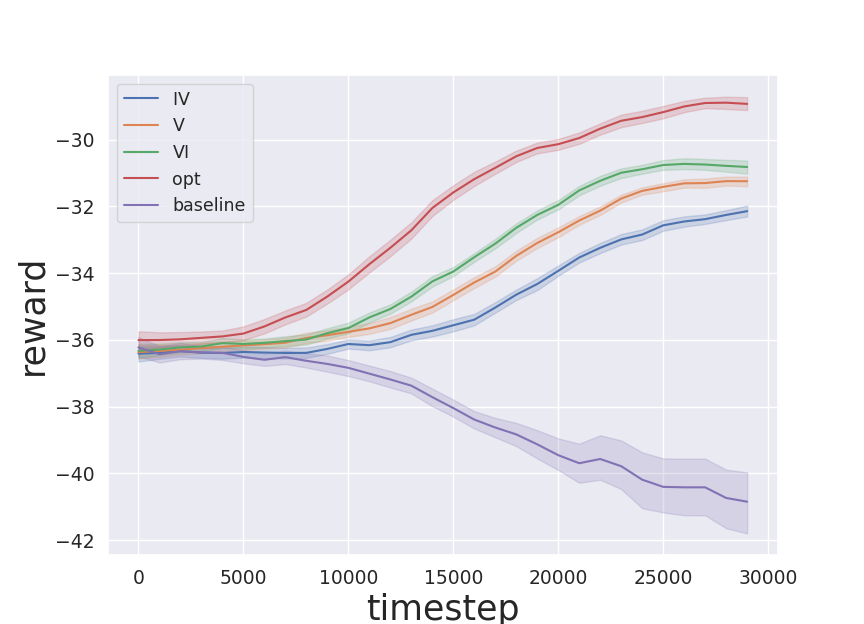}}
     \hfill
    \subfigure[$k=3$]{\includegraphics[width=0.3\textwidth]{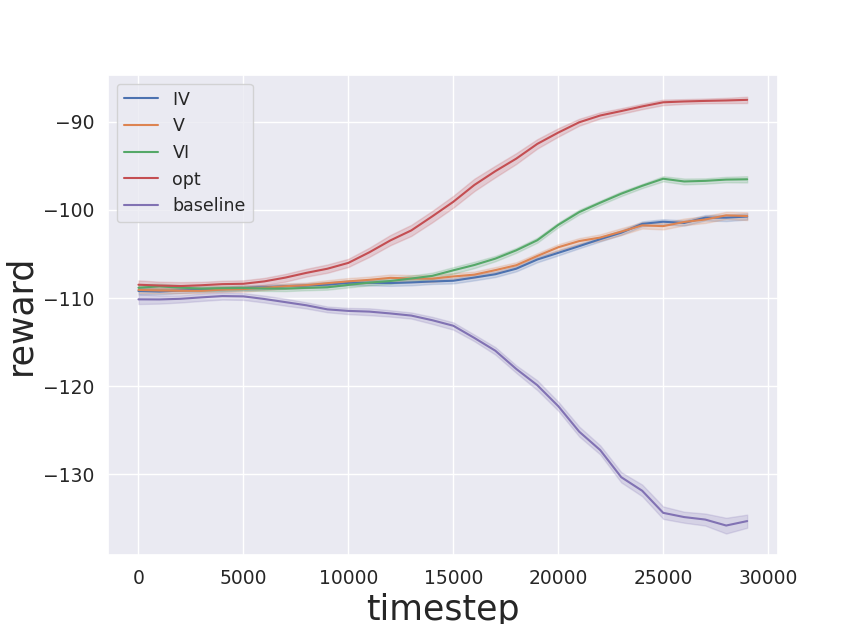}}
     \hfill
     \subfigure[$k=5$]{\includegraphics[width=0.3\textwidth]{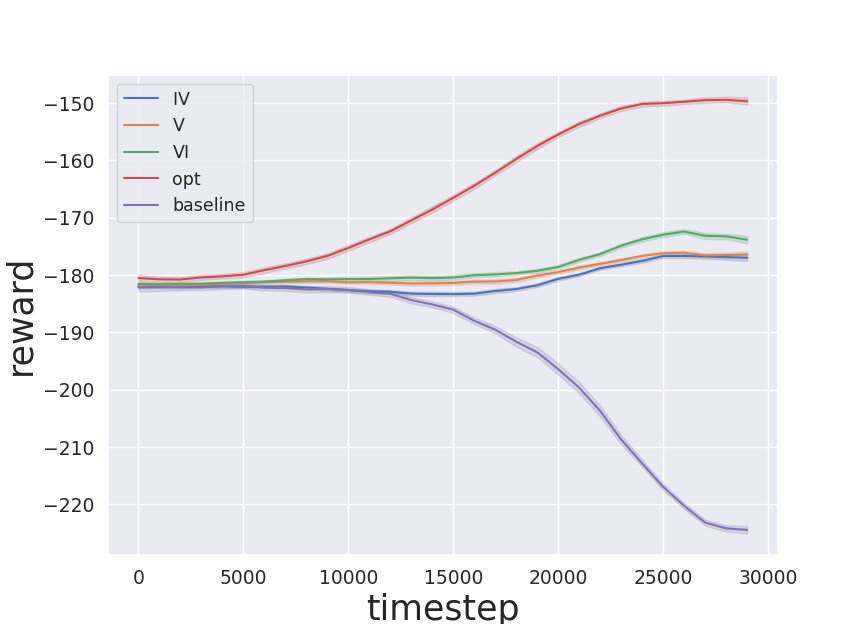}}
		\end{minipage}

    \caption{Robustness check experiment results for the nonparametric setting. Top panel: The gradient descent ascent loss $\|W^{t}-W^{sad}\|_{F}$ for different settings of instrument strength under different dimensions. Bottom panel: The performance curves (with 95\% confidence interval) of reward versus the time steps for different transition functions. The time step is the episode for SPEDE.}
    \label{fig:nonparametric-non-diagonal-result}
        
\end{figure}

\subsection{Assessment with MovieLens Dataset} \label{sec:movielense}

We construct a semi-synthetic data based on MovieLens 1M dataset \citep{harper2015movielens},
MovieLens is a dataset that records people's ratings for different movies and it contains approximately 1 million ratings (on the scale 0 -- 5) of 3952 movies created by 6040 individuals. 
The rating matrix is a highly sparse matrix, containing very few user/movie rating pairs.

We now describe our semi-synthetic setup based on the user/movie rating pairs.
Let $R$ denote the rating matrix and
let $R=\hat{U}\Sigma \hat{V}^{\top}$ the SVD of the rating matrix $R$.
The rows of the matrix $\hat U$ represent the preference of each user for different movie categories, 
and the rows of the matrix $\hat V$ represent the membership of a movie in these categories. 
We focus on the top 10 categories and 10 movies selected by singular value decomposition (SVD). 
We keep the top 10 singular values in matrix $\Sigma$ and the corresponding leading singular vectors in $\hat{U}$ and $\hat V$.
We denote the resulting matrices by $U$ and $V$, respectively. 
Finally, we use $\tilde{V}$ to denote the first 10 rows of $V$. The matrix $\tilde{V}$ is a 10 by 10 matrix of movies by categories.


Figure \ref{fig:recommender} shows the causal diagram that represents data generation in our semi-synthetic application. 
To construct a Markov decision process,  
we let the state be the user's preference 
(captured by ratings for movie categories), 
and the action be whether the user watches these movies. 
In this application, the IV is the recommendation from the recommender system. 
In recommender systems, the recommendation is sufficiently randomized 
conditionally on the user's characteristics and only affects preference
by encouraging the user to watch the recommended movie.
The UC in this setting are the unobserved factors that affect both whether the user watches the recommended movie (the action) and his preference (the state variable) simultaneously.
For example, the director of a movie could affect the action 
(how likely the user actually watches the movie) and 
the state variable (the updated preference of movie categories after watching) 
at the same time. Our goal is to identify a sequence of movie recommendations that, when actually followed by users, will result in high ratings for movies.

\begin{figure}[t]
	\centering
	\hspace{-.0cm}
	\fbox{\includegraphics[scale=.6]{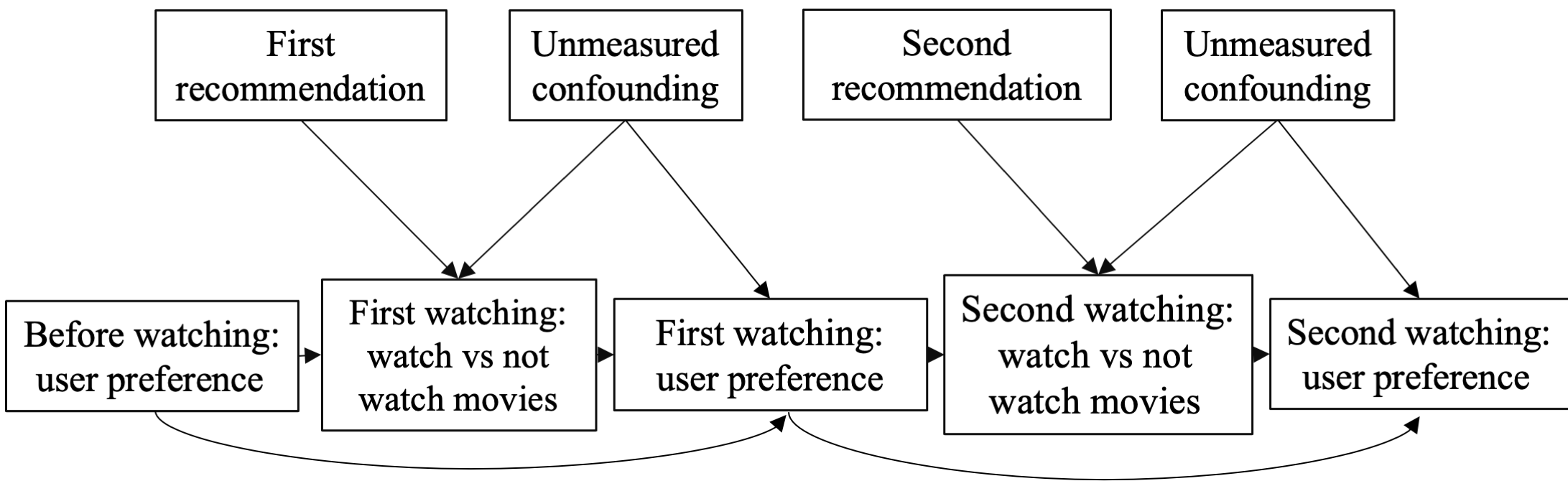}}
	\\
	\hspace{0cm} \fbox{\includegraphics[scale=.6]{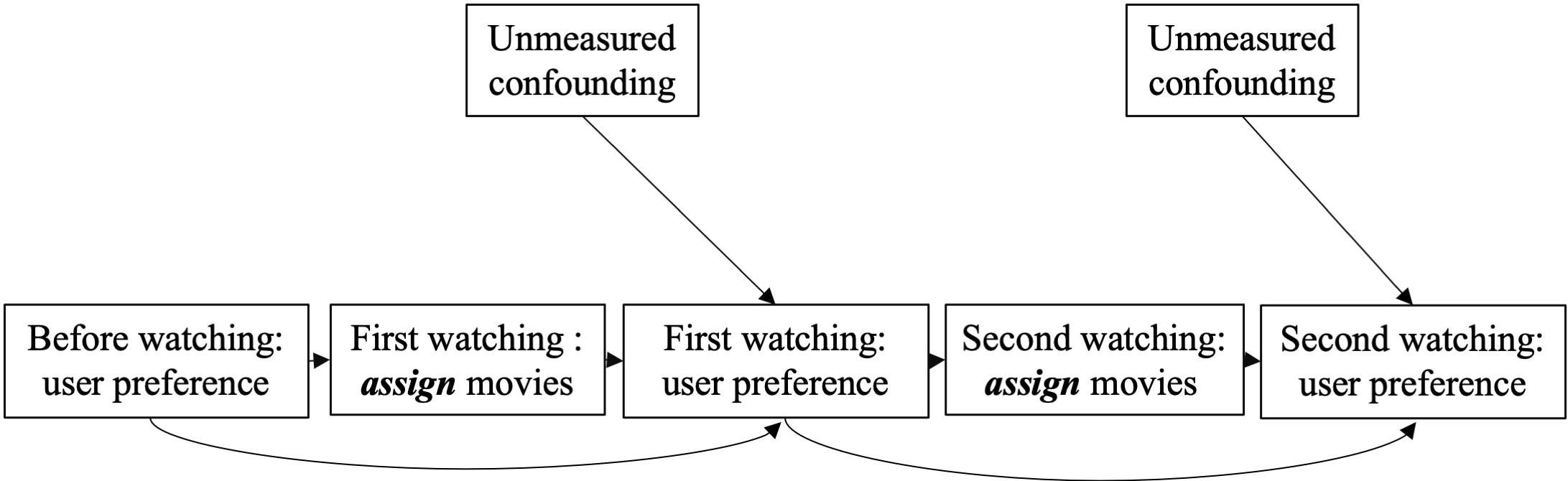}}
	\caption{The recommender application. \ifjmlr Top \else Left \fi panel: DAG representing data generation process in a recommender system where UCs are present. \ifjmlr Bottom \else Right \fi panel: DAG representing a recommender system in action.}
\label{fig:recommender}
\end{figure}

Formally, at the beginning of an episode, 
we randomly choose one user's initial ratings for the 10 movie categories from the rows of matrix $U$ (with Gaussian noise) as the initial state variable $x_{0} \in \R^n$.
We let a sequence of UCs $\{e_{h}\}_{h}$ follow a Gaussian distribution, that is,
$\{e_{h}\}_{h}  \stackrel{ \iid} {\sim} \mathcal{N}(0, I_{10})$, a sequence of i.i.d.~IVs, $z_{h} \in \R^{10}$, $h = 1,\dots, H$ follows a multinomial distribution with $n$ total recommendations, and probability of recommending the $i$-th movie $p_{i}$. 
Note that the IV is independent of the current state and UC.
At the $h$-th step, the behavior policy samples the action $a_{h}$, 
and then the next state $x_{h+1}$ is generated as 
\[
a_{h,i}  \sim \textrm{Bernoulli}\bigg(\frac{1}{1+e^{-2z_{h,i}-0.1x_{h,i}+0.8e_{h,i}}}\bigg) \;,\quad
x_{h+1}=F^*(x_h, a_h) + e_h = x_{h} + \tilde{V}^{\top}a_{h}  + e_h\;.
\]
The action $a_{h,i}$, the $i$-th entry of $a_h$, is generated by sampling from a Bernoulli distribution, where the parameter is a logistic transform of a linear combination of instruments, states and confounders. 
The transition function mimics how a user would update her preference after watching a movie. 
If a user watches one particular movie, 
we add movie's category to the current user's preference, and
if the user does not watch any movies, we do not update the preference.
The reward function is $r_{h}(x_{h},a_{h})=\lVert \operatorname{Proj}_{[0,5]} \big(\Sigma\tilde{V}^{\top}x_{h})\rVert_1$, 
representing the sum of ratings. \change{We generate 200 episodes,
each with horizon $H=100$.} We use different multinomial distributions to control 
the strength of the instrument and study its influences on our algorithm. 
Under the above setup, the sequences $\{z_{h}\}_{h} $ and $\{e_{h}\}_{h}$ are both i.i.d.~and independent of the current state variable. Therefore, $z_{h}$ is a valid IV and $e_{h}$ is a valid UC in the proposed dynamics. We use three different multinomial distributions:
\change{
\begin{enumerate}[label=\Roman*.]
    \item $n=10$, $p_{1}=p_{2}=\cdots=p_{4}=\frac{1}{5}$, $p_{5}=p_{6}=\cdots=p_{10}=\frac{1}{30}$; 
    \item $n= 10$, $p_{1}=p_{2}=\frac{1}{5}$, $ p_{3}=p_{4}=\cdots=p_{6}=\frac{1}{10}$, $ p_{7}=p_{8}=\cdots=p_{10}=\frac{1}{20}$; 
    \item $n=10$, $p_{1}=p_{2}=\cdots=p_{10}=\frac{1}{10}$. 
\end{enumerate}
}

\change{We also consider a revised data generation process based on the fact that recommendation is usually made based on user preference. The difference from the  model above is the extra arrow from user preference to recommendation. In the recommender system context, this extra arrow represents that the recommendation is made based on user preference. The sequence of IVs, $z_{h} \in \R^{10}$, $h = 1,\dots, H$ still follows a multinomial distribution with $n$ total recommendations, but the probability of recommending the $i$-th movie $p_{i}$ is determined by the estimated ranking of the $i$-th movie based on the current user preference. To be specific, we have a multinomial distribution with probability $q_{1},\cdots,q_{10}$, where $q_{1}\geq q_{2}\geq \cdots \geq q_{10}$. At the $h$-th step, the ranking of the $i$-th movie can be estimated as the $i$-th entry of $\operatorname{Proj}_{[0,5]} \big(\Sigma\tilde{V}^{\top}x_{h})$. We sort the estimated ranking in decreasing order and suppose the ranking of the $i$-th movie is the $j$-th highest ranking. Then, we let $p_{i}=q_{j}$. We use three different multinomial distributions:
\begin{enumerate}[label=\Roman*.]
    \item $n=10$, $q_{1}=q_{2}=q_{3}=\frac{1}{3}$, $q_{4}=q_{5}=\cdots=q_{10}=0$; 
    \item $n=10$, $q_{1}=q_{2}=\frac{1}{5}$, $ q_{3}=q_{4}=\cdots=q_{6}=\frac{1}{10}$, $ q_{7}=q_{8}=\cdots=q_{10}=\frac{1}{20}$; 
    \item $n=10$, $q_{1}=q_{2}=\cdots=q_{10}=\frac{1}{10}$. 
\end{enumerate}
}

In the three offline dynamics only the distribution of IV is different; 
their corresponding evaluation dynamics are the identical. 
This enables us to evaluate how the strength of instrument affects estimation of optimal policy in the evaluation dynamics.


We instantiate Algorithm~\ref{algo} with feature maps $\phi(x,a) = [1,x,a]^{\top}$ and $\psi(x,z)=[1,x,z]^{\top}$. Note that the current state $x$ is also a variable in the feature map $\psi$. With this feature map, the true transition function can be written as $W^{*}\phi(x,a)$, where $W^{*}=[I,\tilde{V}^{\top}]$. We use the minibatch stochastic gradient descent to compute $\Wsad$. We set initial estimates $K_0$ as a zero matrix and $W_0$ as a matrix where every entry is one fifth. \change{At $t$-th iteration, stepsizes we use are $\eta_t^\theta=\frac{1}{550+t}$ and $\eta_t^\omega=\frac{1}{1800+t}$. For the revised recommender application, stepsizes are $\eta_t^\theta=\frac{1}{600+t}$ and $\eta_t^\omega=\frac{1}{1800+t}$.} The estimated transition function can be expressed as $W^T \phi(x,a)$. In Phase 2, we use the SPEDE planning algorithm.


We compare our method with ordinary regression where we use the feature map $\phi(x,a)$ for a fair comparison. 
We use the same baseline estimator as specified in the parametric setting. 
Recall that ordinary regression does not take IV $z_{h}$ and UC $e_{h}$ into consideration.
We use the SPEDE planning algorithm to recover the optimal policy.

\begin{figure}[t]
    \centering
\begin{minipage}{0.8\textwidth} 
        \subfigure[Original Model]{\includegraphics[width=0.5\textwidth]{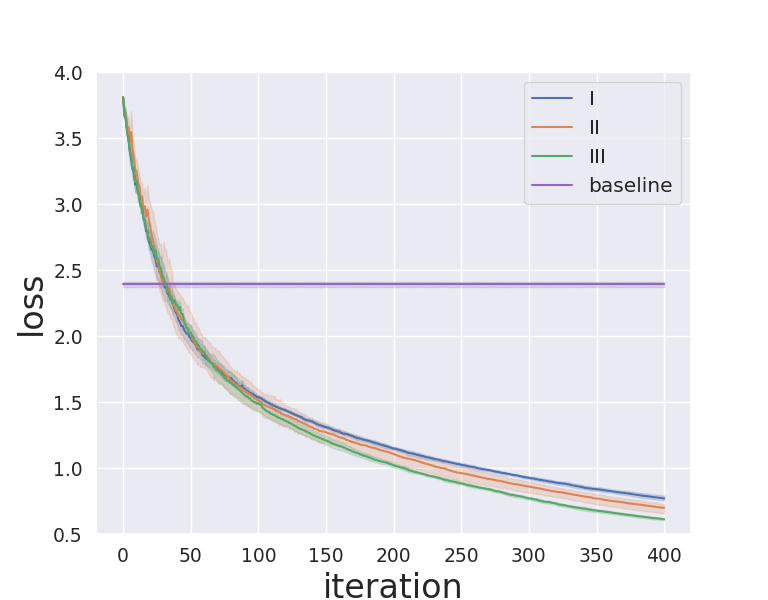}}
     \hfill
    \subfigure[Revised Model]{\includegraphics[width=0.5\textwidth]{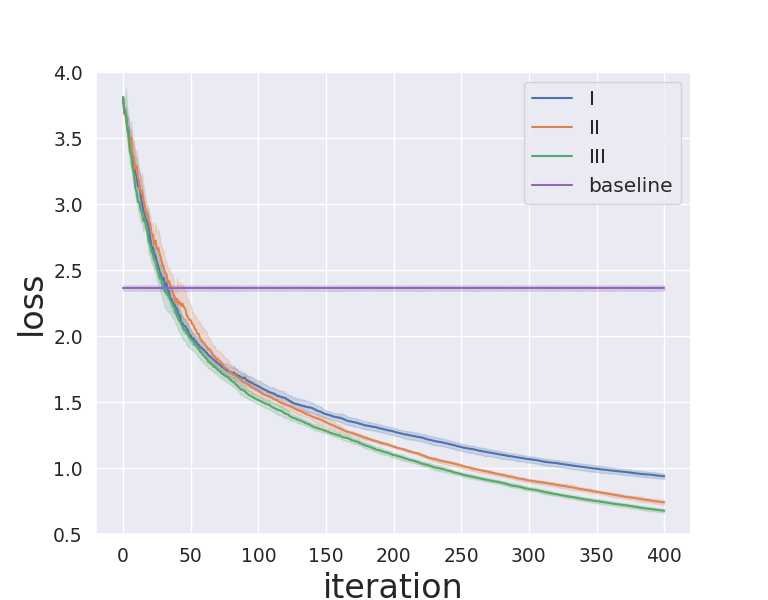}}
		\end{minipage}
\begin{minipage}{0.8\textwidth} 
        \subfigure[Original Model]{\includegraphics[width=0.5\textwidth]{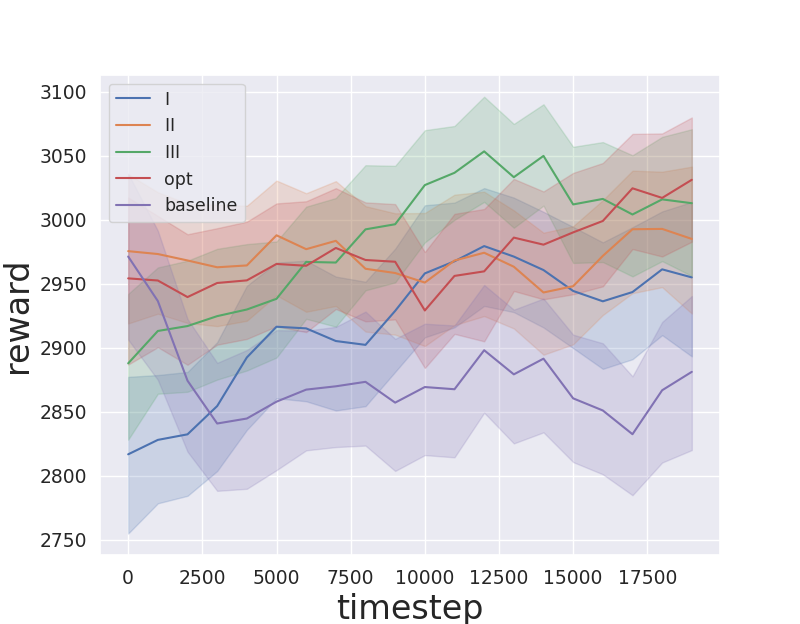}}
     \hfill
    \subfigure[Revised Model]{\includegraphics[width=0.5\textwidth]{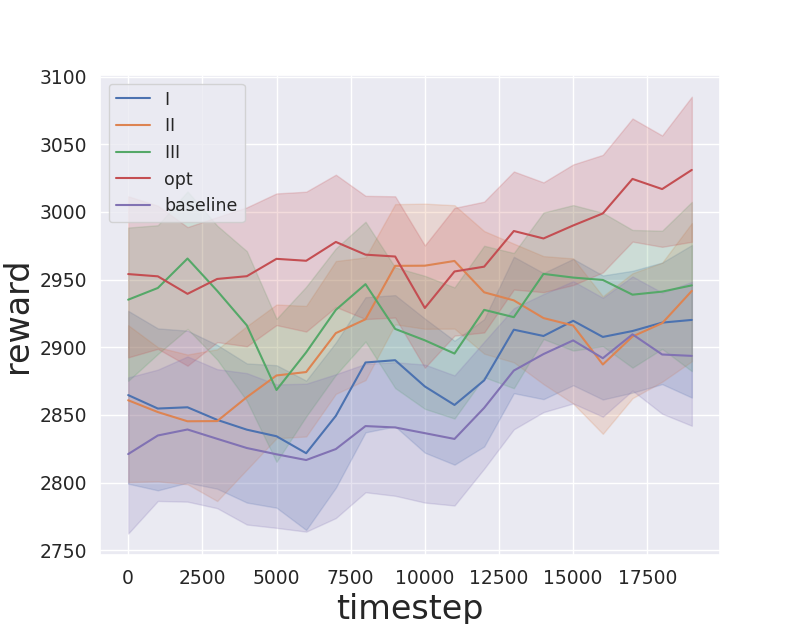}}
		\end{minipage}

    \caption{Experiment results on the MovieLens dataset. Top panel: The gradient descent ascent loss $\|W^{t}-W^{sad}\|_{F}$ for different settings of instrument strength under different data generation process. Bottom panel: The performance curves (with 95\% confidence interval) of reward versus the time steps for different transition functions. The time step is the episode for SPEDE.}
    \label{fig:recommender-result}
\end{figure}

\begin{table}[p]
\centering
\begin{tabular}{c *{3}{ccc}}
\toprule
\diagbox{index}{model} & Original Model & Revised Model \\ \hline
\midrule
I      & 0.012  &  0.002  \\ 
II     & 0.018  & 0.019   \\ 
III    & 0.020  & 0.021   \\ 
\bottomrule
\end{tabular}
\caption{Instrument Strength of different dynamics (MovieLens dataset)}
\label{tbl:recommender-instrument strength}
\end{table}

\begin{table}[p]
\centering
\begin{tabular}{c *{3}{ccc}}
\toprule
\diagbox{index}{model} & Original Model & Revised Model \\ \hline
\midrule
I     & 4.92   & 3.18    \\ 
II    & 4.21   & 3.33   \\ 
III   & 4.25   & 3.13   \\ 
\bottomrule
\end{tabular}
\caption{Average computation time (in seconds) for estimation phase (MovieLens dataset)}
\label{tab:computational time metrics under recommender setting}
\end{table}

\begin{table}[p]
\centering
\begin{tabular}{c *{3}{ccc}}
\toprule
\diagbox{index}{model} & Original Model & Revised Model \\ \hline
\midrule
Baseline & $149.98\pm 79.84$  & $137.47\pm 76.56$  \\ 
I        & $76.24\pm 81.27$  & $110.88\pm 73.90$   \\ 
II       & $46.20\pm 75.71$  & $89.43\pm 74.68$   \\ 
III      & $18.30\pm 75.32$  & $85.34\pm 81.11$   \\ 
\bottomrule
\end{tabular}
\caption{Regret of different dynamics (MovieLens dataset)}
\label{tbl:recommender-regret}
\end{table}

Table~\ref{tbl:recommender-instrument strength} shows the instrument strength of different dynamics. Table~\ref{tab:computational time metrics under recommender setting} shows the computational time for the estimation phase. Figure \ref{fig:recommender-result} shows the results on the MovieLens 1M dataset. The top panel shows how the strength of the instrument affects the convergence of GDA in Phase 1. The estimation loss, measured by $\| W^t - W^*\|_F$, decreases faster as the strength of the instrument increases. We also plot the loss for the baseline in the top panel. We note that our method achieves a smaller loss when estimating the transition function compared to the baseline. The bottom panel plots the reward and its 95\% confidence interval obtained by the SPEDE planning algorithm. We observe that the baseline reward is lower than the optimal reward. This shows that in the presence of UC, a poor policy is produced due to the bias in the ordinary regression estimates, as well as the increased estimation error. We summarize the results in Table \ref{tbl:recommender-regret}. We compute instrument strength by averaging the five smallest nonzero eigenvalues from uncentered covariance matrices of the feature maps $\phi$ and $\psi$. The regret is calculated in the same way as in the parametric setting. Compared to the baseline, the policies obtained by the IV-based estimation all achieve higher rewards than the one obtained by OLS. Moreover, these policies have a similar reward as the policy obtained by planning with the true underlying model (the curve labeled `opt'). We also observe that there is a very minor numerical difference between the original model and the revised model.

\section{Conclusion and Discussions}

Our model is motivated by real-world applications of RL in recommender systems and healthcare,
where UCs are present. 
We show that, for additive nonlinear transition dynamics,
a valid IV can help identify the confounded transition function.
The proposed IVVI algorithm is based on a primal-dual formulation of
the conditional moment restriction implied by the IV. Moreover,
our stochastic approximation approach to the nonparametric IV problem is
of independent interest. We derive the convergence rate of IVVI. 
Furthermore, we derive the sample complexity of offline RL with 
IVs in the presence of unmeasured confounders.

\subsection*{An Updated Survey}
After this paper has been posted publicly, there is an emerging literature on the use of causal inference in RL.
Thanks to the suggestion by an anonymous reviewer, we present an updated survey for these works, even though some of these works come after this paper.

The work by \citet{wang2018bounded} considers average treatment estimation under unobserved confounders. They study binary IVs and binary treatment, derive conditions for ATE identification and propose a multiple robust estimator based on existing regression based, inverse probability weighting based, and G-formula based estimators.
Building upon the paper by \citet{wang2018bounded}, \citet{cui2021semiparametric} consider estimation of optimal treatment in the presence of unmeasured confounders. They look at settings with binary IVs and binary treatments.
With additional assumptions on the IVs (the no unmeasured common effect modifier or the independent compliance type conditions in their paper), the authors derive identification results for the optimal treatment and proposed multiply robust classification-based estimators. 
Different from our work, the model in the above two papers is one-stage and thus not sequential model. Our stochastic apporoixmation-based estimator is computationally attractive.
The very recent work by \citet{bilodeau2022adaptively} studies the problem of online learning an low-regret dynamic policy. They consider a bandit setting where the reward depends on the action only. Interestingly, they allow the presence of unobserved confounder that affects both the action and the reward.
Different from our paper, we consider an MDP setting where reward depends on the action and the state variable. They consider discrete action space while we work with continuous action space."

There is a line of work that explore the use of IV in RL.
\citet{chen2021estimating} use existing partial identification results for IV and study policy improvement under the binary-treatment binary-instrument setting. Different from our work, our 
paper studies the continuous-treatment continuous-instrument case, and the transition function is identified by the condition moment restriction.
\citet{fu2022offline} focus on the discrete-treatment discrete-instrument case, show that by using certain weighting schemes the values of policies are identified, and adapt pessimistic offline RL algorithm to learn the optimal policy. 
The work of \citet{xu2023instrumental} also study the discrete-treatment discrete-instrument case but, different from \citet{fu2022offline}, they derive the efficient influence function and propose a more efficient estimator for policy values. 
Different from these two works, our identification and estimation strategy is based on conditional moment restrictions. 
\citet{yu2022strategic} study a novel learning setting which they termed strategic MDP. Their model features the strategic interactions between a principal and a sequence of myopic agents with private types. They show that IV structure exists in the model. 
While we work under different learning settings, both their work and ours use conditional moment restrictions for identification and estimation.

Researchers have also explored other causal structure in RL. For example, \citet{wang2021provably} studies confounded MDP where front-door or back-door adjustments are available, \citet{bennett2021proximal} studies the case where proxy variables are present. Even in the IV case, our work has inspired several related work to explore the use of IV for other tasks in RL, such as offline policy evaluation \citet{xu2023instrumental}, offline policy learning in strategic MDP \citet{yu2022strategic}, policy improvement \citet{chen2021estimating} and offline RL with discrete instruments \citet{fu2022offline}. What we aim to show is the CMDP-IV model we propose is a reasonable model for quite a few datasets, and the estimation method is backed by both theoretical proofs and synthetic and semi-synthetic experiments.

\subsection*{Relaxation of \cref{as:dynamic}}

There are several directions to proceed from \cref{as:dynamic}.

For example, we allow instruments to depend on previous history.
Concretely, we allow $z_h$ to depend on $\{x_h, \dots, x_1, z_{h-1},\dots, z_1\}$.
	Consider 
	\begin{align*}
		& x_{h+1} = F(x_h, a_h)+ e_h, \, a_h \sim \pi_h(\cdot | x_h, z_h, e_h), 
		\\
		& \{e_1,\dots, e_H, x_1\} \text{ are independent} , 
		 x_1 \sim \xi, \, e_h \sim N(0, 1) \text{ for all $h$}
		\\
		& z_1 \sim \cP_{z,1}(\cdot|x_1),\, z_h\sim \cP_{z,h}(\cdot | x_h, \dots, x_1, z_{h-1},\dots, z_1) \text{ for all } h.
		\end{align*}
	By the same reasoning as above, a conditional moment restriction will be implied and estimation can be done. 
	Let $L_h$ be the law of $(x_h, a_h, z_h, x_{h+1})$, and $L(x,a,z,x')$ be the average mixture of $\{L_1,\dots, L_H\}$.
	Let $p_{x,z,h} (x, z)$ be the marginal of $(x_h, z_h)$. Then the density of $L_h$ is $p_{x,z,h}(x,z) \cP_e(e) \pi_h(a|x,z,e) 1(z' = F^*(x,a) + e)$, and the density of $L$ is 
	\begin{align*}
		& \frac1H \sum_{h=1}^H p_{x,z,h}(x,z) \cP_e(e) \pi_h(a|x,z,e) 1(z' = F^*(x,a) + e)
		\\
		& = \cP_e(e) \underbrace{\bigg(\frac 1H \sum_{h=1}^H p_{x,z,h}(x,z)\bigg)}_{=:p_{x,z}(x,z)} 
		\underbrace{ \bigg(\sum_{h=1}^H \frac{p_{x,z,h}(x,z)}{\sum_{h'=1}^H p_{x,z,h'}(x,z)} \pi_h(a|x,z,e)\bigg) }_{=:\bar \pi (a|x,z,e)}
		\\
		& \quad \cdot 1(z' = F^*(x,a) + e)	
	\end{align*}
	Clearly for $(x,a,z,e,x') \sim L$ we have $\E[F^*(x,a) - x' | z] = 0$.

	We do not aim to exhaust all possibilities in the paper because we mainly aim to develop the idea that in the confounded MDP setting, one could identify and estimate transition dynamics and thus the optimal policy through conditional moments implied by instrumental variables.

	Episode-wise dependence is also possible. Let $\cF_t =\sigma \{ x_{h,\tau}, z_{h,\tau}, a_{h,\tau}, x'_{h,\tau} \}_{h\in[H],\tau=1,\dots, t}$ be the data of the first $t$ episodes.
	Then we could let $x_{1,t+1} \sim \xi(\cdot | \cF_t)$, i.e., the first state at the $(t+1)$-th episode can be chosen depending on previous $t$ episodes. This is particularly relevant when the behavior policy is updated in the observation process.

The additive structure seems hard to relax in our opinion.
A possible extension is as follows. For ease of notation let the state space $\mathcal{X}$ be $\R$.
Then for a transition function $F: \R \times \cA \times \R \to \R$, the observation dynamics writes 
\begin{align*}
& x_{h+1} = F(x_h, a_h, e_h), \, a_h \sim \pi_h(\cdot | x_h, z_h, e_h), 
\\
& \{e_1,\dots, e_H, x_1, z_1, \dots, z_H\} \text{ are independent} , 
\\
& x_1 \sim \xi, \, e_h \sim N(0, 1), z_h\sim \cP_z \text{ for all } h.
\end{align*}
Following \citet{chernozhukov2007instrumental}, assume $F(x,a,\cdot)$ is strictly increasing on $\R$ for all $(x,a)$. Let $L_h$ be the law of $(x_h, a_h, z_h, x_{h+1})$, and $L(x,a,z,x')$ be the average mixture of $\{L_1,\dots, L_H\}$.
Then by the same argument as that of \citet{chernozhukov2007instrumental}, for any $\tau \in (0,1)$, it holds 
\begin{align}
    \E\big[1\big( x' < F(x, a, \tau)\big) - \tau  \big| z\big] = 0\,
\end{align} 
where the expectation is taken w.r.t\ $(x,a,z, x')\sim L$.
However, having arrived at a conditional moment restriction, we notice several difficulties to proceed. 
First, monotonicity is imposed on the transition function $F(x,a,\cdot)$, and it is unclear how to test this assumption on real data or exploit it for estimation.
Second, the indicator function present in the conditional moment restriction is nonsmooth, bringing challenges to theoretical analysis. 
Traditional estimation methods in existing literature are not suitable for RL applications.
Consider the special case of single decision stage ($H=1$), which is essentially the nonparametric quantile IV (NPQIV) problem in the econometric literature. There, most papers use traditional computation-heavy nonparametric estimators such as sieve or kernel estimators
\citep{chernozhukov2007instrumental,chernozhukov2005iv,horowitz2007nonparametric}, which are not suitable for RL applications where online procedures are definitely preferred.

\subsection*{Comment on \cref{as:iiddata}}

The i.i.d.\ data assumption is purely for simplifying notations. We just replace anywhere in Algorithm 1 where we evaluate a function $f(x,z,a,x')$ at a point $(x, z, a, x')$ drawn from the average visitation distribution $\bar{d}_{\pi_b}$ with $\frac{1}{H} \sum_{h=1}^{H} f(x_h, z_h, a_h, x_{h+1}) $, where $\{ x_h, z_h, a_h, x_{h+1} \}_h$ follows the dynamics describe by  $x_{h+1} = F^*(x_h, a_h) + e_h$ and $a_h \sim \pi_h(\cdot | x_h, z_h, e_h)$.

	Concretely, Lines 3-4 in Algorithm 1
	\begin{align*}
	& \phi_t \leftarrow \phi\left(x_t, a_t\right), \psi_t \leftarrow \psi\left(z_t\right) \\
	& W_{t+1} \leftarrow W_t-\eta_t^\theta \cdot\left(K_t \psi_t \phi_t^{\top}\right), \quad K_{t+1} \leftarrow K_t+\eta_t^\omega \cdot\left(K_t \psi_t \psi_t^{\top}+x_t^{\prime} \psi_t^{\top}-W_t \phi_t \psi_t^{\top}\right) .
	\end{align*}
	should be replaced with 
	\begin{align}
	\label{eq:new_update}
		& W_{t+1} \leftarrow W_t-\eta_t^\theta \cdot
		\Bigg(K_t \bigg(\frac1H \sum_{h=1}^H \psi(z_{h,t} ) \phi( x_{h,t}, a_{h,t}) ^ \top \bigg)\Bigg)
		\\
	& K_{t+1} \leftarrow 
	K_t+\eta_t^\omega \cdot
	\Bigg(\frac1H \sum_{h=1}^H 
	\bigg(K_t\psi(z_{h,t} )^{\otimes 2}
	+  x_{h+1,t} \psi(z_{h,t} )^\top 
	- W_t  \phi( x_{h,t}, a_{h,t}) \psi(z_{h,t})^\top\bigg)\Bigg) .
	\notag
	\end{align}
	where $(x_{h,t}, z_{h,t}, a_{h,t}, x_{h+1, t})$ are the $h$-th step data in the $t$-th episode.

	The proof of convergence for the stochastic approximation procedure goes through by the following reasoning. 
	Suppose we are at the $t$-th iteration of the primal-dual algorithm, then conditional on the previous $(t-1)$ iteration (i.e., data from the previous $t-1$ episodes), the gradients in Eq \ref{eq:new_update} are unbiased estimates of the corresponding population gradient.

	Finally, we remark that in practice, such an i.i.d.\ sampling oracle from $\bar d_{\pi_b}$ can be approximated by the following sampler: draw $k$ uniformly from $[K]$ and $h$ uniformly from $[H]$, and then output $\{x_{h,t}, z_{h,t}, a_{h,t}, x'_{h,t}\}$ which is data at the $h$-th timestep in the $t$-th episode.

\acks{We sincerely thank the editor, Professor Eric Laber, and three anonymous reviewers for their constructive comments. We also thank the attendees of the TTIC Chicago Summer Workshop: New Models in Online Decision Making for Real-World Applications (2022) for many helpful discussions. YW was supported in part by the Office of Naval Research under  grant number N00014-23-1-2590, and the National Science Foundation under  Grant No. ~2231174, and No. ~2310831, No. 2428059, and a Michigan Institute for Data Science Propelling Original Data Science (PODS) grant. MK was supported in part by the National Science Foundation under Grant No.~ECCS-2216912.}

\newpage
\appendix

\section{Proof Sketch}\label{sketch}

The proof consists of two parts: the analysis of the convergence of the stochastic gradient descent-ascent (Line~\ref{algo:beginsgda}--\ref{algo:endsgda}) and the analysis of the planning phase using the estimated model (Line~\ref{algo:initvhp}--\ref{algo:max1}).

In Remark~\ref{rm:oneovertrate} we emphasized the stochastic minimax optimization problem is only strongly concave in the dual variable. This motivates us to study the recursion of the following asymmetric potential function. For some $\lambda>0$, define
\$
\Tilde{P_t} = \E \big[ \|W_t - W^* \|_F\sq\big] + \lambda \E \big[\|K_t - K^*(W_t) \|_F\sq\big]
\$
where $K^*(W) = (WA^\top - C)B\inv$ with $A$, $B$ and $C$ defined in \cref{eq:defAB}. The matrix $K^*(W)$ is the optimal dual variable in the saddle-point problem \cref{eq:estimatorw} when the primal variable is fixed at $W$. In order to get around the assumption of bounded variance of stochastic gradients, which is common in the optimization literature \citet{nemirovski2009robust}, we follow the idea in the work of \cite{sgdandhogwild2018nguyen} where we upper bounds the variance of stochastic gradients by the suboptimality of the current iterate; see Lemma~\ref{lm:obs:varboound}. Thus our algorithm does not require projection in each iteration. A careful analysis of the recursion for the sequence $\{\Tilde{P_t}\}$ shows the error in squared Frobenius norm converges at the rate $O(1/t)$.

The second element in our analysis is the decomposition of difference of value functions, which is adapted from Lemma~4.2 of \cite{provably2020cai}.

\begin{lemma}[Suboptimality Decomposition]\label{lm:errordecomposition}
It holds that for all states $x\in \cS$, 
\#
 V_1^*(x) - V_1^{\hat \pi}(x) =  & \sumh \E_{\pi^*} [\iota _h (x_h,a_h) \cond x_1=x] 
 \label{eq:errordecomp}
\\
&  + 
\sumh \E_{\pi^*} [ \xi_h(x_h) \cond x_1=x ]  - \sumh \E_{\hat \pi } [\iota_h(x_h,a_h) \cond x_1 = x], 
\notag
\#
where $\hat \pi$ is the output of Algorithm \ref{algo}, the expectations $\E_{\pi^*}$ and $\E_{\hat\pi}$ are taken over trajectories generated by policies $\pi^*$ and $\hat\pi$ under the true transition function $F^*$, respectively,  $\xi_h=\langle \hat Q_h, \pi^*_h - \hat \pi _h \rangle_\cA $ for all $x\in \cS$, and $\iota_h = (r_h + \P\Vhat_{h+1}) - \hat Q_h$ for all $(x,a)\in \cS \times \cA$.
\end{lemma}
\begin{proof}
See Appendix \ref{pf:lm:errordecomposition} for a detailed proof.
\end{proof}

\section{Structural Causal Model and Intervention} \label{sec:defofscm}

Structural Causal Models (SCMs) provide a formalism to discuss the concept of causal effects and intervention. We briefly review its definition in this section and refer readers to \citet[Ch.~7]{pearl2009causality} for a detailed survey of SCMs.

 A structural causal model is a tuple $(A,B,F,P)$, where $A$ is the set of exogenous (unobserved) variables, $B$ is the set of endogenous (observed) variables, $F$ is the set of structural functions capturing the causal relations, and $P$ is the joint distribution of exogenous variables. 
 An SCM is associated with a causal directed acyclic graph, where the nodes represent the endogenous variables and the edges represent the functional relationships. In particular, each exogenous variable $X_j\in B$ is generated through $X_{j}=f_{j}(X_{\mathrm{pa}_{D}(j)}, U_{j})$ for some $f_j\in F$, $U_j\in B$, where $\mathrm{pa}_{D}(j)$ denotes the set of parents of $X_j$ in $D$. A distribution over the endogenous variables is thus entailed. 

An intervention on a set of endogenous variables $X\subseteq B$ assigns a value $x$ to $X$ while keeping untouched other exogenous and endogenous variables and the structural functions, thus generating a new distribution over the endogenous variables. We denote by $\doo(X = x)$ the intervention on $X$ and write $\doo(x)$ if it is clear from the context. A stochastic intervention on a set of endogenous variables $X \subseteq B$ assigns a distribution $p$ to $X$ regardless of the other exogenous and endogenous variables as well as the structural functions. We denote by $\doo(X \sim p)$ the stochastic intervention on $X$. An intervention induces a new distribution over the endogenous variables.

For any two variables $X,Y\in B$ with a directed path from $X$ to $Y$ in $D$, we say the causal effect from $X$ to $Y$ is \textit{confounded} if $p(y | \doo(X=x)) \neq p(y| X=x)$ \citep[Def.~6.39]{peters2017elements}.
\section{Proofs} \label{app_proofs_of_theorems}

\subsection{Proof of Proposition~\ref{lm:mixture}} \label{pf:lm:mixture}
\begin{proof}[Proof of Proposition~\ref{lm:mixture}]

We recall the trajectories of a behavior policy is generated through \cref{eq:modelbegin} with $\{ e_h\}_h \indep \{z_h \}_h$. Let $p_{x,h}$ be the marginal distribution of $x_h$. Also define the probability density function and probability mass function
\$
& p_{a,h}(a\cond x,z,e) \defeq \pi_{b,h}(a\cond x,z,e) ,
\\
& p_{x'}(x'\cond x,a,e) \defeq \indi\{ x' = F^*(x,a)+e \}.
\$
Then the marginal distribution of $(x_h,a_h,z_h,e_h,x'_h)$, denoted $d_{h,\pi_b, *}$ (we use $*$ to emphasize the presence of unobserved confounder $e_h$), admits the factorization 
\$
d_{h,\pi_b,*}(x,a,z,e,x') = \cP_z(z)\cP_e(e) p_{x,h}(x) \cdot p_{a,h}(a\cond x,z,e) \cdot p_{x'}(x'\cond x,a,e).
\$
And the average visitation distribution of all random variables $\{x_h,a_h,z_h,e_h,x'_h\}_h$ is
\$
& \bar d_{\pi_b,*}(x,a,z,e,x') 
\\
&\defeq \frac1H \sum_{h=1}^H d_{h,\pi_b,*}(x,a,z,e,x')
\\
&=  \cP_z(z)\cP_e(e)  \cdot \Bigg(\sum_{h=1}^H {p_{x,h}(x)} p_{a,h}(a\cond x,z,e)\Bigg) \cdot p_{x'}(x'\cond x,a,e)
\\
&= \cP_z(z)\cP_e(e)  \cdot\Bigg(\frac{1}{H}\sum_{h=1}^H p_{x,h}(x)\Bigg) \cdot \Bigg(\sum_{h=1}^H \frac{p_{x,h}(x)}{\sum_{k=1}^H p_{x,k}(x)} p_{a,h}(a\cond x,z,e)\Bigg) \cdot p_{x'}(x'\cond x,a,e).
\$
Define the weighted policy $\bar \pi(a\cond x,z,e)=\big(\sum_{h=1}^H {p_{x,h}(x)}{} p_{a,h}(a\cond x,z,e)\big) / \sum_{h=1}^H p_{x,h}(x)$ and the average state visitation distribution $p_x=\frac{1}{H}\sum_{h=1}^H p_{x,h}(x)$. Then $(x,a,z,e,x')\sim \bar d_{\pi_b,*}$ can be equivalently written as
\$
z\sim \cP_z,\,e\sim \cP_e,\,x\sim p_x, \, a\sim \bar \pi(\cdot \cond x,z,e),\, x'=F(x,a) + e.
\$
We conclude if $(x,a,z,e,x')\sim \bar d_{\pi_b,*}$ then $x'=F^*(x,a)+e$ with $\E[e\cond z]=0$.
\begin{remark}\label{rm:extend} 
We also have $\E[e\cond x,z]=0$ so we could extend the instrument $z$ to $ z \leftarrow [x,z]$, and the algorithm and the theory in this paper remain the same.
\end{remark}

\end{proof}

\subsection{Proof of Proposition~\ref{lm:mupisillposenss}}
\label{pf:lm:mupisillposenss}
\begin{proof}[Proof of Proposition~\ref{lm:mupisillposenss}]
First note for $f =\phi \cdot \theta \in \Hphi$, the operator $\Pi_\psi \cT f$ admits the form
\$
\Pi_\psi \cT f = \psi^\top \E[\psi(z)\psi(z)^\top] \inv \E[\psi(z)  (\theta \cdot \phi)(x,a)] = \psi^\top B\inv A\theta.
\$
Recall $\|f \|_\phi = \| \theta \|$. Then the feature map ill-poseness can be written as
\$
{\muIV}\defeq \min_{f\in \Hphi} \frac{\|\Pi_\psi \cT f \|_\ltz \sq }{\| f\|_\phi\sq} = \min_{\theta \neq 0} {\frac{{\theta^\top (A^\top B\inv A)\theta}}{\theta^\top \theta}},
\$
which is the minimum eigenvalue of the matrix $A^T B\inv A$. This completes the proof of Proposition~\ref{lm:mupisillposenss}.
\end{proof}

\subsection{Proof of Lemma \ref{lm:errordecomposition}}  \label{pf:lm:errordecomposition}

To facilitate the discussion, we recall the definitions of relevant quantities and define some auxiliary operators. We define the operators $\J_h$ and $\hat \J_h$
\# \label{eq:defjh}
\left(\mathbb{J}_{h} f\right)(x)=\left\langle f(x, \cdot), \pi_{h}^{*}(\cdot \mid x)\right\rangle, 
\quad
( \hat \J_{h} f)(x)=\left\langle f(x, \cdot), \hat \pi_{h}(\cdot \mid x)\right\rangle
\#
for any $h\in [H]$ and function $f: \mathcal{S} \times \mathcal{A} \rightarrow \mathbb{R}$. For any function $g:\cS \rightarrow \R$, given the model parameter $\hat W$, define the operator
\$
(\hat \P g) (x,a) = \int  g (x')\cP_{\hat W}(x'\cond x,a)\diff x' ,
\$
where $\cP_{\hat W}(x'\cond x,a)$ is the probability density of $d_x$-dimensional Gaussian distribution with mean $\hat W \phi(x,a)$ and variance $\sigma^2 I_{d_x}$ (we overload notations and let $\cP$ denotes both the distribution and the density of a Gaussian). For the true underlying transition dynamics with model parameter $W^*$, we define the operator \# \label{eq:defp}( \P g) (x,a) = \int  g (x')\cP_{ W^*}(x'\cond x,a)\diff x'.\#

We define the quantity 
\# \label{eq:defxih}
\xi_{h}(x)=(\mathbb{J}_{h} \hat Q_{h})(x)-(\hat \J_{h} \hat Q_{h})(x)=\langle \hat  Q_{h}(x, \cdot), \pi_{h}^{*}(\cdot \mid x)-\pi_{h}(\cdot \mid x)\rangle
\#
for any $h\in [H]$ and all state $x\in \cS$.

Now we clarify the relationship among $(\pi^*, \Qstar, \Vstar)$, $(\hat \pi, \Qhat, \Vhat)$ and $(\hat \pi, \Vpihat, \Qpihat)$. Recall the Bellman equation of the optimal policy $\pi^*$. For $h=1,\ldots,H$,
\#
& Q^*_h = r_h + \P( V^*_{h+1}), \quad \forall (x,a), \label{eq:bellqstar}
\\
& \Vstarh = \langle \pi^*,Q^*_{h}\rangle = \J_h Q^*_h, \quad \forall x, \label{eq:bellvstar}
\\
& V^*_{H+1} = 0
\#
and the set of Bellman optimality equations that $\pi ^*$ satisfies: $\pi^*_h(x) =\argmax_a \Qstarh(x,a)$, and  $\Vstarh=\max_a \Qstarh$. 

The update rules of $\hat \pi$ in Algorithm \ref{algo} imply the following equations relating $\hat \pi$, $\Qhat$ and $\Vhat$. For $h=1,\ldots, H$,
\#
& \Qhath = r_h + \hat \P \hat V_{h+1}, \quad \forall (x,a), \label{eq:updateqhat}
\\
& \hat \pi_h(\cdot \cond x) = \argmax_a \Qhath(x,a), \quad \forall x, \label{eq:updatepihat} 
\\
& \Vhath = \la \Qhath,\hat \pi_h \ra = \max_a \Qhath = \hat \J_h \Qhath, \quad \forall x. \label{eq:updatevhat}
\#
We recall the definition of the model prediction term
\# \label{eq:defiotah}
\iota_h = (r_h + \P \hat V_{h+1}) - \Qhath
\#
for all $(x,a)\in \cS \times \cA$.
Finally, since $\Qpihat$ and $\Vpihat$ are the $Q$ function and value function of the output policy $\hat \pi$, the Bellman equations for $\hat \pi$ holds: for $h=1,\ldots, H$
\#
& \Qpihath = r_h + \P  V^{\hat \pi}_{h+1},  \quad \forall (x,a) \label{eq:bellqhat}
\\
& \Vpihath = \la \Qpihath, \hat \pi_h\ra = \hat \J_h \Qpihath,  \quad \forall x \label{eq:bellvhat}
\\
& V^{\hat \pi }_{H+1} = 0.
\#

\begin{proof}[Proof of Lemma \ref{lm:errordecomposition}]
We first write 
\$
V^*_1 - V^{\hat \pi}_1 = (V^*_1 - \hat V_1) - (\hat V_1 - V^{\hat \pi}_1).
\$
Next we analyze the two terms separately.

\textbf{Part I: Analysis of $(V^*_1 - \hat V_1) $.} For all state $x\in \cS$, and any $h=1,\ldots, H$
\#
V^*_h - \hat V_h & = \la \pi^*_h ,\Qstarh \ra - \la \Qhath ,\hat \pi _h\ra \label{0907203a}
\\
&= \J_h \Qstarh - \hat \J_h \Qhath \label{0907203b}
\\
&= \J_h (\Qstarh - \Qhath) + (\J_h - \hat J_h) \Qhath \label{0907203c}
\\
&= \J_h (\Qstarh - \Qhath) + \xi_h \label{0907203d}
\\
&= \J_h( [r_h + \P \Vstarhp ] -[r_h + \P\Vhathp - \iota_h] ) + \xi_h \label{0907203e}
\\
&= \J_h\P (\Vstarhp - \Vhathp ) + \J_h \iota_h + \xi_h. \label{0907203f}
\#
Here 
\eqref{0907203a} follows from Bellman equations of $\Vstarh$ \cref{eq:bellvstar} and the update rule of $\Vhath$ \cref{eq:updatevhat}; 
\eqref{0907203b} follows from the definition of operators $\J_h$ and $\hat \J_h$ \cref{eq:defjh};
in \eqref{0907203c} we add and subtract $\J_h \Qhath$;
\eqref{0907203d} follows from definition of $\xi_h$ in \cref{eq:defxih};
\eqref{0907203e} follows by using the Bellman equations satisfied by $\Qstarh$ and the definition of $\iota_h$ in \cref{eq:defiotah}.

Next we apply the above recursion formula for the sequence $\{V^*_h - \hat V_h\}_{h=1}^{H}$ repeatedly and obtain 
\$
V_{1}^{*}- \hat V_{1} =
\left(\prod_{h=1}^{H} \J_{h} \mathbb{P} \right)\left(V^{*}_{H+1}- \hat V_{H+1}\right)+
\sum_{h=1}^{H}\left(\prod_{i=1}^{h-1} \J_{i} \mathbb{P}\right) \mathbb{J}_{h} \iota_{h}+
\sum_{h=1}^{H}\left(\prod_{i=1}^{h-1} \J_{i} \mathbb{P}\right) \xi_{h}.
\$
Using $V^{*}_{H+1}= \hat V_{H+1}=0$ gives
\# \label{eq:recursionsh}
V_{1}^{*}- \hat V_{1} = \sum_{h=1}^{H}\left(\prod_{i=1}^{h-1} \J_{i} \mathbb{P}\right) \mathbb{J}_{h} \iota_{h}
+\sum_{h=1}^{H}\left(\prod_{i=1}^{h-1} \J_{i} \mathbb{P}\right) \xi_{h}.\#
By definitions of $\P$ in \cref{eq:defp}, $\J_h$ in \cref{eq:defjh}, and $\xi_h$ in \cref{eq:defxih}, we can equivalently write \cref{eq:recursionsh} in the form of expectation w.r.t the optimal policy $\pi^*$. For all $x\in \cS$,
\# \label{eq:errorde1}
V_{1}^{*}(x)- \hat V_{1}(x) = 
 \sumh \E_{\pi^*} [\iota _h (x_h,a_h) \cond x_1=x] + \sumh \E_{\pi^*} [ \xi_h(x_h)] \cond x_1=x ] .
\#

\textbf{Part II: Analysis of $(\hat V_1 - V^{\hat \pi}_1) $.} Notice for any $h=1,\ldots, H$,
\#
\Vhath - \Vpihath & = \hat \J_h \Qhath - \hat \J_h \Qpihath \label{9702201a}
\\
&= \hat \J_h ([r_h + \P\Vhathp - \iota_h] - [r_h+\P \Vpihath]) \label{9702201b}
\\
&= \hat \J_h \P (\Vhathp - \Vpihathp) - \hat \J_h \iota_h. \label{9702201c}
\#
Here \eqref{9702201a} follows from the update rule of $\Vhath$ \cref{eq:updatevhat} and the Bellman equation satisfied by $\Vpihath$ in \cref{eq:bellvhat};
\eqref{9702201b} follows from the Bellman equation satisfied by $\Qpihath$ in \cref{eq:bellqhat} and the definition of the model prediction error $\iota_h$ in \cref{eq:defiotah}.

Apply the recursion repeatedly we obtain
\$
\hat V_1 - V^{\hat \pi}_1 =
\left(\prod_{h=1}^{H} \hat \J_{h} \mathbb{P} \right)\left(\hat V _{H+1}- V^{\hat \pi}_{H+1}\right) -
\sum_{h=1}^{H}\left(\prod_{i=1}^{h-1} \hat \J_{i} \mathbb{P}\right) \hat {\J}_{h} \iota_{h}
\$
Using $\hat V _{H+1}=0$ by Line \ref{algo:initvhp} of Algorithm \ref{algo} and $V^{\hat \pi}_{H+1}=0$, we obtain 
\#\label{eq:recurssionhh}
\hat V_1 - V^{\hat \pi}_1 = -
\sum_{h=1}^{H}\left(\prod_{i=1}^{h-1} \hat \J_{i} \mathbb{P}\right) \hat {\J}_{h} \iota_{h}.
\#
By definition of $\hat\J_h$ in \cref{eq:defjh}, we write \cref{eq:recurssionhh} in the form of expectation w.r.t. the policy $\hat \pi$, and we have for all state $x\in \cS$
\# \label{eq:errorde2}
\hat V_1(x) - V^{\hat \pi}_1 (x) = - \sumh \E_{\hat \pi} [\iota _h (x_h,a_h) \cond x_1=x] .
\#
Putting together \cref{eq:errorde1} and \cref{eq:errorde2} completes the proof of Lemma \ref{lm:errordecomposition}.
\end{proof}

\subsection{Proof of Theorem~\ref{thm:convergence}}
\label{pf:thm:convergence}
We define
\$
& \mu_A = \sigma_{\min}\big(\sqrt{A^\top A}\big), \quad&&
L_A = \sigma_{\max}\big(\sqrt{A^\top A}\big),
\\
& \mu_B = \sigma_{\min}(B), \quad&&
L_B = \sigma_{\max}(B)\,,
\$
where for a symmetric positive definite matrix $M$, the matrix $\sqrt{M}$ is the unique matrix such that $M=\sqrt{M}\sqrt{M}$.
Recall the update rule in Algorithm \ref{algo} is 
\# \label{eq:210118}
W_{t+1}=  W_{t}-\eta^\theta_t \cdot (K_t \psi_t) \phi_t ^\top, \quad K_{t+1}= K_t + \eta^\omega_t\cdot (K_t \psi_t + x'_t  -W_t\phi_t) \psi_t ^\top \,. 
\#

Recall the saddle-point problem \cref{eq:theminimax} and we denote the saddle-point function by $\Phi_i$, i.e.,
\#\Phi_i(\theta, \omega)  \defeq \theta ^\top A^\top \omega - b_i^\top \omega - \tfrac{1}{2} \omega ^\top B \omega , \label{eq:def:Phi}\#
where $b_i = \E[x_i\psi(z)^\top]$. 
Given $\Phi_i$ defined above, we optimize out the dual variable, and define the primal function $P_i$ and the optimal dual variable $\hat \omega_i$ as follows.
\#
P_i(\theta) &=  \max_{\ome} \Phi_i(\th,\ome)
=\tfrac{1}{2}(A\th - b_i)^\top B^{-1} (A\th - b_i)
\label{eq:def:P}
\\
\hat \ome _i(\theta) &=  \argmax_{\ome} \Phi_i(\th,\ome) 
=
B^{-1} (A\theta - b_i).
\notag
\#
Uniqueness of $\hat \ome _i(\theta)$ is guaranteed by on the full-rankness of $A$ and $B$ (Assumption \ref{as:nonddualfeature}).
Define by $(\theta^{\sad}_i, \omega^{\sad}_i)$ the saddle-point of the convex-concave function $\Phi_i$. Then we have
\$ 
\theta^\sad_i = \argmin_{\theta}P_i(\theta),\quad  \omega^\sad_i = \hat \ome_i (\theta^*_i).
\$

Due to the separable structure of the update \cref{eq:210118}, if we denote the iterates $(W_t,K_t)$ by $W_t=[\theta_{1,t},\dots, \theta_{d_x,t}]^\top$ and $K_t = [\omega_{1,t},\dots,\omega_{d_x,t}]^\top$, then we can equivalently write the update as follows. For $i=1,\ldots, d_x$,
\# 
\theta_{i,t+1} 
&=
 \theta_{i,t} - \etathet \Tilde{\nabla}_\th \Phi_i(\theta_{i,t},\omega_{i,t})  \notag
\\
&=
\theta_{i,t} -\etathet (\phi(x_t,a_t) \psi(z_t)^\top) \omega_{i,t}  \label{eq:thetaupdate}
\\
\omega_{i,t+1} 
&=
 \omega_{i,t} + \etaomet \Tilde{\nabla}_\ome \Phi_i(\theta_{i,t},\omega_{i,t})   \notag
\\&=
 \omega_{i,t} + \etaomet (\phi(x_t,a_t)^\top \theta_{i,t} -x_{i,t}' - \psi(z_t)^\top \omega_{i,t}  )\psi(z_t) \,.
 \label{eq:omegaupdate}
\#

 Denote by $( W^\sad, K^\sad )$ the saddle-point of the problem \cref{eq:estimatorw}. Let $( \theta^\sad,  \omega^\sad)$ be the saddle-point of $\Phi_i$ in \cref{eq:def:Phi}. Since the minimax problem \cref{eq:estimatorw} is separable in the each coordinate of the primal and the dual variables, we have $ \theta^\sad = { W_i^\sad}$ and $ \omega ^\sad=  K_i^\sad$, for all $i=1,\ldots, d_x$, where $ W_i^\sad$ is the $i$-th row of the matrix $ W^\sad$, and $ K_i^\sad$ is the $i$-the row of $ K^\sad$. So we turn to study the convergence of $\{ \theta_{i,t},\omega_{i,t}\}_t$ to the saddle-point of $\Phi_i$.  

In the rest of the discussion we will ignore the subscript $i$ in $\ome_{i,t},\theta_{i,t}, x_{i,t}, x'_{i,t},\Phi_i, P_i,\hat \ome_i$ and $b_i$. Define the gradient of $\Phi$ evaluated at $(\thet,\omet)$, $\nabla_\theta \Phi$ and $\nabla_\omega \Phi$, and its stochastic version given a new data tuple $\xi_t=(x_t, a_t, z_t, x'_t)$, $\Tilde{\nabla}_\th \Phi$ and $\Tilde{\nabla}_\ome \Phi$, by
\# \label{eq:def:sg}
\nabla_\th \Phi(\thet,\omet) &= A^\top \omet,
\quad
&& \Tilde{\nabla}_\th \Phi(\thet,\omet;\xi_t) =  (\phi(x_t,a_t) \psi(z_t)^\top) \omega_{t} 
\\
\nabla_\ome \Phi(\thet,\omet) &= A\thet - b - B\omet, 
\notag
\quad
&& \Tilde{\nabla}_\ome \Phi(\thet,\omet;\xi_t) =  ( \phi(x_t,a_t)^\top \theta_{t} -x_{t}' - \psi(z_t)^\top \omega_{t}  )\psi(z_t) .
\notag
\#
 We will ignore the dependence of $\natthephi$ and $\natomephi$ on $\xi_t$ from now on. Define the auxiliary update sequences given the stochastic update sequence $\{ \thet,\omet\}$ in \cref{eq:thetaupdate} and \cref{eq:omegaupdate},
\$
\Tilde{\theta}_{t+1} &= \thet - \etathet {\nabla}_\th \Phi(\thet,\omet)
&&= \thet -\etathet A^\top \omet
\\
\hat{\theta}_{t+1} &= \thet - \etathet {\nabla P}(\thet )
&&= \thet - \etathet A^\top B^{-1} (A\thet - b) \,,
\\
\Tilde{\omega}_{t+1} &= \omet +\etaomet \nabla \Phi(\thet,\omet)
&&=   \omet + \etaomet (A\thet - b - B\omet)\,.
\$
 Define the $\sigma$-algebras $\cF_0 = \sigma\{\th_0,\ome_0 \}$, and $\cF_t =\sigma \{\th_0,\ome_0, \{ x_j,a_j,z_j,x'_j\}_{j=0}^{t-1}\}$ for $t=1,\ldots, T$. Note $\xi_{t-1} \in \cF_t$ but $\xi_t\notin \cF_t$. Note that for all $t\geq 1$, the random variables $\xi_{t-1}$, $\theta_t,\omet, \Tilde{\theta}_{t+1}, \Tilde{\omega}_{t+1}$ and $\hat{\theta}_{t+1}$ are deterministic given $\cF_t$, and we obviously have 
 \$\E[\Tilde{\nabla}_\th \Phi(\thet,\omet)  \cond \cF_t]= \nabla_\th \Phi(\thet,\omet)  
 \quad \text{and}\quad
 \E[\Tilde{\nabla}_\ome \Phi(\thet,\omet)  \cond \cF_t]= \nabla_\ome \Phi(\thet,\omet) .\$
We will denote $\E_t[\cdot]=\E[\cdot \cond \cF_t]$.

We start with some basic observations of the functions $P$ and $\Phi$.

\begin{lemma} \label{lm:obs}
Consider the functions $P$ in \cref{eq:def:P} and $\Phi$ in \cref{eq:def:Phi}.
\begin{enumerate}
   
    \item  \label{lm:obs:primalkappa} Recall $\muIV$ and $L_P$ are the minimum and the maximum eigenvalues of the matrix $A^\top B\inv A$, respectively. Then the function $P$ is $\muIV$-strongly convex and $L_P$-smooth. Moreover, we have $\muIV \geq \mu_A^2 / L_B$, and $L_P \leq \min\{ 1, L_A^2 / \mu_B\}$.
    
    \item  \label{lm:obs:dualkappa} For any fixed $\theta$, the function $\ome\mapsto -\Phi(\theta,\ome)$ is $\mu_B$-strongly convex and $\LB$ smooth.
    
     \item (Proposition~\ref{prop:wstartequalwsad})\label{lm:obs:identification} Assumptions~\ref{as:spec} and \ref{as:nonddualfeature} imply the existence and uniqueness of a  matrix $W^*=[W_1^*,\ldots,W_{d_x}^*]\in\R^{d_x\times d_{\phi}}$ such that $\E[W^*\phi(x,a)\cond z] = \E[x'\cond z]$. 
    Assumption \ref{as:nonddualfeature} implies the uniqueness of the saddle-point $( \theta^\sad, \omega^\sad)= \argmin_{\theta\in\R^{d_\phi}} \max_{\omega \in \R^{d_\psi}} \Phi_i(\th,\ome)$.
    Furthermore, in addition to Assumptions~\ref{as:spec} and \ref{as:nonddualfeature}, if Assumption \ref{as:dualapprox} holds, then $W^*_i = \theta ^\sad$ and $\omega^\sad = \hat \omega_i(\hat \theta)=0$.
\end{enumerate}
\begin{proof}
See \cref{pf:lm:obs}.
\end{proof}
\end{lemma}

Item~\ref{lm:obs:identification} above shows that under the assumptions listed in Theorem~\ref{thm:convergence}, the saddle-point of $\Phi_i$ equals to the $i$-th row of the unknown transition matrix $W^*$. To emphasize this we now define by $(\thest,\omest)$ the saddle-point of the function $\Phi$.
Next we present some descent lemmas about the sequence $\{ \thet, \omet\}$. Denote the second moment of the stochastic gradient evaluated at the saddle-point of $\Phi$, $(\thest,\omest)$ by \$\signathesq = \E[ \|\Tilde{\nabla}_\theta \Phi(\thest,\omest)\|^2] \quad \text{and} \quad \signaomesq = \E[ \|\Tilde{\nabla}_\ome \Phi(\thest,\omest)\|^2],\$ where $\Tilde{\nabla}_\theta \Phi$ and $\Tilde{\nabla}_\ome \Phi$ are defined in \cref{eq:def:sg}. First we show the variance of stochastic gradient can be bounded by the suboptimality of the current iterate.

\begin{lemma}[Bounding variance of stochastic gradients]
 \label{lm:obs:varboound} Consider the sequence $\{ \omet, \thet\}$. If \cref{as:boundedfeature} holds, then
\#
\E_t \big[ & \|\Tilde{\nabla}_\theta \Phi(\thet,\omet) - {\nabla}_\theta \Phi(\thet,\omet)\|^2\big]  \notag
\\ &\leq
4 ( \muB\inv \| \thet - \thest\|^2 + \| \omet - \omehat(\thet)\|^2    ) + 2 \signathesq,
\label{eq:vartheta}
\\
\E_t \big[ & \|\Tilde{\nabla}_\ome \Phi(\thet,\omet) - {\nabla}_\ome \Phi(\thet,\omet)\|^2\big]  \notag
\\ 
& \leq
16 ( \muB\inv\| \thet - \thest\|^2 +  \| \omet - \omehat (\thet)\|^2 ) + 2 \signaomesq.
\label{eq:varomega}
\#
where we condition on $\cF_t$ and take expectation over the new data tuple $\xi_t$.
\end{lemma}
\begin{proof}
See \cref{sec:pf:lm:obs:varboound}.
\end{proof}

\begin{lemma}[One-step descent of primal update] \label{lm:primal:descent}
Consider the update sequence $\{\omet,\thet\}$. Let \ref{as:boundedfeature} (bounded feature map) and \ref{as:nonddualfeature} hold.
{If $\etathet \leq \frac{2}{\muIV + \LP}$}, then 
\$
\E \big[ \| \theta_{t+1} - \theta^*\|^2\big] 
&  \leq 
(1- \muIV \eta_t^\theta  +4 \muB\inv \etathetsq )
\cdot \E\big[ \| \theta_{t} - \theta^*\|^2\big] 
\\
& \quad +
(\muIV^{-1}  \eta^\theta_t + 4 \etathetsq)\cdot \E\big[\| \omega_t - \hat \omega (\theta_t)\|^2 \big]  
\\ 
& \quad +
2 (\eta^\theta_t)^2 \cdot \sigma^2_{\nabla \theta}
\$
\end{lemma} 
\begin{proof}
See \cref{sec:pf:lm:primal:descent}.
\end{proof}

\begin{lemma}[One-step descent of dual update] 
\label{lm:dual:descent} Consider the update sequence $\{ \omet, \thet\}$. Let \ref{as:boundedfeature}~and~\ref{as:nonddualfeature} hold. {If $\etathet \leq \frac{2}{\muB+ \LB}$}, then
\#
\E\big[\| \ometp - \hat \omega(\thetp) \|^2\big]
\leq &  
\big(1-\mu_B \etaomet + 
32 (
\muB\invsq{ (\etathet)^2}{(\etaomet)\inv}  + \etaometsq + \muB\inv\etathetsq
)
\big) \cdot \E \big[\| \omet - \hat \omega (\thet)\|^2\big]  
\notag
\\
& +  32
\big(
\muB\invsq{ (\etathet)^2}{(\etaomet)\inv}  + \muB\inv \etaometsq + \muB\invsq \etathetsq
\big)
\cdot \E \big[\| \thet - \theta^*\|^2\big] 
\notag
\\
& +32 \big( \etaometsq \signaomesq + \muB\inv\etathetsq  \signathesq \big)
\,. 
\label{eq:1027b}
\#
\end{lemma}
\begin{proof}
See \cref{sec:pf:lm:dual:descent}.
\end{proof}

Equipped with Lemmas \ref{lm:primal:descent} and \ref{lm:dual:descent}, we can derive a recursion by choosing appropriate stepsize sequences $\etaomet$ and $\etathet$. We set 
\$
\etathet = \frac{\beta}{\gamma +t},\quad \etaomet = \frac{\alpha \beta}{\gamma + t}
\$
for some positive $\alpha,\beta,\gamma$, which will be chosen later. For some positive $\lambda$ (to be chosen later) we define the potential function $P_t$ with $a_t = \E[\|\thet - \thest \|\sq]$ and $b_t=\E[ \| \omet - \hat \ome (\thet)\|\sq]$,
\$
P_t = a_t + \lambda b_t,
\$
and then derive a recursion formula for $P_t$. We have by Lemma \ref{lm:primal:descent} and \ref{lm:dual:descent},
\#
P_{t+1} 
&=a_{t+1}+\lambda b_{t+1}
\notag
\\
& \leq \big(1- \muIV \etathet + 2^5 (\lambda \alpha\inv \muB\invsq \etathet + \III) \big) a_t
\notag
\\
&
+
\big(1- \muB \etaomet + 2^5 ( \alpha \invsq \cdot \muB\invsq \etaomet +\lambda \inv \muIV \inv \etathet + \I)  \big) (\lambda b_t)
\notag
 \\
 &+ \II
\label{eq:recursion1}
\#
where
\#
\I &= \etaometsq + \muB\inv\etathetsq + \lambda\inv \etathetsq,
\notag
\\
\notag
\II & = 2 (\eta^\theta_t)^2 \cdot \sigma^2_{\nabla \theta} + 4 \lambda  \big( \muB\inv \etathetsq \signathesq + \etaometsq \signaomesq\big)
\notag
\\
\III &= \muB\inv \etathetsq + \lambda \muB\inv \etaometsq + \lambda \muB\invsq \etathetsq,
\notag
\#

Our strategy is straight-forward. We find a suitable choice of the free parameters $(\lambda, \gamma,\alpha,\beta)$ such the the sequence $\tilde{P}_t$ decays at the rate $1/t$.

\textbf{Step 1.} Choose $\gamma = \gamma(\alpha, \beta,\lambda)$ such that (i) the stepsize requirements in Lemmas \ref{lm:primal:descent} and \ref{lm:dual:descent} are met, and (ii) the two terms $2^5 \cdot \III$ and $2^5 \cdot \I$ are less than $\frac{1}{2} \muIV\etathet$ and $\frac{1}{2} \muB\etaomet$, respectively.

For any positive $\alpha,\beta,\lambda$, we pick $\gamma$ large enough such that the following inequalities hold for all $t\geq1$,
\$
& 2^5 \cdot \III \leq \frac{1}{2} \muIV\etathet
\\
& 2^5 \cdot \I \leq \frac{1}{2} \muB\etaomet
\$

Note $\eta^\theta_0 =\beta/\gamma$, and $\eta^\omega_0 = \alpha\beta/\gamma$. The above inequalities suggest it suffices to set $\gamma$ large enough. Concretely, for any fixed positive $(\alpha,\beta,\lambda)$ with, we can make $\gamma$ satisfy the following inequalities
\# \label{eq:chooselambda0}
\gamma & \geq 2^8 \cdot \max\{ 
\beta\cdot \muB\inv\muIV\inv, 
{\alpha\sq\lambda\beta \muB\inv\muIV\inv},
\beta\lambda\muB\invsq\muIV\inv,
\alpha\beta \muB\inv,
\alpha\inv\beta \muB\invsq, 
\alpha\inv\lambda\inv\beta\muB\inv \}
\#
To ensure the stepsizes are small enough to meet the conditions in Lemma \ref{lm:primal:descent} and \ref{lm:dual:descent} we need for all $t$,  
\$
\etathet \leq \frac{2}{\LP+\muIV},\quad \etaomet \leq \frac{2}{\LB + \muB},
\$
it suffices to control $\eta^\theta_0$ and $\eta^\omega_0$ by setting
\#
\gamma \geq \max\{ \beta, \alpha\beta\}. \label{eq:chooselambda1}
\#
For any fixed $(\alpha,\beta,\lambda)$, the inequalities \cref{eq:chooselambda0} and \cref{eq:chooselambda1} give the choice of $\gamma$.

\textbf{Step 2.} Pick $\alpha,\lambda$ such that the recursion reduces to the form $P_{t+1}\leq (1-\frac14 \muIV \etathet) P_t + \text{noise}$.
By the choice of $\gamma$ in Step 1 (Eq. \ref{eq:chooselambda0} and Eq. \ref{eq:chooselambda1}), the recursion \cref{eq:recursion1} reduces to 
\# \label{eq:recursionafterlambda}
P_{t+1} 
&=a_{t+1}+\lambda b_{t+1}
\\
& \leq (1- \frac12 \muIV \etathet + 2^5 (\lambda \alpha\inv \muB\inv \etathet ) ) a_t
\notag
\\
&
+
(1- \frac12 \muB \etaomet + 2^5 ( \alpha \invsq \cdot \muB\invsq \etaomet +\lambda \inv \muIV \inv\etathet )  ) (\lambda b_t)
\notag
 \\
 &+ \II
\notag
\#
We find $(\alpha,\lambda)$ such that
\#
& 2^5 (\lambda \alpha\inv \muB\invsq \etathet ) \leq \frac14 \muIV\etathet
\notag
\\
 & 2^5 ( \alpha \invsq \cdot \muB\invsq \etaomet +\lambda \inv \muIV\inv \etathet ) \leq \frac14 \muB\etaomet
\notag
\#
It suffices to set 
\#
& \lambda = \muB^{1/2} \label{eq:choicelambda}
\\
& \alpha = 2^8 \cdot \muB^{-1.5} \muIV^{-1} \label{eq:choicealpha}
\#

Together the choice of $\lambda ,\alpha$ in \cref{eq:choicelambda} and \cref{eq:choicealpha} implies that the recursion \cref{eq:recursionafterlambda} simplifies to
\# \label{eq:recursion:siglambdasig}
P_{t+1} & \leq \big(1-\frac14 \muIV\etathet\big)a_t + (1-\frac14 \muB \etaomet)(\lambda b_t) + \II
\\
& \leq  \big(1-\frac14 \muIV\etathet \big) P_t +\Big( 2 (\eta^\theta_t)^2 \cdot \sigma^2_{\nabla \theta} + 4\lambda  \big(\muB\inv \etathetsq \signathesq + \etaometsq \signaomesq\big)\Big)\,, 
\label{eq:0129555}
\#
where we used $1-\frac14 \muB\etaomet \leq 1-\frac14 \muIV\etathet$ because \cref{eq:choicealpha} implies $\alpha \geq \muIV\muB\inv$.

Next we bound the last term in \cref{eq:0129555}. Now we study $\signathesq, \signaomesq$. By Item~\ref{lm:obs:identification} of Lemma~\ref{lm:obs}, we have the primal variable in the saddle-point of the minimax problem \cref{eq:def:Phi} equals to the truth that generates the data, i.e., we have $x_t' = x_{t+1} = (\thest) \cdot \phi(x_t,a_t)+e_t$, and that $\omest =0$. The variances of the gradient at the optima $(\thest, \omest)$ are 
\$
\signathesq&= \E_{\xi_t} \big[\| \Tilde{\nabla}_\th \Phi(\thest,\omest;\xi_t)\|\sq\big] 
\\
& =  \E\big[\|(\phi(x_t,a_t) \psi(z_t)^\top) \omest  \|\sq\big] 
\\
& = 0
\$
and 
\$
\signaomesq &= \E_{\xi_t}\big[\| \Tilde{\nabla}_\ome \Phi(\thest,\omest;\xi_t) \|\sq\big] 
\\
& =  \E\big[ \| ( \phi(x_t,a_t)^\top \thest -x_{t}' - \psi(z_t)^\top \omest  )\psi(z_t) \|\sq\big] 
\\
&  = \E\big[\|e_t \psi(z_t) \|\sq\big]
\\ 
& \leq \E[  e_t^2] = \sigma^2
\$
where we have used $\sup_z \|\psi(z)\|_2 \leq 1$ by \ref{as:boundedfeature}. This implies 
\$
2 (\eta^\theta_t)^2 \cdot \sigma^2_{\nabla \theta} + 4 \lambda  \big(\muB\inv \etathetsq \signathesq + \etaometsq \signaomesq\big) = \lambda \cdot 4 \etaometsq \cdot \sigma^2.
\$
We now restore the omitted state dimension index $i$, and the recursion \cref{eq:recursion:siglambdasig} writes
\$ 
\E\big[\|& \theta_{t+1,i} - \theta_{i}^*\|\sq\big] + \lambda \E\big[\|  \omega_{t+1,i}-\hat \omega_i(\theta_{t+1}) \|\sq\big]
\\
& \leq  
\big(1-\frac14 \muIV\etathet \big)  \Big( \E[\| \theta_{t,i} - \theta_{i}^*\|\sq] + \lambda \E\big[\|  \omega_{t,i}-\hat \omega_i(\theta_{t}) \|\sq\big] \Big) + \lambda \cdot 4 \etaometsq \cdot \sigma^2 . 
\$
Summing over $i=1,\ldots, d_x$, we have a recursion formula on the sequence $\Tilde{P_t}=\E[\|W_t-W^*\|_F^2]+\lambda \E[\|K_t-\hat K (W_t) \|_F^2]$.
\#
\Tilde{P}_{t+1} \leq (1-\frac14 \muIV\etathet) \Tilde{P_t} +  \lambda \cdot 4 \etaometsq \cdot d_x \sigma^2 .
\label{eq:readyforinduction}
\#

\textbf{Step 3.} Pick $\beta,\nu$ such that $\Tilde{P_t}=O(\nu t\inv)$. Set  
\$
& \beta = 8\muIV\inv,
\\
& \nu = \max \Big \{\gamma \Tilde{P}_0, \,\, \big(\frac14 \muIV\beta-1 \big)\inv \beta\sq \alpha\sq\lambda \cdot d_x \sigma^2  \Big \} = \max\{\gamma \Tilde{P}_0,\,\,\text{const.}\times \muIV^{-4} \muB^{-2.5}\}.
\$
Together with our choice of $\alpha$ in \cref{eq:choicealpha} and $\lambda$ in \cref{eq:choicelambda}, we have the following choice of $\gamma$ (Eq. \ref{eq:chooselambda0} and Eq. \ref{eq:chooselambda1})
\$
\gamma = 2^8\cdot \alpha\sq \beta \lambda\cdot \muB\inv\muIV\inv = \text{const.}\times \muIV^{-4} \muB^{-3.5}.
\$

Next, we claim for all $t\geq 0$,
\# \label{eq:Ptinduction}
\Tilde{P}_t \leq \frac{\nu}{\gamma + t}\,.
\#
We prove by induction. For the base case $t=0$, the inequality \cref{eq:Ptinduction} holds by definition of $\nu$. Next, assume for some $t\geq 0$, the inequality \cref{eq:Ptinduction} holds. We investigate $P_{t+1}$. 
By the recursion formula \cref{eq:readyforinduction},
\#
\Tilde{P}_{t+1} & \leq (1-\frac14 \muIV \etathet) \Tilde{P}_t +  \lambda \cdot 4 \etaometsq \cdot d_x \sigma^2
\label{eq:11307a}
\\
& \leq \frac{\gamma + t - \frac14 \muIV \beta}{\gamma + t} \cdot \frac{\nu}{\gamma + t} + \lambda \frac{4\alpha\sq\beta\sq\cdot d_x \sigma^2 }{(\gamma+t)\sq}
\label{eq:11307b}
\\
& =   \frac{(\gamma + t - 1)\nu}{(\gamma + t)\sq } - \frac{(\frac14 \muIV \beta - 1)\nu}{(\gamma + t)\sq}  + \lambda \frac{4\alpha\sq\beta\sq\cdot d_x \sigma^2 }{(\gamma+t)\sq}
\label{eq:11307c}
\\
& \leq \frac{\nu}{\gamma+t+1}.
\label{eq:11307d}
\#
where \cref{eq:11307a} holds due to the recursion formula \cref{eq:readyforinduction};
\cref{eq:11307b} holds due to the induction assumption that $\Tilde{P_t}\leq \nu /(\gamma + t)$;
\cref{eq:11307c} holds because (i) $4\inv \muIV \beta - 1=1\geq 0$ by our choice of $\beta$, and (ii) the definition of $\nu$ ensures the sum of last two terms in \cref{eq:11307c} is negative;
\cref{eq:11307d} holds because $(\gamma+t-1)/(\gamma+t)\sq\leq (\gamma+t+1)\inv$.
This proves the claim \cref{eq:Ptinduction}.

This proves Theorem~\ref{thm:convergence}.

\subsection{Proof of Theorem \ref{thm:convergence} (ii)}
\begin{proof}
We recall the error decomposition of $\Vstar - \Vpihat$ presented in Lemma \ref{lm:errordecomposition}.
Conditioning on the training data, the matrix $W_T$ and the functions $\{\iota_h\}_h$ are deterministic. Recall $\xi_h=\langle \hat Q_h, \pi^*_h - \hat \pi _h \rangle_\cA $ for all $x\in \cS$, and $\iota_h = (r_h + \P\Vhat_{h+1}) - \hat Q_h$ for all $(x,a)\in \cS \times \cA$. 
First by definition of $\xi_h=\langle \hat Q_h, \pi^*_h - \hat \pi _h \rangle_\cA$ and that $\hat \pi _h$ is greedy w.r.t.\ $\Qhath$, we have
\$
\sumh \E_{\pi^*} [ \xi_h(x_h) \cond x_1=x ] \leq 0 \quad \text{for all } x.
\$
Based on the error decomposition of  $\Vstar - \Vpihat$ (Lemma \ref{lm:errordecomposition}), we have for all $(x,a)$,
\#
\|\Vstar - \Vpihat \|_\infty &= \sup_{x} \Vstar(x) - \Vpihat(x)
\notag
\\
&\leq \sup_x \Bigg\{ 
\sumh \E_{\pi^*} [\iota _h (x_h,a_h) \cond x_1=x] + \sumh \E_{\hat \pi} [\iota _h (x_h,a_h) \cond x_1=x]
\Bigg\}. \label{eq:decomp:noxi}
\#
Next we derive an upper bound for $\| \iota_h\|_\infty = \sup_{x,a}|\iota_h(x,a) |$.
\#
\sup_{x,a} | \iota_h(x,a)| & = \sup_{x,a} \Big| (r_h + \P\Vhathp) - \Qhath \Big| 
\notag
\\
& =  \sup_{x,a} \Big| (r_h + \P\Vhathp) -  (r_h  + \hat \P \Vhathp) \Big|
\label{eq:118856b}
\\
& =  \sup_{x,a} \Big| \P\Vhathp -   \hat \P \Vhathp \Big|
\notag
\\
& \leq  \sup_{x,a} \Bigg\{ \sqrt{\E_{x'\sim \cP_{W^*}(\cdot \cond x,a)}\big[\Vhathp(x') \sq \big]} \cdot \min\Big( \frac{\|(W_T - W^*)\phi(x,a) \|_2}{\sigma},  1\Big) \Bigg\}
\label{eq:118856d}
\\
& \leq  \min \Big\{ \frac{\|W_T - W^* \|}{\sigma},  1\Big\}\cdot H.
\label{eq:118856e}
\#
Here \cref{eq:118856b} holds by definition of $\Qhath$
\cref{eq:118856d} holds due to Lemma \ref{lm:gausdiff}; recall $\cP_{W}(x'\cond x,a)$ is the probability density of multivariate Normal with mean $ W \phi(x,a)$ and variance $\sigma^2 I_{d_x}$.
\cref{eq:118856e} holds because for all $h\in[H]$ we have $\Vhath\leq H$, and that $\| (W_T-W^*)\phi(x,a)\|\leq \|W_T-W^*\|\|\phi(x,a)\|$. Note for all $(x,a)$ we have $\| \phi(x,a)\|\leq 1$ (\cref{as:boundedfeature}).

Next we continue from \cref{eq:decomp:noxi}.
\$
\|\Vstar - \Vpihat \|_\infty & \leq \sup_x \Bigg\{  \sumh \E_{\pi^*} \big[ \| \iota _h\|_\infty \cond x_1=x\big] + \sumh \E_{\hat \pi} \big[ \| \iota _h\|_\infty \cond x_1=x\big] \Bigg\}
\\
& \leq 2H\cdot \max_{h\in[H]}\|\iota_h\|_\infty
\\
& \leq 2H\sq \cdot \min\Big\{ \frac{\|W_T - W^* \|}{\sigma},  1\Big\}  \leq 2H \sq \sigma\inv \cdot \| W_T-W^*\|.
\$
Now we take expectation on both sides w.r.t.\ the sampling process, we have
\$
\E \big[ \|\Vstar - \Vpihat \|_\infty  \big] 
& \leq 2H\sq \sigma\inv \cdot \E\big[\| W_T-W^*\|\big]
\\
& \leq  2H\sq \sigma\inv \cdot \sqrt{\E \big[\| W_T-W^*\|^2_F\big]} \leq  2H\sq \sigma\inv \sqrt{\frac{\nu}{\gamma + T}}\,.
\$
Note we trivially have $\|\Vstar - \Vpihat\|_\infty\leq H $. So we conclude
\$
\E \big[ \|\Vstar - \Vpihat \|_\infty \big] \leq  H \cdot \min \Bigg \{ 2H\sigma\inv \sqrt{\frac{\nu}{\gamma + T}},1   \Bigg\}.
\$
This completes the proof of Theorem~\ref{thm:convergence} (ii).
\end{proof}

\subsection{Proof of Theorem~\ref{lm:misspecification}} \label{pf:lm:misspecification}
\begin{proof}[Proof of Theorem~\ref{lm:misspecification}]
Denote $\theta^\sad = W^\sad_i$. We omit the subscript $i$ in $f^*_i$ and $x_i'$. This theorem studies the relation between the two quantities:
\begin{itemize}
    \item An element in the primal function space, $ \phi \cdot \thest \in \Hphi$, where $\theta^*$ solves the following minimax problem.
\# \label{eq:1210848}
\min_{f\in \Hphi} \max_{u\in \Hpsi} \E[(f(x,a) - x') u(z)] - \frac12 \E[u(z)\sq].
\#
\item The truth $f^*$ that satisfies $\E[f^*(x,a)\cond z] = \E[x'\cond z]$.
\end{itemize}

It can be verified that the optimal primal variable of the above minimax problem \cref{eq:1210848} exists and is unique. Specifically, for $f = \theta \cdot \psi \in \Hphi$, due to \ref{as:nonddualfeature}, the inner maximization is uniquely attained at 
\$
\psi \cdot \hat{\omega}(\theta) \in \Hpsi,\quad
\hat{\omega}(\theta) \defeq \E[\psi(z)\psi(z)^\top] \inv \E\big[\psi(z) \cdot (f(x,a) - x')\big] .
\$
Also note 
\$
\psi \cdot \hat{\omega}(\theta)= \Pi_\psi\cT ( \theta\cdot \phi - f^*)
\$
due to the definition of the projection operator $\Pi_\psi:\ltz \to \Hpsi$, defined by for all $u\in \ltz$,
\$
\Pi_\psi u = \argmin_{u' \in \Hpsi} \| u - u'\|_{\ltz  } = \psi^\top \E[\psi(z)\psi(z)^\top]\inv \E[\psi(z) u(z)].
\$ Now we plug in the optimal value and define, for $f\in \Hphi$,
\$L(f) & \defeq  \max_{u\in \Hpsi} \E\big[(f(x,a) - x') u(z)\big] - \frac12 \E[u(z)\sq]
\\
& = \frac12 \E[\psi(z) \cdot (f(x,a) - x')] ^\top B\inv \E[\psi(z) \cdot (f(x,a) - x')]
\\
& =\frac12 \| \Pi_\psi \cT(f-f^*) \|_{\ltz}^2.
\$
The unique minimizer of $L(f)$ over $\Hphi$ is
\$
 \phi \cdot \theta^\sad\in\Hphi,\quad \theta^\sad =  [ A^\top B\inv A ]\inv A^\top B\inv \E[\psi(z)x'] \in \R^{d_\phi}.
\$
Note
\$
Qf^* = \phi \cdot \theta^\sad
\$
by definition of the operator $Q$ in Theorem~\ref{lm:misspecification}.
We define $\hat f = \Pi_\phi f^*$, the projection of $f^*$ onto $\Hphi$ w.r.t\ the norm $\|\cdot\|_{\ltxa}$. We have the decomposition
\$
\| f^* - \theta^\sad \cdot \phi \|_{\ltxa}\leq \| f^* - \hat f \|_{\ltxa} + \| \hat f - \theta^\sad \cdot \phi \|_{\ltxa}.
\$
For the first term we have $\| f^* - \hat f \|_{\ltxa}\leq \eta_1$ by definition of $\eta_1$. For the second term, we further decompose and use the definition of $\muIV$ and Proposition~\ref{lm:mupisillposenss}.
\#
\| \hat f - & \theta^\sad \cdot\phi\|_{\ltxa}
\notag
\\
&\leq \| \hat f - \theta^\sad \cdot\phi\|_\phi
\label{eq:1129:a}
\\ 
& \leq \muIV\inv \cdot\|\cT (  \hat f - \theta^\sad \cdot\phi) \|_{\ltz}
\label{eq:1129:b}
\\
& \leq \muIV\inv \cdot \big( \|\cT (  \hat f - f^*) \|_{\ltz} + \|\cT (  f^* - \theta^\sad \cdot\phi) \|_{\ltz} \big)
\label{eq:1129:c}
\\
& \leq \muIV\inv \cdot \big( \|\cT (  \hat f - f^*) \|_{\ltz} + \|\Pi_\psi\cT (  f^* - \theta^\sad \cdot\phi) \|_{\ltz} + \eta_2\cdot \mu \big)
\label{eq:1129:d}
\\
& \leq \muIV\inv \cdot \big( \|\cT (  \hat f - f^*) \|_{\ltz} + \|\Pi_\psi\cT (  f^* - \hat{f} \,) \|_{\ltz} + \eta_2\cdot \mu \big)
\label{eq:1129:e}
\\
&\leq \muIV\inv \cdot \big( 2\|\cT (  \hat f - f^*) \|_{\ltz}  + \eta_2\cdot \mu \big)
\label{eq:1129:f}
\\
& \leq  2c \cdot\eta_1 +  \muIV\inv \cdot   \eta_2\cdot \mu .
\label{eq:1129:g}
\#
Here
\cref{eq:1129:a} follows since $\phi$ is bounded;
\cref{eq:1129:b} follows by definition of $\muIV$;
\cref{eq:1129:c} follows since $T$ is linear and we use I inequality;
\cref{eq:1129:d} follows by definition of $\eta_2$ and $\mu$;
\cref{eq:1129:e} follows because $\phi^\top \thest$ minimizes $f\mapsto \|\Pi_\psi\cT ( f^* - f)\|_\ltz\sq$ over $\Hphi$ and that $\hat f \in \Hphi$;
\cref{eq:1129:f} follows because the projection operator is non-expansive;
\cref{eq:1129:g} follows by definition of the constant~$c$.

This completes the proof of Theorem~\ref{lm:misspecification}.
\end{proof}
  

\section{Proofs of Lemmas in \cref{app_proofs_of_theorems}}

\subsection{Proof of Lemma~\ref{lm:obs}}
\label{pf:lm:obs}
\begin{proof}

\textbf{Proof of Item~\ref{lm:obs:primalkappa} in Lemma~\ref{lm:obs}.} For strong convexity, we show that the minimum eigenvalue of $\nabla^2P(\theta)$ and is lower bounded by $\muA\sq\LB\inv$. Since the matrix $B$ is full rank (\cref{as:nonddualfeature}) and thus its inverse $B\inv$ has a unique square root $B^{-1/2}$ such that $B\inv =B^{-1/2}B^{-1/2}$. For any $w\in \R^{d_\psi}$ with unit norm we have $\| B^{-1/2}w \| \geq \LB^{-1/2}$. For any $v \in \R^{d_\phi}$ such that $\|v\|=1$,
\$
v^\top \nabla^2P(\theta) v
& = v^\top A^\top B^{-1} A v
= v^\top A^\top B^{-1/2} B^{-1/2} A v 
\\
& = \|B^{-1/2} A v\|^2 \geq \LB\inv \|Av\|^2 \geq \muA\sq \LB\inv 
\$
where we have used the fact that the matrix $A$ has full column rank (\cref{as:nonddualfeature}, $\operatorname{rank}(A)=d_\phi$) and thus for any $u\in \R^{d_\phi}$ such that $\|u\|=1$ we have $\|Au\|\geq \muA $. The proof of $\LP\leq \LA\sq\muB\inv$ follows by similar reasoning. To see $\LP\leq1$, recall $D = \E[\phi(x,a)\phi(x,a)^\top]$. We note
\$
\| A^\top B\inv A\| = \| D^{1/2}(D^{-1/2}A^\top B^{-1/2}) (B^{-1/2}A^\top D^{-1/2})D^{1/2}\| \leq \| D\| \leq 1,
\$
where we have used $\| D^{-1/2}A^\top B^{-1/2}\| \leq 1$ and by \ref{as:boundedfeature} $\|D \|\leq 1$.

\noindent \textbf{Proof of Item~\ref{lm:obs:dualkappa} in Lemma~\ref{lm:obs}.} This is obvious by noting for any $\theta$, $\nabla^2_\ome \Phi(\theta,\omega)= - B$ and that the matrix $B$ satisfies $\mu_B I_{d_\psi} \preceq B\preceq L_BI_{d_\psi}$ with $0<\muB$ (\cref{as:nonddualfeature}).

\noindent\textbf{Proof of Item~\ref{lm:obs:identification} in Lemma~\ref{lm:obs}.} The existence of $W^*$ such that $\E[W^*\phi(x,a)\cond z] = \E[x'\cond x,a]$ is guaranteed by \cref{as:spec}. From this equation, we multiply both sides by $\E[\phi(x,a)\given z]$ and take expectation w.r.t\ $z$, we obtain \$W \E[\E[\phi(x,a)\given z]\times \E[\phi(x,a)\given z]] = \E[\E[x'\given x,a]\times \E[\phi(x,a)\given z]].\$ So if the matrix $\E[\E[\phi(x,a)\given z]\times \E[\phi(x,a)\given z]]$ is invertible then $W^*$ is the unique solution to the above equation. Such invertibility is implied by \cref{as:nonddualfeature}. 

Next we show the existence and uniqueness of the saddle-point of $\Phi_i$. For any fixed $\theta$, by full-rankness of $B$ (\cref{as:nonddualfeature}), the map $\omega\mapsto \Phi_i(\theta,\omega)$ is uniquely maximized at $\omega =\hat \omega_i(\theta)=B\inv (A\theta - b_i)$. Recall $P_i(\theta) = \max_\ome \Phi_i(\th,\ome)=\tfrac{1}{2}(A\th - b_i)^\top B^{-1} (A\th - b_i)$. By Item~\ref{lm:obs:primalkappa} of Lemma~\ref{lm:obs}, the minimum eigenvalue of $\nabla^2 P$ is bounded away from zero due to full-rankness of $A$ and $B$ (\cref{as:nonddualfeature}). Thus $P$ has a unique minimizer. 

Next, we show $W^*_i = \theta^\sad$. \ref{as:dualapprox} implies $\eta_2$ in Theorem~\ref{lm:misspecification} is zero. \ref{as:spec} implies $\eta_1$ in Theorem~\ref{lm:misspecification} is zero. So Theorem~\ref{lm:misspecification} implies $W^*_i = \theta^\sad$.



Finally we show $\hat \omega_i(\theta^\sad)=0$. Recall $\hat \omega_i(\theta) = B\inv (A\theta - b_i)$ for any $\theta \in \R^{d_\phi}$. Recall $b_i$ is defined as $b_i = \E[x'_i \psi(z)]$. Since $\theta^\sad=W^*_i$, we have 
\$A\theta^\sad - b_i & = \E[\psi(z)(\phi(x,a)^\top \theta^\sad - x'_i)] = \E[\psi(z)(\phi(x,a)^\top W^*_i - x'_i)]=\E[\psi(z)e_i]
\\
& =\E[\psi(z)\E[e_i\cond z]]=0.
\$
We conclude $\hat \omega_i(\theta^\sad)=0$.
\end{proof}

\subsection{Proof of Lemma~\ref{lm:obs:varboound}} \label{sec:pf:lm:obs:varboound}
\begin{proof}[Proof of Lemma~\ref{lm:obs:varboound}]
For the inequality \cref{eq:vartheta}, conditioning on $\cF_t$, we take expectation over the new data $\xi_t=(x_t,a_t,z_t,x'_t=x_{t+1})$ (note $\xi_t \notin \cF_t$)
\$
\E_t[\|\Tilde{\nabla}_\theta \Phi(\thet,\omet) - {\nabla}_\theta \Phi(\thet,\omet)\|^2]  &\leq
\E_t[\| \Tilde{\nabla}_\theta \Phi(\thet,\omet)\|^2]
\\
&\leq 2 \E_t[\| \Tilde{\nabla}_\theta \Phi(\thet,\omet)- \Tilde{\nabla}_\theta \Phi(\thest,\omest) \|\sq] + 2 \E_t[\| \Tilde{\nabla}_\theta \Phi(\thest,\omest) \|\sq]
\$
For the first term we use that $\Tilde{\nabla}_\th \Phi(\thet,\omet;\xi_t) =  (\phi(x_t,a_t) \psi(z_t)^\top) \omega_{t} $ and that $\phi$ and $\psi$ are bounded by one (\cref{as:boundedfeature}).
\$
\E_t[\| \Tilde{\nabla}_\theta \Phi(\thet,\omet)- \Tilde{\nabla}_\theta \Phi(\thest,\omest) \|\sq] & = \E_t[\| \phi_t \psi_t^\top (\omet - \omest) \|\sq]
 \leq \| \omet - \omest \| \sq
\$
We bound $ \| \omet - \omest \|\sq$ by 
\#
\| \omet - \omest \|\sq & \leq 2 \| \omet - \hat \ome (\thet) \| \sq + 2 \| \hat \ome(\thet) - \omest\| \sq
\label{eq:12121141}
\\
& = 2 \| \omet - \hat \ome (\thet) \| \sq + 2 \|  (B\inv A)(\thest - \thet) \| \sq
\notag
\\
& \leq 2  \| \omet - \hat \ome (\thet) \| \sq + 2  \LP \muB\inv \cdot \| \thest - \thet \| \sq
\label{eq:12121141b}
\\
& \leq 2 ( \| \omet - \hat \ome (\thet) \| \sq + \muB\inv \cdot \| \thest - \thet \| \sq)
\label{eq:12121141c}
\#
where in \cref{eq:12121141} we use that $\omest = B\inv (A\thest - b)$ and $\hat \omega (\thet) = B\inv (A\thet - b)$; in \cref{eq:12121141b} we use $\|B\inv A \| = \|B^{-1/2} (B^{-1/2}A) \| \leq \muB^{-1/2} \LP ^{-1/2}$; in \cref{eq:12121141c} we use $\LP\leq 1$. This completes the proof of the first inequality.

For the second inequality \cref{eq:varomega} we use similar reasoning.
\#
\E_t[ & \|\Tilde{\nabla}_\ome \Phi(\thet,\omet) - {\nabla}_\ome \Phi(\thet,\omet)\|^2] 
\notag
\\
& \leq \E_t[\|\Tilde{\nabla}_\ome \Phi(\thet,\omet) \|^2] 
\notag
\\
& \leq 2\E_t[\|\Tilde{\nabla}_\ome \Phi(\thet,\omet) - \Tilde{\nabla}_\ome \Phi(\thest,\omest)\|^2]  + 2 \E_t[\|\Tilde{\nabla}_\ome \Phi(\thest,\omest)\|^2 ]
\notag
\#
For the first term, note $\Tilde{\nabla}_\ome \Phi(\thet,\omet;\xi_t) =  ( \phi(x_t,a_t)^\top \theta_{t} -x_{t}' - \psi(z_t)^\top \omega_{t}  )\psi(z_t) .$ and thus we have
\#
\E_t[\|\Tilde{\nabla}_\ome \Phi(\thet,\omet) - \Tilde{\nabla}_\ome \Phi(\thest,\omest)\|^2] & = \E_t[\| \psi_t \phi_t^\top(\thet - \thest) + \psi_t\psi_t^\top(\omet - \omest) \|\sq]
\notag
\\
& \leq 2 \| \thet - \thest\|\sq + 2 \| \omet - \omest \|^2.
\notag
\\
& \leq (2 + 4\LP \muB\inv ) \| \thet - \thest\|\sq + 4 \| \omet - \hat \omega(\thet) \|^2 
\notag
\\
& \leq 2^3 (\muB\inv  \| \thet - \thest\|\sq  + \| \omet - \hat \omega(\thet) \|^2)
\notag
\#
where we have used \ref{as:boundedfeature}, and $\mu_B\inv\geq 1$ and $\LP\leq 1$. This proves \cref{eq:varomega}. So we complete the proof of Lemma~\ref{lm:obs:varboound}.
\end{proof}

\subsection{Proof Lemma~\ref{lm:primal:descent}}
\label{sec:pf:lm:primal:descent}

\begin{proof}[Proof of Lemma~\ref{lm:primal:descent}]
Conditioning on $\cF_t$, we have
\# \label{eq:271033a}
\E_t[\|\thetp - \thest  \|^2 ] = \| \E_t[ \thetp - \thest ]\|^2 + \E_t[\|(  \thetp - \thest) - \E_t[  \thetp - \thest ] \|^2 ]
\#
We bound the first term in \cref{eq:271033a}.
\#
\| \E_t[ \thetp - \thest ]\|^2 &= \| {\Tilde{\theta}_{t+1}} - \thest \|^2 
\notag
 \\
 & \leq \big(\| {{\hat \theta}_{t+1}} - \thest\| + \|{\Tilde{\theta}_{t+1}} - {{\hat \theta}_{t+1}} \|\big)^2
 \notag
 \\
 & \leq  \big((1-\etathet\muIV)\| \thet - \thest \| + \|{\Tilde{\theta}_{t+1}} - {{\hat \theta}_{t+1}} \|\big)^2
 \label{eq:271030c}
 \\
 &\leq (1-\etathet\muIV)\| \thet - \thest \|^2 + \frac{1}{\etathet \muIV} \|{\Tilde{\theta}_{t+1}} - {{\hat \theta}_{t+1}} \|^2 \label{eq:271030d},
\#
Here in \cref{eq:271030c} we use Lemma~\ref{lm:gddescent} since (i) ${{\hat \theta}_{t+1}} = \thet - \etathet \nabla P(\thet)$, (ii) $P$ is $\muIV$-strongly convex and $\LP$-smooth (Lemma~\ref{lm:obs}), and (iii) our choice of stepsize.
In \cref{eq:271030d} we use that for any $\epsilon\in (0,1)$, it holds $((1-\epsilon)a+b)\sq \leq (1-\epsilon)a\sq + \epsilon\inv b\sq$; see Lemma~\ref{lm:splitsquare} for a proof.

We bound the second term in \cref{eq:271030d} by 
\$
 \big\|{\Tilde{\theta}_{t+1}} - {{\hat \theta}_{t+1}}\big \|^2 & =(\etathet)^2 \| \nabla_\theta \Phi(\thet,\omet) - \nabla_\theta P(\thet )\|^2
 \\
 &= (\etathet)^2 \| A^\top \omet - A^\top \hat \ome(\thet)\|^2
 \\
 &\leq (\etathet)^2 L_A^2 \| \omet - \hat \ome(\thet)\|^2\,.
\$
Continuing from \cref{eq:271030d}, we have
\#
\big\| \E_t[\thetp - \thest ]\big\|^2 \leq  (1-\etathet\muIV)\| \thet - \thest \|^2 +\etathet\cdot { L_A^2 }{ \muIV \inv } \cdot \| \omet - \hat \ome(\thet)\|^2 \label{eq:1027a}
\#
Next we bound the second term in \cref{eq:271033a}.
\#
\E_t\big[\big\|(  \thetp - \thest) - \E_t[  \thetp - \thest ] \big\|^2 \big] &= \E_t\big[\big\| \thetp - \E_t[\thetp]\big\|^2\big]
\notag
\\
&= \E_t\big[\big\|  \thetp - {\Tilde{\theta}_{t+1}} \big\|^2\big]
\notag
\\
&= (\etathet)^2 \cdot \E_t\big[\big \| \natthephi(\thet,\omet) - \nabla \Phi(\thet,\omet) \big\|^2\big]\,. \label{eq:1017c}
\#
This can be bounded by Lemma~\ref{lm:obs:varboound}. Plugging into \cref{eq:271033a} the bounds in \cref{eq:1027a} and \cref{eq:1017c}, 
\#
\E_t[\big\|\thetp - \thest  \big\|^2 \big] & \leq  (1-\etathet\muIV)\| \thet - \thest \|^2 +(\etathet) { L_A^2 }{ \muIV \inv}\| \omet - \hat \ome(\thet)\|^2
\notag
\\ 
& \quad\quad +  (\etathet)^2 \E_t\big[ \big\| \natthephi(\thet,\omet) - \nabla \Phi(\thet,\omet) \big\|^2\big]
\notag
\\
 &\leq  (1-\etathet\muIV)\| \thet - \thest \|^2 +(\etathet) { L_A^2 }{ \muIV \inv}\big\| \omet - \hat \ome(\thet)\big\|^2 
\notag
\\
& \quad\quad + \etathetsq \cdot \big( 4\big\| \omet - \omehat(\thet)\big\|^2  + 4\LP\muB\inv \| \thet - \thest\|^2 + 2 \signathesq \big)
\notag
\#
where we have used Lemma~\ref{lm:obs:varboound}. Taking expectation on both sides, we get
\#
\E[ \| \theta_{t+1} - \theta^*\|^2] & \leq \big(1- \muIV \eta_t^\theta  + 4\LP\muB\inv \etathetsq \big) \cdot \E\big[ \| \theta_{t} - \theta^*\|^2\big] 
\notag
\\
& \quad + \big(L_A^2\muIV^{-1}\eta^\theta_t  + 4 \etathetsq\big)\cdot \E\big[\big\| \omega_t - \hat \omega (\theta_t)\big\|^2 \big]  
\notag
\\ 
& \quad + 2 (\eta^\theta_t)^2 \cdot \sigma^2_{\nabla \theta}
\notag
\\
& \leq \big(1- \muIV \eta_t^\theta  +4 \muB\inv \etathetsq \big) \cdot \E\big[ \| \theta_{t} - \theta^*\|^2\big] 
\notag
\\
& \quad + \big(\muIV^{-1}  \eta^\theta_t + 4 \etathetsq\big)\cdot \E\big[\big\| \omega_t - \hat \omega (\theta_t)\big\|^2 \big]  
\notag
\\ 
& \quad + 2 (\eta^\theta_t)^2 \cdot \sigma^2_{\nabla \theta}
\notag
\#
where we use $\LP\leq 1$ and $\LA\leq 1$. This completes the proof of Lemma~\ref{lm:primal:descent}.
\end{proof}

\subsection{Proof of Lemma~\ref{lm:dual:descent}}
\label{sec:pf:lm:dual:descent}
\begin{proof}[Proof of Lemma~\ref{lm:dual:descent}]
We first bound the one-step difference of primal updates.

\begin{lemma}[One-step difference] \label{lm:diffprime}
 Consider the update sequence $\{ \omet, \thet\}$. Conditioning on $\cF_t$, we have
\#
\big\| \E_t[\thetp - \thet ]\big\|^2 \leq 2 (\etathet)^2  \big( L_P^2 \cdot  \| \thet - \theta^*\|^2 + L_A^2 \cdot  \| \omet - \hat \omega(\thet) \|^2 \big) \,.
\notag
\#
\end{lemma}

\begin{proof}[Proof of Lemma~\ref{lm:diffprime}]
We start by noting
\#
\big\| \E_t[ \thetp - \thet ]\big\|^2 &= \big\| {\Tilde{\theta}_{t+1}} - \thet \big\|^2 = \etathetsq \cdot \| A^\top \omet \|^2
\notag
\\
&\leq\etathetsq  \cdot \big( 2\|A^\top \hat \ome(\thet) \|^2 + 2\| A^\top \omet -A^\top \hat \ome (\thet)\|^2\big) \,.  \label{eq:1049c}
\#
For the first term in \cref{eq:1049c}, we have
\#
 \|A^\top \hat \ome(\thet) \| = \| \nabla P(\thet)\| =  \| \nabla P(\thet) - \nabla P (\thest)\| \leq L_P \| \thet - \thest\|.
 \label{eq:1053a}
\#
For the second term in \cref{eq:1049c}, we have
\# \label{eq:1052a}
\| A^\top \omet -A^\top \hat \ome (\thet)\|^2\leq  L_A^2 \| \omet - \hat \ome(\thet)\|^2.
\#
Plugging into \cref{eq:1049c} the bounds in \cref{eq:1053a} and \cref{eq:1052a}, we complete the proof of Lemma~\ref{lm:diffprime}.
\end{proof}

Now we prove Lemma~\ref{lm:dual:descent}. Conditioning on $\cF_t$, we have
\#
\E_t \big[ \|  \ometp - \hat \omega(\thetp) \|^2 \big] & = 
\big\| \E_t[  \ometp - \hat \omega(\thetp) ] \big\|^2
\label{eq:1070a}
\\
&\quad + \E_t \Big[ \big\| \big( \ometp - \hat \omega(\thetp)\big) - \E_t \big[ \ometp - \hat \omega(\thetp) \big] \big\|^2 \Big]\,.
\label{eq:1070b}
\#
Next we bound the first term in \cref{eq:1070a}
\#
\big\| \E_t \big[  \ometp - \hat \omega(\thetp) \big] \big\|^2 & = \big\| \E_t \big[  \ometp - \hat \omega(\thet) \big]+ \E_t \big[ \hat \omega(\thet) - \hat \omega(\thetp) \big] \big\|^2
\notag
\\
& \leq \Big(  \big\| \E_t \big[  \ometp - \hat \omega(\thet) \big] \big\| + \big\| \E_t \big[ \hat \omega(\thet) - \hat \omega(\thetp) \big] \big\| \Big)^2
\notag
\\
& = \Big( \| {\Tilde{\omega}_{t+1}} - \hat \omega(\thet) \| + \big\| \E_t \big[ \hat \omega(\thet) - \hat \omega(\thetp) \big] \big\| \Big)^2
\notag
\\
\label{eq:114833}
& \leq \Big(  (1- \mu_B \etaomet) \| \omet -  \hat \ome(\thet) \| + \big\| \E_t \big[ B^{-1}A (\thet -\thetp) \big] \big\| \Big)^2
\\
& \leq \Big(  (1- \mu_B \etaomet) \| \omet -  \hat \ome(\thet) \| + {L_A  }{\mu_B\inv} \cdot \big\| \E_t[  \thet -\thetp ] \big\| \Big)^2
\notag
\\
& \leq (1- \mu_B \etaomet) \| \omet -  \hat \ome(\thet) \|^2 + \LP\muB\inv \cdot \frac{1}{\mu_B \etaomet} \cdot \big\| \E_t [  \thet -\thetp  ]  \big\| ^2 \,.  
\label{eq:1120a}
\#
Here in \cref{eq:114833} we use that (i) ${\Tilde{\omega}_{t+1}} = \omet +\etaomet \nabla \Phi(\thet,\omet)$, (ii) for $\theta_t$, the map $\omega \mapsto -\Phi(\thet,\ome)$ is $\muB$-strongly convex and $\LB$-smooth (Lemma~\ref{lm:obs}), (iii) our choice of stepsize, and (iv) $\hat \omega(\thet)$ is the minimizer of the map $\omega \mapsto -\Phi(\thet,\ome)$. In \cref{eq:1120a} we use that for any $\epsilon\in (0,1)$, any $a,b\in \R$, it holds $((1-\epsilon)a+b)\sq \leq (1-\epsilon)a\sq + \epsilon\inv b\sq$.

Using Lemma~\ref{lm:diffprime} we can bound the second term in \cref{eq:1120a} by $ \| \omet -  \hat \ome(\thet) \|$ and $\|\thet - \thest \|$.

Now we bound the second term in \cref{eq:1070b}.
\#
\E_t\big [\|& (\ometp - \hat \omega(\thetp)) - \E_t[\ometp - \hat \omega(\thetp)] \|^2 \big] 
\notag
\\
&\leq \E_t\Big[ 2 \big\|  \ometp  - \E_t[\ometp ]  \big\|^2 + 2 \big\| \hat \omega(\thetp)- \E_t \big[\hat \omega(\thetp)\big] \big\|^2 \Big] \,. \label{eq:1101820d}
\#
For the first term in \cref{eq:1101820d} we have
\#
\E_t \big[ \big\|  \ometp  - \E_t[ \ometp]  \big\|^2 \big] & = (\etaomet)^2  \cdot \E_t\big[\big\|\natomephi(\thet,\omet) - \nabla_\ome \Phi(\thet,\omet)\big\|^2\big] 
\notag
\\
\notag
& \leq \etaometsq \cdot \big( 16 \| \thet - \thest\|^2 + 16 \| \omet - \omehat (\thet)\|^2 + 2 \signaomesq \big)
\#
where we have used Lemma~\ref{lm:obs:varboound}. For the second term in \cref{eq:1101820d} we have
\#
\E_t \big[ \big\| & \hat \omega(\thetp)- \E_t[\hat \omega(\thetp)] \big\|^2 \big] 
\notag
\\
&= \E_t\big[ \big\| B^{-1}A\thetp - \E_t[B^{-1}A\thetp]  \big\|^2\big]
\notag
\\
&\leq  \LP\muB\inv \cdot \E_t\big[ \| \thetp - \E_t[\thetp ]\|^2 \big]
\notag
\\
&=
 \LP\muB\inv \cdot \etathetsq \cdot \E_t \big[\big\|\Tilde{\nabla}_\theta \Phi(\thet,\omet) - {\nabla}_\theta \Phi(\thet,\omet)\big\|^2\big] 
\notag
\\
& \leq \muB\inv \cdot \etathetsq \cdot ( 4\| \omet - \omehat(\thet)\|^2  + 4 \muB\inv \| \thet - \thest\|^2 + 2 \signathesq)
\,. \label{eq:1119a}
\#
where we have used Lemma~\ref{lm:obs:varboound} in \cref{eq:1119a}.

Continuing from \cref{eq:1101820d} (the variance part), we obtain
\#
\E_t\Big[ & \big\| \big(\ometp - \hat \omega(\thetp)\big) - \E_t[\ometp - \hat \omega(\thetp)]\big \|^2 \Big] 
\notag
\\
\leq & 2^5 \big(\muB\inv \etaometsq + \muB\invsq \etathetsq  \big)  \| \thet - \thest \|\sq
\notag
\\
& + 2^5 \big(\etaometsq + \muB\inv \etathetsq \big) \|\omet - \hat \ome(\thet) \|\sq 
\notag
\\
& + 4 \big(\muB\inv \etathetsq \signathesq + \etaometsq \signaomesq\big)
\label{eq:111828}
\#
Putting together \cref{eq:1120a}, \cref{eq:111828} and Lemma~\ref{lm:diffprime}, we have
\#
\E_t[\| \ometp - \hat \omega(\thetp) \|^2] \leq & (1-\mu_B \etaomet)  \| \omet - \hat \omega (\thet)\|^2  \notag
\notag
\\
& + 2^5 \big( \muB\invsq \etathetsq/\etaomet + \muB\inv \etaometsq + \muB\invsq \etathetsq\big) \|\thet - \thest  \|\sq 
\notag
\\
& + 2^5  \big(\muB\invsq \etathetsq/\etaomet + \etaometsq + \muB\inv \etathetsq\big) \| \omet - \hat \omega (\thet)\|^2  
\notag
\\ 
& + 4 \big(\muB\inv \etathetsq \signathesq + \etaometsq \signaomesq\big)
\,. \label{eq:1230b}
\#
This completes the proof of Lemma~\ref{lm:dual:descent}
\end{proof}
\section{Supporting Lemmas} 

\begin{lemma} 
\label{lm:splitsquare}
For any $\epsilon\in (0,1)$, any $a,b\in \R$, it holds $( (1-\epsilon)a + b)\sq \leq (1-\epsilon)a\sq + \epsilon\inv b\sq$.
\end{lemma}
\begin{proof}
By the \csineq, we have for all $\beta > 0$, $(a+b)\sq \leq(1+\beta)a\sq + (1+\beta\inv)b\sq$. Setting $\beta=\epsilon(1-\epsilon)\inv$ completes the proof.
\end{proof}
\begin{lemma}[One-step gradient descent for smooth and strongly-convex function] \label{lm:gddescent} 
Suppose $f:\R^d \to \R$ is a $\beta$-smooth and $\alpha$-strongly convex function. Let $ x^* = \argmin_{x\in \R^d} f(x)$. For any $0<  \eta \leq \tfrac{2}{\alpha +\beta}$ and any $x\in \R^d$, let $x^+ = x - \eta \nabla f(x)$. Then $\|{x}^+-{x}^*\| \leq(1- \alpha \eta)\|x-{x}^*\|$.
\end{lemma}
\begin{proof}
See Lemma~3.1 of \cite{du2019linear}.
\end{proof}

\begin{lemma}[Expectation Difference Under Two Gaussians, Lemma~C.2 in \citealt{sham2020information}] \label{lm:gausdiff}
For Gaussian distribution $\mathcal{N}\left(\mu_{1}, \sigma^{2} I\right)$ and $\mathcal{N}\left(\mu_{2}, \sigma^{2} I\right)$ ($\sigma^2 \neq 0$), for any positive measurable function $g$, we have
\$\E_{z \sim N_{1}}[g(z)]-\mathbb{E}_{z \sim N_{2}}[g(z)] \leq \min \left\{\frac{\left\|\mu_{1}-\mu_{2}\right\|}{\sigma}, 1\right\} \sqrt{\mathbb{E}_{z \sim \mathcal{N}_{1}}\left[g(z)^{2}\right]}\,.\$
\end{lemma}

\begin{proof} For completeness we present a proof. Note
\$ 
\E_{z \sim N_{1}}[g(z)]-\mathbb{E}_{z \sim N_{2}}[g(z)]  
&=\mathbb{E}_{z \sim N_{1}}\left[g(z)\left(1-\frac{\mathcal{N}_{2}(z)}{\mathcal{N}_{1}(z)}\right)\right] 
\\
& \leq \sqrt{\mathbb{E}_{z \sim N_{1}}\left[g(z)^{2}\right]} \sqrt{\int \frac{\left(\mathcal{N}_{1}(z)-\mathcal{N}_{2}(z)\right)^{2}}{\mathcal{N}_{1}(z)} d z}
\\
&=\sqrt{\mathbb{E}_{z \sim N_{1}}\left[g(z)^{2}\right]} \sqrt{\exp \left(\frac{\mid \mu_{1}-\mu_{2} \|^{2}}{2 \sigma^{2}}\right)-1} \,.
\$
Since $g\geq 0$ we have $\E_{z \sim N_{1}}[g(z)]-\mathbb{E}_{z \sim N_{2}}[g(z)] \leq\E_{z \sim N_{1}}[g(z)]\leq \sqrt{\mathbb{E}_{z \sim N_{1}}\left[g(z)^{2}\right]} .$ Finally, we use $\operatorname{exp}(x) \leq 1+2 x$ for $0 \leq x \leq 1$.
\$
\E_{z \sim N_{1}}[g(z)]-\mathbb{E}_{z \sim N_{2}}[g(z)]  
&\leq \sqrt{\mathbb{E}_{z \sim N_{1}}\left[g(z)^{2}\right]} \sqrt{\min \left\{\exp \left(\frac{\left\|\mu_{1}-\mu_{2}\right\|^{2}}{2 \sigma^{2}}\right)-1,1\right\}} 
\\
&\leq \sqrt{\mathbb{E}_{z \sim N_{1}}\left[g(z)^{2}\right]} \cdot \min \left\{\frac{\left\|\mu_{1}-\mu_{2}\right\|^{}}{\sigma^{}}, 1\right\} \,.
\$
This completes the proof.
\end{proof}

\vskip 0.2in
\bibliography{graphbib.bib}
\bibliographystyle{abbrvnat}
\end{document}